\documentclass{article}
\usepackage[utf8]{inputenc}
\usepackage{xcolor}
\usepackage{booktabs}
\usepackage[round,compress]{natbib}
\usepackage[colorlinks=true,bookmarks=false,linkcolor=blue,urlcolor=blue,citecolor=blue,breaklinks=true]{hyperref}
\usepackage[accepted]{icml2024}
\usepackage{fancyhdr}
\usepackage{amsfonts}
\usepackage{amsmath}
\usepackage{amsthm}
\usepackage{amssymb}
\usepackage{dsfont}
\usepackage{multirow}
\usepackage{xcolor}
\usepackage{color}
\usepackage{graphicx}

\usepackage{verbatim}
\usepackage{xspace} 
\newcommand{\algname}[1]{{\sf #1}\xspace}

\usepackage{enumerate}
\usepackage{enumitem}

\providecommand{\lin}[1]{\ensuremath{\left\langle #1 \right\rangle}}

  \providecommand{\R}{\mathbb{R}} 
  
  \DeclareMathOperator{\E}{{\mathbb E}}

  \DeclareMathOperator*{\argmin}{arg\,min}

  \renewcommand{\aa}{\mathbf{a}}
  \providecommand{\bb}{\mathbf{b}}

  \renewcommand{\gg}{\mathbf{g}}
  \providecommand{\hh}{\mathbf{h}}

  \let\lll\ll
  \renewcommand{\ll}{\mathbf{l}}

  \providecommand{\vv}{\mathbf{v}}
  
  \providecommand{\xx}{\mathbf{x}}
  \providecommand{\yy}{\mathbf{y}}
  \providecommand{\zz}{\mathbf{z}}
  
  \providecommand{\mA}{\mathbf{A}}

  \providecommand{\cO}{\mathcal{O}}

  \usepackage{bm}

\RequirePackage[colorinlistoftodos,bordercolor=orange,backgroundcolor=orange!20,linecolor=orange,textsize=scriptsize]{todonotes}
\providecommand{\mycomment}[3]{\todo[caption={},color=#3!20,inline]{\textbf{#1: }#2}}%
\providecommand{\myinlinecomment}[3]{%
  {\color{#1}#2: #3}}%
\newcommand\commenter[2]%
{%
  \expandafter\newcommand\csname i#1\endcsname[1]{\myinlinecomment{#2}{#1}{##1}}
  \expandafter\newcommand\csname #1\endcsname[1]{\mycomment{#1}{##1}{#2}}
}
  
\newtheorem{lemma}{Lemma}
\newtheorem{corollary}[lemma]{Corollary}

\newtheorem{definition}{Definition}
\newtheorem{remark}[lemma]{Remark}
\newtheorem{assumption}{Assumption}
\newtheorem{theorem}[lemma]{Theorem}

\usepackage[colorlinks=true,linkcolor=blue]{hyperref} 
\usepackage[capitalize,noabbrev]{cleveref}

\usepackage{url}

\renewcommand{\epsilon}{\varepsilon}

\usepackage{tcolorbox}

\newtcbox{\comparison}{on line,
  colframe=blue,colback=white,
  boxrule=0.5pt,arc=4pt,boxsep=0pt,left=6pt,right=6pt,top=6pt,bottom=6pt}
\providecommand{\Avg}{{\frac{1}{n}\sum_{i=1}^n}}

\commenter{seb}{red}
\commenter{xiaowen}{blue}
\commenter{anton}{orange}

\begin{document}

\twocolumn[
\icmltitle{Federated Optimization with Doubly Regularized Drift Correction}

\icmlsetsymbol{equal}{*}

\begin{icmlauthorlist}
\icmlauthor{Xiaowen Jiang}{yyy,comp}
\icmlauthor{Anton Rodomanov}{yyy}
\icmlauthor{Sebastian U. Stich}{yyy}
\end{icmlauthorlist}

\icmlaffiliation{yyy}{CISPA Helmholtz Center for Information Security, Saarbrücken, Germany}
\icmlaffiliation{comp}{Universität des Saarlandes, Saarbrücken, Germany}

\icmlcorrespondingauthor{}{\{xiaowen.jiang, anton.rodomanov, stich\}@cispa.de}

\icmlkeywords{Machine Learning, ICML}

\vskip 0.3in
]

\printAffiliationsAndNotice 

\begin{abstract}
    Federated learning is a distributed optimization paradigm that allows training machine learning models across decentralized devices while keeping the data localized. The standard method, \algname{FedAvg}, suffers from client drift which can hamper performance and increase communication costs over centralized methods. Previous works proposed various strategies to mitigate drift, yet none have shown uniformly improved communication-computation trade-offs over vanilla gradient descent.

    In this work, we revisit \algname{DANE}, an established method in distributed optimization.
    We show that (i) \algname{DANE} can achieve the desired communication reduction under Hessian similarity constraints. 
    Furthermore, (ii) we present an extension, \algname{DANE+}, which supports arbitrary inexact local solvers and has more freedom to choose how to aggregate the local updates.  We propose (iii) a novel method, \algname{FedRed}, which has improved local computational complexity and retains the same communication complexity compared to \algname{DANE/DANE+}. This is achieved by using doubly regularized drift correction.

\end{abstract}

\begin{table*}[ht!]
\resizebox{\textwidth}{!}
{\begin{minipage}{1.82\textwidth}
\centering
\begin{tabular}{@{}ccccccccc@{}}
\toprule
\multirow{2}{*}{\textbf{Algorithm}} &
 \multicolumn{2}{c}{\textbf{$\mu$-strongly convex}} & \multicolumn{2}{c}{\textbf{General convex}} & 
 \multicolumn{2}{c}{\textbf{Non-convex}} & 
 \multirow{2}{*}{\textbf{Guarantee}} 
 \\
 \cmidrule(lr){2-3}\cmidrule(lr){4-5}\cmidrule(lr){6-7}
 & \multicolumn{1}{c}{\# comm rounds} & 
\multicolumn{1}{c}{\# local steps at round $r$} & \multicolumn{1}{c}{\# comm rounds} & 
\multicolumn{1}{c}{\# local steps at round $r$} & \multicolumn{1}{c}{\# comm rounds} & 
\multicolumn{1}{c}{\# local steps at round $r$} & 
\\
\midrule
Centralized GD~{\small{\cite{nesterov-book}}} 
\footnote{The smoothness parameter $L$ for centralized \algname{GD} can be as small as the Lipschitz constant of the global function $f$.}
& $\cO\bigl(\frac{L}{\mu}\log(\frac{R_0^2}{\epsilon})\bigr)$ 
& $1$ 
& $\cO\bigl( \frac{LR_0^2}{\epsilon} \bigr)$ 
& $1$ 
& $\cO\bigl( \frac{LF_0}{\epsilon} \bigr)$ 
& $1$
& deterministic 
\\
\rule{0pt}{4ex} 
Scaffold~{\small{\cite{scaffold}}} 
\footnote{To achieve $||\nabla f(\xx^R)||^2 \le \epsilon$,
\algname{Scaffold} requires 
$\frac{L + \delta_B K}{\mu K} \log (\frac{R_0^2}{\epsilon})$
communication rounds
for strongly-convex quadratics,  and 
$\frac{(L + \delta_B K) F_0}{K \epsilon}$
communication rounds
for convex quadratics, where $K$ is the number of local steps.
 }
& $\cO\bigl(\frac{L}{\mu}\log(\frac{R_0^2}{\epsilon})\bigr)$ 
& $\forall K \ge 1$ 
  \footnote{\algname{Scaffold} allows to use any number of local 
            steps $K$ by choosing the stepsize to be inversely
            proportional to $K$. }
& $\cO\bigl( \frac{LR_0^2}{\epsilon} \bigr)$ 
& $\forall K \ge 1$ 
& $\cO\bigl( \frac{LF_0}{\epsilon} \bigr)$ 
& $\forall K \ge 1$
& deterministic 
\\
\rule{0pt}{4ex} 
FedDyn~{\small{\cite{feddyn}}}
\footnote{\algname{FedDyn} and \algname{SONATA} assume that the local subproblem can be 
solved exactly. `-' means there is no definition of local steps.\label{FedDyn}}
& $\cO\bigl(\frac{L}{\mu}\log(\frac{R_0^2 + A^2/L}{\epsilon})\bigr)$
& -
& $\cO( \frac{LR_0^2 + A^2}{\epsilon})$
& -
& $\cO( \frac{L F_0}{\epsilon} )$
& -
& deterministic
\\
\rule{0pt}{4ex} 
Scaffnew~{\small{\cite{proxskip}}} 
\footnote{For \algname{Scaffnew} and \algname{FedRed-GD},
the column `\# comm rounds' represents the expected number of 
total communications required to reach $\epsilon$ accuracy.
The column  `\# of local steps at round $r$' 
is replaced with the expected number of local steps between two communications. The general convex result of \algname{Scaffnew} is established in Theorem 11 in the \algname{RandProx} paper~\cite{randprox}.
\label{scaffnew}}
& $\cO\bigl(\sqrt{\frac{L}{\mu}}\log(\frac{R_0^2 + \frac{H_0^2}{\mu L}}
{\epsilon})\bigr)$ 
& $\sqrt{\frac{L}{\mu}}$
& $\cO\bigl( \frac{pLR_0^2+\frac{H_0^2}{Lp}}{\epsilon}\bigr)$
& $\frac{1}{p}$
& unknown
& unknown
& in expectation 
\\
\rule{0pt}{4ex} 
MimeMVR~{\small{\cite{mime}}}
\footnote{Mime assumes 
$\Avg ||\nabla f_i (\xx) - \nabla f(\xx)||^2 \le G^2$ for any
$\xx \in \R^d$. Note $G = +\infty$ for some simple quadratics.}
& unknown
  \footnote{`unknown' means no theoretical results are established so far. }
& unknown
& unknown
& unknown
& $\cO\bigl(
    \frac{\delta_B G F_0}{\sqrt{n} \epsilon^{3/2}} 
    + \frac{G^2}{n\epsilon} + \frac{\delta_B F_0}{\epsilon} 
    \bigr)$
& $\frac{L}{\delta_B}$
& deterministic
\\
\rule{0pt}{4ex} 
SONATA~\cite{sonata}
\footref{FedDyn}
&
$\cO\bigl(\frac{\delta_B}{\mu} \log (\frac{R_0^2}{\epsilon})\bigr)$
&
-
&
unknown
&
unknown
&
unknown
&
unknown
&
deterministic
\\
\rule{0pt}{4ex} 
CE-LGD~\cite{ce-lgd}
\footnote{The communication and computational complexity for the 
non-convex case is min-max optimal for minimizing twice-continuously differentiable smooth functions under $\delta_B$-BHD using first-order gradient methods.}
&
unknown
&
unknown
&
unknown
&
unknown
&
$\cO( \frac{\delta_B F_0}{\epsilon} )$
&
$\frac{L}{\delta_B}$
&
in expectation
\vspace{1mm}
\\
\hline
\vspace{1mm}
\rule{0pt}{4ex} 
SVRP-GD~\cite{svrp}
\footnote{\algname{SVRP} aims at minimizing a different measure which is the total amount of information transmitted between the server and the clients.}
&
$\cO\bigl( (n + \frac{\delta_A^2}{\mu^2}) \log (\frac{R_0^2}{\epsilon})\bigr)$
&
$\cO\Bigl(\frac{L\delta_A^2+\mu}{\mu\delta_A^2+\mu}\log\bigl(\frac{\max\{\mu^2n/\delta_A^2,1\}}{\epsilon}\bigr)\Bigr)$
&
unknown
&
unknown
&
unknown
&
unknown
&
in expectation
\\
\hline
\rule{0pt}{4ex} 
DANE~\cite{dane} (\textbf{this work}) 
& $\cO\bigl(\frac{\delta_A}{\mu} \log(\frac{R_0^2}{\epsilon}) \bigr)$ 
& -\footnote{\algname{DANE} uses exact local solvers.}
& $\cO(\frac{\delta_A R_0^2}{\epsilon})$ 
& -
& -
& -
& deterministic
\\
\rule{0pt}{4ex}
$\text{DANE+-GD}$ (\textbf{this work}, Alg.~\ref{Alg:FrameworkDeterministic}) 
& $\cO\bigl(\frac{\delta_A}{\mu} \log(\frac{R_0^2}{\epsilon}) \bigr)$ 
& $\cO\Bigl(
\frac{L}{\mu+\delta_A} 
\log\bigl( \frac{L}{\mu + \delta_A}(r+1) \bigr) \Bigr)$ 
& $\cO(\frac{\delta_A R_0^2}{\epsilon})$ 
& $\cO\Bigl(
\frac{L}{\delta_A} 
\log\bigl( \frac{L}{\delta_A}(r+1) \bigr) \Bigr)$ 
& $\cO(\frac{\delta_B F_0}{\epsilon})$  
& $\cO\Bigl( 
        \frac{L}{\delta_B} R
        \Bigr)$
  \footnote{$R$ is the number of communication rounds.}
& deterministic
\\
\rule{0pt}{4ex} 
FedRed-GD (\textbf{this work}, Alg.~\ref{Alg:GDLocalSolver}) 
\footref{scaffnew}
& $\cO\bigl(\frac{\delta_A + \mu}{\mu}
    \log(\frac{R_0^2}{\epsilon}) \bigr)$ 
& $\frac{L}{\delta_A}$ 
& $\cO(\frac{\delta_A R_0^2}{\epsilon})$ 
& $\frac{L}{\delta_A}$
& $\cO(\frac{\delta_B F_0}{\epsilon})$  
& $\frac{L}{\delta_B}$ 
& in expectation 
\\
\bottomrule
\end{tabular}
\end{minipage}}
\caption{\small{Summary of convergence behaviors of the considered \textbf{non-accelerated} distributed algorithms where $L$ and $\mu$ stand for the smoothness and strong-convexity parameters of each function $f_i$, $\delta_A$, 
$\delta_B$ are defined in~\eqref{df:HessianSimilarity} and~\eqref{df:MaxHessianSimilarity}, 
$R_0^2 := ||\xx^0 - \xx^\star||^2$,
$F_0 := f(\xx^0) - f^\star$, 
$H_0^2 := \Avg ||\hh_{i,0} - \nabla f_i(\xx^\star)||^2$, and 
$A^2 := \frac{1}{L n}\sum_{i=1}^n||\nabla f_i(\xx^\star)||^2$.
The suboptimality $\epsilon$ is defined via $||\hat{\xx}^R-\xx^\star||^2$,
$f(\hat{\xx}^R)-f^\star$ and $||\nabla f(\hat{\xx}^R)||^2$ 
respectively
for strongly-convex, general convex, and non-convex functions
($\hat{\xx}^R$ is a certain output produced by the algorithm after $R$ communications.)
}} 
\label{tab:summary}
\end{table*}

\section{Introduction}
With the growing scale of datasets and the complexity of models, distributed optimization plays an increasingly important role in large-scale machine learning.
Federated learning has
emerged as an essential modern distributed learning paradigm where clients (e.g.\ phones and hospitals) collaboratively train a model without sharing their data, thereby ensuring a certain level of privacy~\cite{mcmahan2017FL,kairouz2021advances}. However, privacy leaks can still occur~\cite{inference-attack}.

Tackling communication bottlenecks is one of the key challenges in modern federated optimization~\cite{konevcny2016federated, mcmahan2017FL}. The relatively slow and unstable internet connections of participating clients often make communication highly expensive, which might impact the effectiveness of the overall training process. Therefore, a standard metric to evaluate the efficiency of a federated optimization algorithm is the total number of communication rounds required to reach a certain accuracy.

\algname{FedAvg}~\cite{mcmahan2017FL} as the pioneering algorithm improves the communication efficiency by 
doing multiple local stochastic gradient descent updates before communicating to the server. This strategy has demonstrated significant success in practice.  However, when the data is heterogeneous, \algname{FedAvg} suffers from \emph{client drift}, which might result in slow and unstable convergence~\cite{scaffold}. Subsequently, some advanced techniques, including drift correction and regularization, have been proposed to address this issue~\cite{scaffold,fedprox,feddyn,fedlin,proxskip}. The communication complexity established in these methods depends on the smoothness constant $L$ in general, which shows no advantage over the centralized gradient-based methods (see the first four rows in Table~\ref{tab:summary}, \algname{Scaffnew} has a similar complexity than the centralized fast gradient method under strong-convexity~\cite{nesterov-book} ).

Instead of proving the dependency on $L$, a recent line of research tries to develop algorithms with guarantees that rely on a potentially smaller constant.
\citet{scaffold} first demonstrated that when minimizing quadratics, \algname{Scaffold} can exploit the hidden similarity to reduce communication. 
Specifically, the communication complexity only depends on a measure $\delta_B$, which measures the maximum dissimilarity of Hessians among individual functions. Notably, $\delta_B$ is always smaller than $L$ up to a constant and can be much smaller in practice. Subsequently, SONATA~\cite{sonata}, its accelerated version~\cite{acc-sonata}, and \algname{Accelerated ExtraGradient sliding}~\cite{grad-sliding} show explicit communication reduction in terms of $\delta_B$ under strong convexity. 
For instance, the required communication rounds for \algname{SONATA} to reach $\epsilon$ accuracy is $\cO\bigl(\frac{\delta_B}{\mu}\log(1/\epsilon)\bigr)$ which is $\frac{L}{\delta_B}$ smaller than that achieved by centralized \algname{GD} (see line 6 of Table~\ref{tab:summary}).
Later, \algname{CE-LGD}~\cite{ce-lgd} also achieves improved communication complexity for smooth non-convex functions. 

More recently, \citet{svrp} introduced the averaged Hessian dissimilarity constant $\delta_A$, which can be even smaller than $\delta_B$. They established communication complexity with respect to $\delta_A$ for \algname{SVRP} under strong-convexity. Compared to \algname{SONATA}, the communication rounds turn into $\cO\bigl((n+\frac{\delta_A^2}{\mu^2})\log(1/\epsilon))$ (see line 8 of Table~\ref{tab:summary}). However, the aim of \algname{SVRP} is to reduce the total amount of information transmitted between the server and clients rather than to reduce the total number of communication rounds. This is orthogonal to the focus of this work.

So far, there are no methods that achieve provably good communication complexity with respect to $\delta_A$ for solving strongly-convex and general convex problems. In this work, we revisit \algname{DANE}, an established method in distributed optimization. Previously, its communication efficiency has only been proved when minimizing quadratic functions with a high probability~\cite{dane}.
 
We show that this historically first foundational algorithm already achieves good theoretical bounds in terms of $\delta_A$ under function similarity. 
We further utilize this fact to design simple algorithms that are efficient in terms of both local computation and total communication complexity for minimizing more general functions.

\paragraph{Contributions.}
We make the following contributions:

\begin{itemize}[leftmargin=12pt,itemsep=0pt,topsep=0pt]
    \item We identify the key mechanism that allows communication reduction for arbitrary functions when they share certain similarities. 
    This technique, which we call \emph{regularized drift correction}, 
    has already been implicitly used in \algname{DANE}~\cite{dane} and other algorithms.
    \item We establish improved communication complexity for \algname{DANE} in terms of $\delta_A$ under both general convexity and strong convexity.
    \item We present a slightly extended framework \algname{DANE+} for \algname{DANE}. 
    \algname{DANE+} allows the use of inexact local solvers and arbitrary 
    control variates.
    In addition, it has more freedom to choose how to aggregate the local updates.
    We  
    provide a tighter analysis of the communication and 
    local computation complexity for this framework, 
    for arbitrary continuously differentiable functions. 
    We show that \algname{DANE+} achieves 
    deterministic communication reduction across all the cases.
    \item We propose a novel framework \algname{FedRed} which employs
    \emph{doubly regularized drift correction}. 
    We prove that \algname{FedRed} enjoys the same communication 
    reduction as \algname{DANE+} but has improved local computational
    complexity. Specifically, when specifying \algname{GD} as a local solver, \algname{FedRed} may require fewer communication rounds than 
    vanilla GD without incurring additional computational overhead.
    Consequently, \algname{FedRed} demonstrates
    the effectiveness of taking standard local gradient steps for the minimization of smooth functions.
\end{itemize}

Table~\ref{tab:summary} summarizes our 
main complexity results. Compared with the first 7 algorithms, \algname{DANE}, \algname{DANE+} and \algname{FedRed} achieve communication complexity that depends on $\delta_A$ instead of $L$ and $\delta_B$ for minimizing convex problems. Compared with \algname{SVRP}, our rates for strongly-convex problems is $\cO(\frac{\delta_A}{\mu})$ instead of $\cO(\frac{\delta_A^2}{\mu^2}+n)$. Moreover, \algname{DANE+} and \algname{FedRed} also achieve the rate of $\cO(\frac{\delta_B F_0}{\epsilon})$ for solving non-convex problems, which is the same as \algname{CE-LGD} and is better than all the other methods. Furthermore, when specifying \algname{GD} as a local solver, \algname{FedRed-GD} is strictly better than \algname{DANE+-GD} in terms of total local computations.

\begin{figure*}[tb!]
    \centering
    \includegraphics[width=0.80\textwidth]{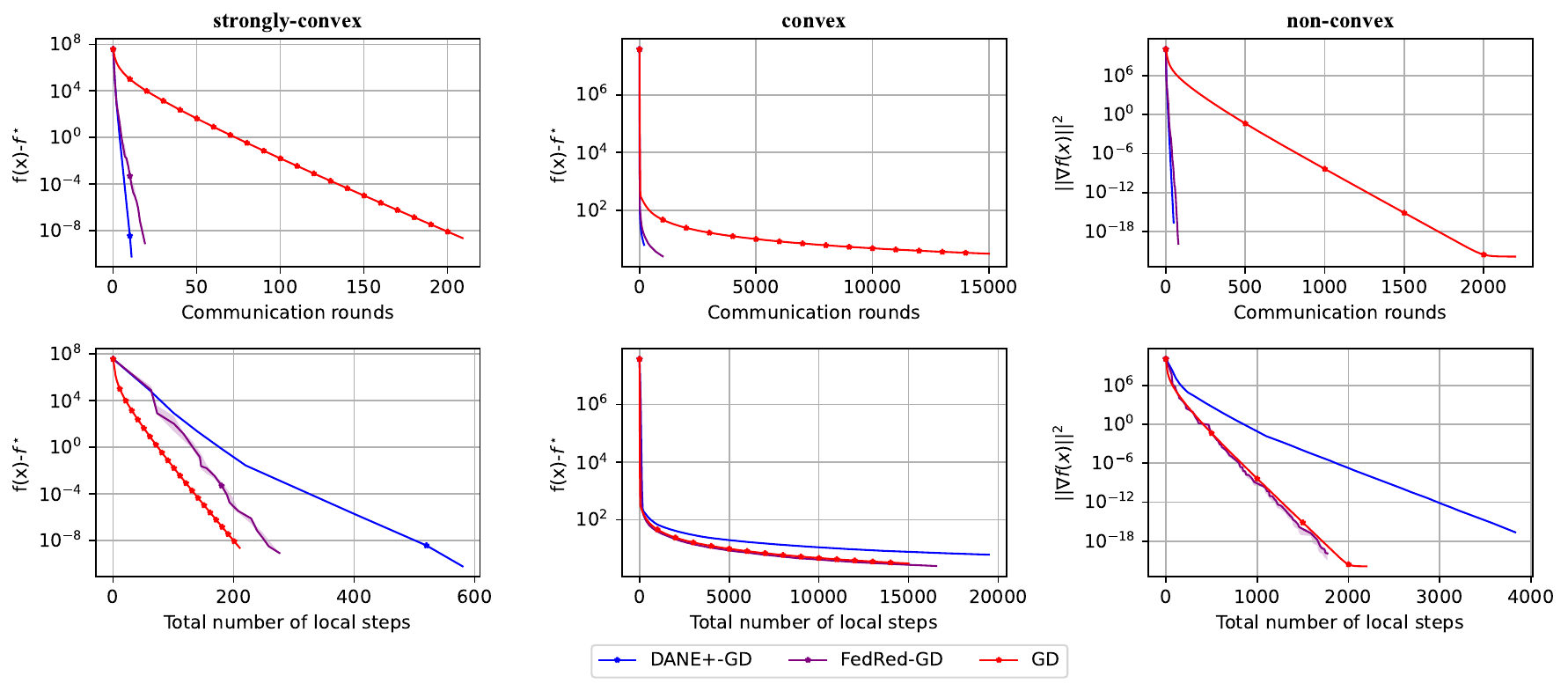}
    \caption{Illustrating communication reduction for \algname{DANE+-GD} and \algname{FedRed-GD} on synthetic dataset using quadratic loss with 
    $\frac{L}{\delta_A}\approx\frac{L}{\delta_B}\gtrsim 20$. 
    \algname{DANE+-GD} and \algname{FedRed-GD} require roughly
    $20$ times fewer communication rounds to 
    reach the same suboptimality as \algname{GD} while the total number of local computations of \algname{FedRed-GD} is at the same scale as \algname{GD}.
    (Repeated 3 times for \algname{FedRed-GD}. The solid lines and the shaded area represent the mean and the region between the minimum and the maximum values.)}
    \label{fig:quadratics}
\end{figure*}

\paragraph{Related work.}
The seminal work \algname{DANE}~\cite{dane} is the first distributed algorithm that shows communication reduction for minimizing
quadratics when the number of clients is large. 
Follow-up works interpret \algname{DANE} as a preconditioning method.
They provide better analysis and derive faster rates for quadratic losses~\cite{disco,dane-hb,giant}. 
\citet{spag} propose \algname{SPAG}, an accelerated method,
and prove a better uniform concentration bound of the conditioning number
when solving strongly-convex problems. 
All these methods assume the iteration-subproblems
can be solved exactly. (In this work, we study inexact local solvers 
for \algname{DANE}).

On the other hand, federated optimization algorithms lean towards 
solving the subproblems inexactly by taking local updates.
The standard FL method \algname{FedAvg} has been shown ineffectiveness 
for convex optimization in the heterogeneous setting ~\cite{woodworth2020minibatch}.
The celebrated \algname{Scaffold}~\cite{scaffold} adds control variate to cope with the client drift issues in heterogeneous networks. It is also the first work to quantify the usefulness of local-steps for quadratics. 
Afterwards, drift correction is employed in many other works such as
FedDC \cite{feddc} and Adabest \cite{adabest}.
\algname{FedPVR}~\cite{fedpvr} propose to apply
drift correction only to the last layers of neural networks.
\algname{Scaffnew}~\cite{proxskip} first
illustrates the usefulness of taking standard local gradient steps under 
strongy-convexity by using a special choice of control variate. 
More advanced methods with refined analysis and features such as client sampling and compression have 
been proposed~\cite{tighterproxskip,tamuna,5thgeneration,randprox}. 
\algname{FedProx}~\cite{fedprox} adds a proximal regularization term to the local functions to control the deviation between client and server models. \algname{FedDyn}~\cite{feddyn} proposes dynamic 
regularization to improve the convergence of \algname{FedProx}. \algname{FedPD}~\cite{fedpd} applies both drift correction and regularization and design the algorithm
from the primal-dual perspective.

A more closely related line of research focuses on the study of convergence guarantee under relaxed function similarity conditions. SONATA~\cite{sonata}  and Accelerated SONATA~\cite{acc-sonata} provide better communication guarantee than vanilla \algname{GD} under strong convexity. \algname{Accelerated ExtraGradient sliding} improves the rate of \cite{acc-sonata} and allows the use of inexact local solvers. \citet{svrp} work with a more relaxed similarity assumption and propose a centralized stochastic proximal point variance-reduced method called \algname{SVRP}. \algname{Mime} combines drift correction with momentum for solving non-convex problems. Recently, \algname{CE-LGD}~\cite{ce-lgd} proposed to use recursive gradients together with momentum for non-convex distributed optimization.

\section{Problem formulation and background}
\label{sec:backgraound}

In this work, we consider the following distributed optimization problem: \vspace{-1mm}
\[
\min_{\xx\in\R^d}
\biggl[ f(\xx)=\frac{1}{n}\sum_{i=1}^n f_i(\xx) \biggr] \;. 
\]
We focus on the standard federated learning setting where $n$ clients jointly train a model and a central server coordinates the global learning procedure. Specifically, during each communication round $r=0,1,2...$,
\begin{enumerate}[label=(\roman*),itemsep=1pt,topsep=1pt]
    \item the server broadcasts $\xx^r$ to the clients;
    \item each client $i \in [n]$ calls a local solver to compute the next iterate $\xx_{i,r+1}$;
    \item 
    the server aggregates 
    $\xx^{r+1}$ using $(\xx_{i,r+1})_{i=1}^n$.
\end{enumerate}
\textbf{Goal:} We assume that establishing the connection between the server and the clients is expensive. Therefore, the main goal is to reduce the total number of communication rounds required to reach a desired accuracy.

\textbf{Notations:}
We use $\xx_{i}^\star$ and $\xx^\star$ to denote $\argmin_x \{f_i(\xx)\}$ and $\argmin_x \{f(\xx)\}$ if they exist. We assume $f$ is lower bounded and denote the infimum of $f$ by $f^\star$.

\textbf{Vanilla/Centralized GD.} Gradient descent is a standard method that solves the problem. At each communication round $r$, each device computes its gradient $\nabla f_i(\xx^r)$, and performs one local gradient descent step: $\xx_{i,r+1} = \xx^r - \eta \nabla f_i(\xx^r)$. Then the server updates the global model by: $\xx^{r+1} = \frac{1}{n}\sum_{i=1}^n \xx_{i,r+1}$. The update rule can thus be written as:
$\xx^{r+1} = \xx^{r} - \eta \nabla f(\xx^r)$. The convergence of this method is well understood~\cite{nesterov-book}.

\textbf{Client drift.}
\algname{FedAvg} or \algname{Local-(S)GD} approximate the global function $f$ by $f_i$ on each device. Each device $i$ 
runs multiple (stochastic) gradient descent steps initialized at $\xx^r$, i.e. $\xx_{i,r+1} \approx \argmin_x f_i(\xx)$. Then $\xx^{r+1} = \Avg \xx_{i,r+1}$ is aggregated by the 
server.

This method suffers from client drift when $\{f_i\}$ are heterogeneous~\cite{scaffold,analysislocalGD} because the local updates move towards the minimizers of each $f_i$. However, in general, $\Avg \xx_{i}^\star \neq \xx^\star$. 

\textbf{Regularization.}
To address the issue of client drift, \algname{FedProx}~\cite{fedprox} proposes to
add regularization to each local function $f_i$ to control the local 
iterates from deviating from the current global state, i.e.\ 
$\xx_{i,r+1} 
\approx 
\argmin_x \{ f_i(\xx) + \frac{\lambda}{2}||\xx - \xx^r||^2 \}$.
However, the non-convergence issue still exists. Suppose we initialize
at $\xx^\star$ ($\xx^0 = \xx^\star$), and local solvers return the exact minimizers.  We have 
$\nabla f_i(\xx_{i,1}) + \lambda (\xx_{i,1} - \xx^\star) = 0$,
which implies that 
$\xx^1 = \xx^\star - \frac{1}{\lambda} \Avg \nabla f_i (\xx_{i,1})$.
Since $\xx^1$ is typically not equal to $\xx^\star$,
$\xx^\star$ is not a stationary point for \algname{FedProx}. \looseness=-1

\textbf{Drift correction.} \algname{Scaffold}~\cite{scaffold} addresses the non-convergence issues by using drift correction. It approximates the global function by shifting each $f_i$ based on the current 
difference of gradients, i.e. 
$\xx_{i,r+1} \approx \argmin_x\{f_i(\xx) 
        + \lin{\nabla f(\xx^r) - \nabla f_i(\xx^r), \xx}\}$,
and uses \algname{(S)GD} to solve the subproblem.
However, \algname{Scaffold} requires much smaller stepsize than 
\algname{GD} to ensure convergence. Consequently, the overall communication complexity is no better than \algname{GD}. Moreover, 
the subproblem itself is not well-defined. For instance, if some $f_i$ is 
a linear function, then the subproblem has no solution. 
Indeed, \algname{Scaffold} can be viewed as a composite gradient descent method
with stepsize set to be infinity, which we state as below.

\textbf{Regularized drift correction.}
The appropriate solution to the previously mentioned issues is to incorporate both drift correction and regularization, a strategy already applied in 
\algname{DANE}. Here, we provide a more intuitive explanation
and cast it as a composite optimization method.
Note that $f$ can be written as $f = f_i + \psi$ where $\psi := f - f_i$.
As each client has no access to $\psi$, the standard approach is to linearize $\psi$ at $\xx^r$ and perform the composite gradient method:
$\xx_{i,r+1} \approx \argmin\{f_i(\xx) + \lin{\nabla \psi(\xx^r), \xx - \xx^r} + \frac{\lambda}{2}||\xx - \xx^r||^2\}$, where $\lambda > 0$ 
should approximately be set to $||\nabla \psi^2 (\xx^r)||$. 
If $||\nabla \psi^2 (\xx^r)||$ is much smaller than the smoothness
parameter of $f$, or equivalently if $f_i$ and $f$ have similar Hessians,
then this approach can provably converge faster than \algname{GD}.
This motivates us to study the following algorithms which incorporate
regularized drift correction to distributed optimization settings.

\section{DANE+ with regularized drift correction}

In this section, we first describe a  framework that generalizes \algname{DANE}(Algorithm~\ref{Alg:FrameworkDeterministic}) 
for distributed optimization with regularized drift correction. We then 
discuss how it achieves communication reduction when
the functions of clients exhibit certain second-order similarities.

\begin{algorithm}[tb]
\begin{algorithmic}[1]
\STATE {\bfseries Input:} 
$\lambda \ge 0$, 
$\xx^0 \in \R^d$
\FOR{$r=0,1,2...$}
\FOR{\textbf{each client $i\in [n]$ in parallel}} 
\STATE
Update $\hh_{i,r}$
\STATE
Set $
\xx_{i,r+1}
\approx
\argmin_{\xx \in \R^d}
\bigl\{ F_{i,r}(\xx) \bigr\} 
$, where
\STATE
$F_{i,r}(\xx) := f_i(\xx) - \langle \xx, \hh_{i,r} \rangle 
+ \frac{\lambda}{2}||\xx-\xx^r||^2 \;. $
\ENDFOR
\STATE
Set
$
\xx^{r+1}=\text{Average}(\xx_{i,r+1})_{i=1}^n 
$,
\STATE
where $\text{Average}(\cdot)$ denotes an averaging strategy. 
\ENDFOR
\caption{DANE+}
\label{Alg:FrameworkDeterministic}
\end{algorithmic}
\end{algorithm}

\begin{algorithm*}[tb]
\begin{minipage}{\textwidth}
\begin{algorithmic}[1]
\STATE {\bfseries Input:} 
$p \in [0,1]$,
$\lambda, \eta \ge 0$, $\xx^0 \in \R^d$
\STATE 
Set $\xx_{i,0} = \Tilde{\xx}_0 = \xx^0$, 
 $\forall i \in [n]$.
\FOR{$k=0,1,2...$}
\FOR{\textbf{each client $i\in [n]$ in parallel}} 
\STATE
Update $\hh_{i,k}$
\STATE
\hspace{50mm}$\xx_{i,k+1} 
=
\argmin_{\xx \in \R^d}
\bigl\{ F_{i,k}(\xx) \bigr\} \;. $
\footnote{When $F_{i,k}$ is non-convex, it is sufficient to 
            return a point such that
            $||\nabla F_{i,k} (\xx_{i,k+1})||=0$ and 
            $F_{i,k}(\xx_{i,k+1}) \le F_{i,k}(\xx_{i,k})$.}
\STATE

where \vspace*{-2mm} $$F_{i,k}(\xx) 
:= 
f_i(\xx)
-
\lin{ \xx, \hh_{i,k} }
+
\frac{\eta}{2}||\xx-\xx_{i,k}||^2
+ \frac{\lambda}{2}||\xx-\Tilde{\xx}_k||^2 \;.$$ 
\STATE 
Sample $\theta_k \in \{0, 1\}$ 
where $\text{Prob} (\theta_k = 1) = p$
\vspace*{1mm}
\STATE
\hspace{40mm}$
\Tilde{\xx}_{k+1} =
\begin{cases}
     \text{Average}( \xx_{i,k+1} )_{i=1}^n & 
    \text{if }\theta_k = 1 
    \\
    \Tilde{\xx}_{k} & \text{otherwise}  
\end{cases}
\;.
$
\ENDFOR
\ENDFOR
\caption{FedRed: \textbf{Fe}derated optimization framework with \textbf{d}oubly \textbf{Re}gularized \textbf{d}rift correction}
\label{Alg:FedRed}
\end{algorithmic}
\end{minipage}
\end{algorithm*}

\begin{algorithm*}[tb]
\begin{algorithmic}[1]
\STATE 
The same as Algorithm~\ref{Alg:FedRed},
except on line 6 \\
\vspace*{-3mm}
$$F_{i,k}(\xx) 
:= 
f_i(\xx_{i,k})
+
\lin{ \gg_i(\xx_{i,k}), \xx - \xx_{i,k}}
-
\lin{ \xx, \hh_{i,k} }
+
\frac{\eta}{2}||\xx-\xx_{i,k}||^2
+ \frac{\lambda}{2}||\xx-\Tilde{\xx}_k||^2 \;,$$ 
where $\gg_i(\xx, \xi_i)$ is a stochastic 
estimator of $\nabla f_i(\xx)$ for any $\xx \in \R^d$.
\caption{FedRed-(S)GD}
\label{Alg:GDLocalSolver}
\end{algorithmic}
\end{algorithm*}

\textbf{Method.} 
During each communication round $r$ of \algname{DANE+}, 
each client $i$ is assigned a local objective function $F_{i,r}$ 
and returns an approximate solution for it by running a local solver.
The function $F_{i,r}$ has three standard components: 
i) the individual loss function $f_i$, ii) an inner product term 
for drift correction, and iii) a regularizer that controls the 
deviation of the local iterates from the current global model. 
After that, the server applies an averaging strategy to aggregate all the iterates returned by the clients. 

To run the algorithm, we need to choose a control variate $\hh_{i,r}$
and an averaging method.

\citet{scaffold} proposed two choices of control
variates for drift correction. In this work, we study the
properties of the standard one which is defined as follows:
\begin{equation}
\hh_{i, r}
:=
\nabla f_i(\xx^r) - \nabla f(\xx^r)
\;.
\label{eq:ControlVariate2}
\end{equation}

Another choice of control variate is also used in \algname{Scaffnew}.
In this work, we focus on Hessian Dissimilarity and 
defer the studies of that one for \algname{DANE+} to Appendix~\ref{sec:ProofsForControlVariate1}.

There exist various ways of averaging vectors. 
Let $(\xx_i)_{i=1}^n$ be a set of vectors where $\xx_i\in\R^d$.
The standard averaging is defined as follows. 
\begin{equation}
    \text{Avg}(\xx_i)_{i=1}^n := \textstyle \Avg \xx_i \;.
    \nonumber
\end{equation}

In this work, we also work on randomized averaging:
\begin{equation}
    \text{Rand}(\xx_{i})_{i=1}^n := \xx_{\hat{j}} \;.
    \nonumber
\end{equation}
where $\hat{j}$ denotes a random variable that follows a 
fixed probability distribution at each iteration.

Suppose that control variate~\eqref{eq:ControlVariate2} is used, 
each local solver provides the exact solution, and that standard 
averaging is applied, then \algname{DANE+} is equivalent to \algname{DANE}.

\textbf{Communication Reduction.}
As discussed in Section~\ref{sec:backgraound}, regularized drift correction
can reduce the communication given that the functions $\{f_i\}$ are similar.
A reasonable measure is to look at the second-order Hessian 
dissimilarity~\cite{mime,scaffold,ali2023}, which is defined as follows:

\begin{definition}[Bounded Hessian Dissimilarity
\label{df:MaxHessianSimilarity}
\cite{scaffold}]
Let $f_i : \R^d \to \R$ be continuously differentiable for any $i\in[n]$, then $\{ f_i \}$ have $\delta_{B}$-BHD if for any $\xx, \yy \in \R^d$
and any $i \in [n]$,
\begin{equation}
    ||\nabla h_i(\xx) - \nabla h_i(\yy)|| \le \delta_{B}\;. 
    \label{eq:MaxHessianSimilarity}
\end{equation}
where $h_i := f_i - f$, and $f := \frac{1}{n}\sum_{i=1}^n f_i$.
\end{definition}

The standard Bounded Hessian Dissimilarity defined in~\cite{scaffold}
assumes $||\nabla h_i^2(\xx)|| \le \delta_B$ for any $\xx \in \R^d$,
which implies Definition~\ref{df:MaxHessianSimilarity} and is thus slightly
stronger. Here, we only need first-order differentiability.

We also work on a weaker Hessian Dissimilarity definition.
\begin{definition}[Averaged Hessian Dissimilarity~\cite{svrp}]
\label{df:HessianSimilarity}
Let $f_i : \R^d \to \R$ be continuously differentiable for any $i\in[n]$, then $\{ f_i \}$ have $\delta_A$-AHD if for any $\xx, \yy \in \R^d$,
\begin{equation}
    \frac{1}{n}\sum_{i=1}^n 
    ||\nabla h_i(\xx) - \nabla h_i(\yy)||^2 
    \le 
    \delta_A^2 ||\xx - \yy||^2 \;.
    \label{eq:HessianSimilarity}
\end{equation}
where $h_i := f_i - f$ and 
$f := \frac{1}{n}\sum_{i=1}^n f_i$.
\end{definition}

Even if each $f_i$ is twice-continuously differentiable,
Definition~\ref{df:HessianSimilarity} is still weaker than 
assuming $\Avg||\nabla h_i^2(\xx)||\le\delta_A^2$ for any 
$\xx\in\R^d$. The details can be found in
Remark~\ref{rm:HessianDissimilarity}.

Suppose each $f_i$ is $L$-smooth, then $\delta_A \le 2 L$ and $\delta_B \le 2 L$. If further assuming each $f_i$ is convex, then $\delta_A \le L$ and $\delta_B \le L$. 
In practice, we expect the clients to be 
similar and thus $\delta_A \lll L$ and $\delta_B \lll L$~\cite{mime}.
In principle, if each $h_i$ is
$\delta_i$-smooth, then $\delta_A = (\frac{1}{n} \sum_i \delta_i^2)^{1 / 2}$ can potentially be much smaller (up to $\sqrt{n}$ times) than $\delta_B = \max_i \{\delta_i\}$.

\subsection{Theory}
We state the convergence rates of \algname{DANE+} in this section. 
All the proofs can be found in Appendix~\ref{sec:ProofsOfMainResults}.

\textbf{Nonconvex.}
When the local functions $\{f_i\}$ are non-convex, 
it becomes difficult to obtain an approximate minimizer of $F_{i,r}$.
Instead, it is more realistic to assume that each local solver can  converge to an approximate stationary point. In this case, 
the standard averaging may be less effective than randomized averaging, 
especially when $\{f_i\}$ are similar. The intuition is as follows.

Suppose that $f_i = f$ for any $i \in [n]$ . 
Then $F_{i,r}(\xx) \equiv f(\xx)$ if we set $\lambda = 0$. 
Assume that each local solver can compute an exact stationary point.
Then it is more reasonable to update the global model by 
selecting an update $\xx_{i,r+1}$ rather than the standardally averaging all the iterates, as the averaged stationary points can be even
worse. That being said, the standard averaging can still be effective and
more stable if each client uses the same local solver. We leave this 
possibility in the future work.

In the following, we show that it is sufficient to arbitrarily 
choose a model from $(\xx_{i,r+1})_{i=1}^n$ to update the global model.
For instance, $\xx^{r+1} = \text{Rand}(\xx_{i,r+1})_{i=1}^n$. 
In practice, this index set can be predetermined and during each round, only a single client is required to do local updates.

\begin{theorem}
    \label{thm:Main-DANE+NonConvex}
    Consider Algorithm~\ref{Alg:FrameworkDeterministic} with control variate~\eqref{eq:ControlVariate2}. Let the global model be
    updated by choosing an arbitrary local model for each communication round. 
    Let $f_i : \R^d \to \R$ be continuously differentiable for any $i \in [n]$.
    Assume that $\{f_i\}$ have $\delta_B$-BHD. 
    Suppose that the solutions returned by local solvers satisfy $F_{i,r}(\xx_{i,r+1}) \le F_{i,r}(\xx^r)$ and
    $||\nabla F_{i,r}(\xx_{i,r+1})|| \le e_{r+1}$
    with $e_{r+1} \ge 0$
    for any $r \ge 0$ and any $i \in [n]$.
    Let $\lambda = a\delta_B$ with $a>1$
    Then after $R$ communication rounds, we have:
    \begin{equation}
        ||\nabla f(\Bar{\xx}^R)||^2
        \le
        \frac{4(a+1)^2}{(a-1)} \frac{\delta_B (f(\xx^0) - f^\star)}{R}
        + 2\frac{1}{R}\sum_{r=1}^{R} e_r^2 \;.
        \nonumber
    \end{equation}
    where $\Bar{\xx}^R = \arg\min_{\xx\in\{\xx^r\}_{r=0}^R} 
    \{||\nabla f(\xx)||\}$.
\end{theorem}

Theorem~\ref{thm:Main-DANE+NonConvex} gives a convergence guarantee
of \algname{DANE+} for arbitrary functions (not necessarily with Lipschitz gradient) with any
local solvers. The minimal squared gradient norm decreases at a rate of $\cO(\frac{\delta_B (f(\xx^0) - f^\star)}{R})$. 
To reach arbitrary accuracy $\epsilon$, the average errors should be at most $\epsilon$.
We next establish the local computational complexity assuming that the gradient of each~$f_i$ is Lipschitz.

\begin{corollary}
    Consider Algorithm~\ref{Alg:FrameworkDeterministic} with control variate~\eqref{eq:ControlVariate2}. Let the global model be
    updated by choosing an arbitrary local model for each communication round. 
    Let $f_i : \R^d \to \R$ be continuously differentiable 
    and $L$-smooth for any $i \in [n]$.
    Assume that $\{f_i\}$ have $\delta_B$-BHD.
    Suppose each local solver runs the standard gradient descent 
    starting from $\xx^r$ for $\Theta\Bigl( \frac{L}{\delta_B}R \Bigr)$ local steps,
    for any $r \ge 0$,
    and returns the point with the minimum gradient norm.
    Let $\lambda = 2\delta_B$. 
    After $R$ communication rounds, we have:
    \begin{equation}
        ||\nabla f(\Bar{\xx}^R)||^2
        \le 
        \frac{72 \delta_B (f(\xx^0) - f^\star)}{R} \;.
        \nonumber
    \end{equation}
    where $\Bar{\xx}^R = \arg\min_{\xx\in\{\xx^r\}_{r=0}^R} 
    \{||\nabla f(\xx)||\}$.
\end{corollary}
The communication complexity of \algname{DANE+-GD} is reduced 
by a factor of $\frac{L}{\delta_B}$ compared with \algname{GD}. 
However, the total gradient computations to reach $\epsilon$ accuracy
is $R=\Theta\bigl(\frac{\delta_B(f(\xx^0)-f^\star)}{\epsilon}\bigr)$ 
times worse than \algname{GD}.

\textbf{Convex.} In contrast to non-convex problems, convex 
optimization can benefit from the standard averaging. This allows 
us to obtain the following convergence guarantee depending on the 
Averaged Hessian Dissimilarity.

\begin{theorem}
\label{thm:MainDANE+InExactSolutionConvex}
    Consider Algorithm~\ref{Alg:FrameworkDeterministic} with control variate~\eqref{eq:ControlVariate2} and the standard
    averaging. 
    Let $f_i : \R^d \to \R$ be continuously differentiable and 
    $\mu$-convex with $\mu \ge 0$ for any $i \in [n]$.
    Assume that $\{f_i\}$ have $\delta_A$-AHD. 
    In general, suppose that the solutions returned by local 
    solvers satisfy $\Avg ||\nabla F_{i,r}(\xx_{i,r+1})||^2 \le e_{r+1}^2$
    for any $r \ge 0$ with $e_{r+1} \ge 0$.
    Let $\lambda \ge 2\delta_A$. After $R$ communication rounds, we have:
    \begin{equation}
    \begin{split}
        f(\Bar{\xx}^R) &- f(\xx^\star) + \frac{\mu}{2} 
        \Avg ||\xx_{i,R} - \xx^\star||^2 
        \le
        \\
        &\frac{\mu}{(1+\frac{\mu}{\lambda})^R - 1} ||\xx^0 - \xx^\star||^2 
        + \frac{2}{\mu + \lambda} \sum_{r=1}^R e_r^2 \;.
    \end{split}
    \nonumber
    \end{equation}
    where $\Bar{\xx}^R := \argmin_{\xx \in \{\xx^r\}_{r=1}^R} f(\xx)$.
\end{theorem}
Theorem~\ref{thm:MainDANE+InExactSolutionConvex} provides the  
convergence guarantee for \algname{DANE+} when the local functions are convex (and not necessarily smooth). Each local solver can be chosen 
arbitrarily, provided that it can return a solution that satisfies 
a specific accuracy condition. The convergence rate 
is a continuous estimate both in $\mu$ and $\lambda$,
since $\frac{\mu}{(1+\frac{\mu}{\lambda})^R - 1} \le \frac{\lambda}{R}$.
To reach a certain accuracy $\epsilon$, the errors should satisfy
$\sum_{r=1}^R e_r^2 \le \frac{\epsilon (\mu + \lambda)}{2}$. 
We next show that if the solutions returned by local solvers satisfy 
more specific conditions, then the additional error term $\sum_{r=1}^R e_r^2$ 
can be removed.
\begin{theorem}
    \label{thm:MainConvexFrameworkInExactSolutionSpecial}
    Under the same conditions as Theorem~\ref{thm:MainDANE+InExactSolutionConvex},
    suppose that the solutions returned by local solvers satisfy $\sum_{i=1}^n ||\nabla F_{i,r}(\xx_{i,r+1})||^2 \le e_r^2 \sum_{i=1}^n ||\xx_{i,r+1} - \xx^r||^2$ for any
    $r \ge 0$ and $e_r \ge 0$.
    Let $\lambda \ge 2\delta_A$ and let 
    $\sum_{r=0}^{+\infty} e_r^2 \le \frac{\lambda (\mu + \lambda)}{8}$.
    After $R$ communication rounds: 
    \begin{align}
        f(\Bar{\xx}^R) &- f(\xx^\star) + \frac{\mu}{2} 
        \Avg ||\xx_{i,R} - \xx^\star||^2 
        \le
        \nonumber
        \\
        &\frac{\mu}{[(1+\frac{\mu}{\lambda})^R - 1]} ||\xx^0 - \xx^\star||^2 
        \le
        \frac{\lambda}{R} ||\xx^0 - \xx^\star||^2 
        \;.
        \nonumber
    \end{align}
    where $\Bar{\xx}^R := \argmin_{\xx \in \{\xx^r\}_{r=1}^R} f(\xx)$.
\end{theorem}

In practice, it is sufficient to stop running the local solver on each device
once the solutions satisfy $||\nabla F_{i,r} (\xx_{i,r+1})|| \le e_r ||\xx_{i,r+1} - \xx^r||$ with 
$
    e_r^2 := \frac{\lambda (\mu + \lambda)}
    {8(r+1)(r+2)}.
$
As a direct corollary of Theorem~\ref{thm:MainConvexFrameworkInExactSolutionSpecial},
we can estimate the required number of local iterations for standard 
first-order methods assuming that each function $f_i$ is $L$-smooth.

\begin{corollary}
\label{thm:MainDANE+InExactSolutionConvexLocalSteps}
    Consider Algorithm~\ref{Alg:FrameworkDeterministic} with control variate~\eqref{eq:ControlVariate2} and standard averaging. 
    Let $f_i : \R^d \to \R$ be continuously differentiable, 
    $\mu$-convex with $\mu \ge 0$, and $L$-smooth for any $i \in [n]$.
    Assume that $\{f_i\}$ have $\delta_A$-AHD with $L \ge \delta_A + \mu$. 
    Suppose for any $r \ge 0$,
    each device uses the standard (fast) gradient descent 
    initialized at $\xx^r$ until $||\nabla F_{i,r} (\xx_{i,r+1})||^2 \le \frac{\lambda (\mu + \lambda)}
    {8(r+1)(r+2)} ||\xx_{i,r+1} - \xx^r||^2$ is satisfied.
    Let $\lambda = 2\delta_A$.
    To have the same convergence guarantee as in 
    Theorem~\ref{thm:MainConvexFrameworkInExactSolutionSpecial},
    the total number of local steps $K_r$ required at each round $r$
    is no more than:
    
    \algname{DANE+-GD}: \vspace{-2mm}
    \begin{equation}
        K_r \le \Theta\Biggl( 
          \frac{L}{\delta_A+\mu}
          \ln\biggl( \frac{L}{\delta_A+\mu} (r+1) \biggr)
          \Biggr) \;.
          \nonumber
    \end{equation}
\algname{DANE+-FGD}: \vspace{-2mm}
    \begin{equation}
        K_r \le \Theta\Biggl( 
          \sqrt{\frac{L}{\delta_A+\mu}}
          \ln\biggl( \frac{L}{\delta_A+\mu} (r+1) \biggr)
          \Biggr) \;.
          \nonumber
    \end{equation}
\end{corollary}

The total computational complexity of \algname{DANE+-GD} is equivalent 
to that of the centralized gradient descent method, up to a logarithmic factor that depends on $r$.

\section{FedRed: an optimization framework with Doubly Regularized Drift Correction}
While \algname{DANE+} can achieve communication speed-up, the 
local computations are inefficient. This is because, for every communication round, a difficult subproblem has to be approximately solved. In this section, we present a general 
framework that improves the overall computational efficiency while maintaining the communication reduction.

\textbf{Key idea: Let us add an additional regularizer to improve the conditioning of the subproblem.}
\algname{FedRed}~\ref{Alg:FedRed} retains two sequences throughout the iteration: $\{\xx_{i,k}\}$ denote the iterates stored on each device, and $\{\Tilde{\xx}_{k}\}$ are the reference points. At each iteration $k$, 
\algname{FedRed} adds an additional regularizer to the proxy function 
defined in \algname{DANE+}. By setting $\eta$ to be a large number
(e.g. $2L$ where $L$ is the Lipschitz constant),
$F_{i,k}$ can become a function with a good condition number
(an absolute constant) and thus 
can be solved efficiently in a constant (up to the logarithmic factor) number of steps.

However, strong regularization prevents aggressive progress of the algorithm.
Therefore, the communication is only triggered with probability $p$,
with the reference points being adjusted to the averaged $(\xx_{i,k+1})_{i=1}^n$.
Lastly, the control variates $\{\hh_{i,k}\}$ 
are updated accordingly. We study the same control variate
as defined in~\eqref{eq:ControlVariate2}:
\begin{equation}
    \hh_{i, k}
    :=
    \nabla f_i(\Tilde{\xx}_k)
    -
    \nabla f(\Tilde{\xx}_k)
    \;.
    \label{eq:ControlVariateLocal2}
\end{equation}

Note that \algname{DANE+} becomes a special case of \algname{FedRed}
when $p = 1$ and $\eta = 0$. We next show the convergence rates for
\algname{FedRed} using exact local solvers.

\subsection{FedRed with Exact Solvers}

\begin{theorem}
    Consider Algorithm~\ref{Alg:FedRed} with control 
    variate~\eqref{eq:ControlVariateLocal2} and randomized averaging
    with random index set $\{i_k\}_{k=0}^{+\infty}$.
    Let $f_i : \R^d \to \R$ be continuously differentiable for any $i \in [n]$. 
    Assume that $\{f_i\}$ have $\delta_B$-BHD.
    Let $\lambda = \delta_B$, $p = \frac{\lambda}{\eta}$ and 
    $\eta \ge 4 \delta_B$.
    For any $K \ge 1$, it holds that:
    \begin{equation}
        \E\Bigl[ ||\nabla f(\Bar{\xx}_{K})||^2 \Bigr]
        \le 
        \frac{150\eta (f(\xx^0) - f^\star)}{K} 
        \;,
        \nonumber
    \end{equation}
    where $\Bar{\xx}_K$ is uniformly sampled from 
    $(\xx_{i_k,k})_{k=0}^{K-1}$.
\end{theorem}
To reach $\epsilon$-accuracy, 
the communication complexity is $p\cO(\frac{\eta}{K})=
\cO(\frac{\delta_B}{\epsilon})$ in expectation, which is the same as \algname{DANE+}. However, due to 
the additional regularization, the subproblem becomes easier to solve.
The same conclusion can also be derived for the convex settings.
\begin{theorem}
    Consider Algorithm~\ref{Alg:FedRed} with control 
    variate~\eqref{eq:ControlVariateLocal2} and the standard
    averaging.
    Let $f_i : \R^d \to \R$ be continuously differentiable, 
    $\mu$-convex with $\mu \ge 0$ for any $i \in [n]$. 
    Assume that $\{f_i\}$ have $\delta_A$-AHD.
    By choosing $p = \frac{\lambda + \mu/2}{\eta + \mu/2}$ and
                $\eta \ge \lambda \ge \delta_A$, 
    for any $K \ge 1$, it holds that:
    \begin{equation}
    \begin{split}
        &\E[ f(\Bar{\Bar{\xx}}_{K}) - f^\star] 
        + 
        \frac{\mu}{4}
        \E\biggl[ 
        ||\Tilde{\xx}_{K} - \xx^\star||^2
        + 
        \frac{1}{n} \sum_{i=1}^n
        ||\xx_{i,K} \\&- \xx^\star||^2
        \biggr]
        \le 
        \frac{\mu}{2} \frac{||\xx_0 - \xx^\star||^2 }{(1 + \frac{\mu}{2\eta})^K - 1}
        \le
        \frac{\eta ||\xx_0 - \xx^\star||^2}{K} \;.
        \nonumber
    \end{split}
    \end{equation}
    where  
    $
    \Bar{\Bar\xx}_{K} 
    :=
    \sum_{k=1}^K \frac{1}{q^k} \Bar{\xx}_k / \sum_{k=1}^K \frac{1}{q^k}$,
    $\Bar{\xx}_k := \Avg \xx_{i,k}$,
    and $q := 1 - \frac{\mu}{2\eta + \mu}$.
    \nonumber
\end{theorem}

\subsection{Comparison to existing algorithms.}
In this section, we compare \algname{FedRed} with two similar algorithms.
\algname{SVRP}~\cite{svrp} is a stochastic proximal point variance-reduced method. As discussed in the introduction, the main focus of \algname{SVRP} is different in this paper. The main algorithmic differences between \algname{SVRP} and \algname{FedRed} are that: 1) \algname{FedRed} uses double regularization and thus the regularized centering points are different 2) For \algname{SVRP}, communication happens at each iteration , while \algname{FedRed} skips communication with probability $1-p$, and 3) \algname{FedRed} allows standard averaging for convex optimization. Another similar algorithm is \algname{FedPD}~\cite{fedpd} where regularized drift correction and the strategy of skipping communication with probability are used. However, due to the different regularizations, \algname{FedPD} is less efficient in local computations. 

\subsection{FedRed-(S)GD}
Previous theorems assume that the exact minimizers (or stationary points)
of the subproblems are returned at each iteration. In this section, 
we show that, in fact, it suffices to make only one local step for minimizing
smooth functions. 

Note the derivation of Algorithm~\ref{Alg:GDLocalSolver} is to simply
linearize $f_i$ at $\xx_{i,k}$. Therefore, the solution of the subproblem  has a closed form:  
$\xx_{i,k+1} = 
\frac{\eta}{\eta+\lambda} \xx_{i,k}
+
\frac{\lambda}{\eta + \lambda} \Tilde{\xx}_k 
-
\frac{1}{\eta + \lambda}
(\gg_i(\xx_{i,k}) - \hh_{i,k})
$,
which is a convex combination of $\xx_{i,k}$ and $\Tilde{\xx}_k$
followed by a gradient descent step.

\begin{figure*}[tb!]
    \centering
    \includegraphics[width=1\textwidth]{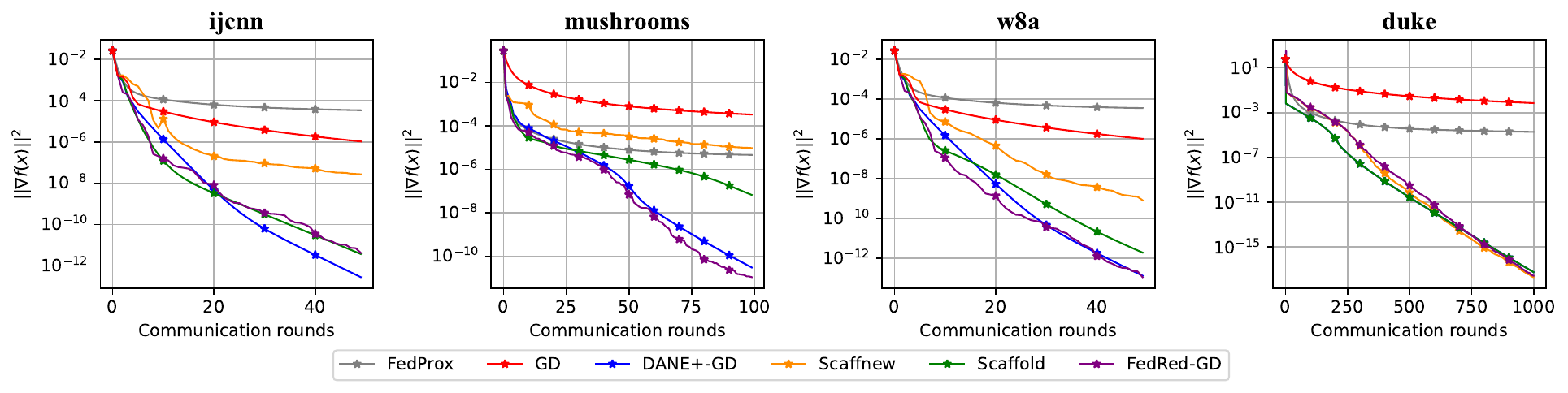}
    \caption{Comparison of \algname{DANE+-GD} and \algname{FedRed-GD} against four other distributed optimizers on four LIBSVM datasets using regularized logistic loss.}
    \label{fig:libsvm-main}
\end{figure*}

Here we allow the use of a stochastic gradient $\gg_i(\xx,\xi_i)$ 
instead of the full gradient $\nabla f_i(\xx)$. 
While more extensions could be made such as approximating 
$\hh_{i,k}$ using stochastic gradients, we do not attempt to be 
exhaustive here as the goal is to illustrate the  
communication reduction. We make the following standard assumptions 
for $\gg_i(\xx, \xi_i)$. 

\begin{assumption}
\label{assump:StochsticGradientfi}
    For any $i \in [n]$, $\gg_i(\xx, \xi_i)$ is an unbiased stochastic 
    estimator of $\nabla f_i (\xx)$ with bounded variance $\sigma^2$
    such that for any $\xx \in \R^d$, it holds that:
    \begin{equation}
        \E[\gg_i(\xx)] = \nabla f_i(\xx),\quad
        \E[||\gg_i(\xx) - \nabla f_i(\xx)||^2] \le \sigma^2. 
    \end{equation}
\end{assumption}

\begin{theorem}
\label{thm:main-FedRed-GDNonconvex}
    Consider Algorithm~\ref{Alg:GDLocalSolver} with control 
    variate~\eqref{eq:ControlVariateLocal2} and randomized averaging
    with random index set $\{i_k\}_{k=0}^{+\infty}$.
    Let $f_i : \R^d \to \R$ be continuously differentiable
    and $L$-smooth for any $i \in [n]$. 
    Assume that $\{f_i\}$ have $\delta_B$-BHD with
    $\delta_B \le L$ and that Assumption~\ref{assump:StochsticGradientfi} holds. 
    By choosing 
    $\eta = 3L + \sqrt{9L^2 + \frac{L \sigma^2 K}{f(\xx^0) - f^\star} } $, 
    $\lambda = \delta_B$
    and $p = \frac{\delta_B}{L}$, for any $K \ge 1$, it holds that:
    \begin{align}
        \E\bigl[||\nabla f(\Bar{\xx}_K)||^2 \bigr]
        &\le 
        \frac{96 L (f (\xx^0) - f^\star)}{K}
        \\
        &+ 24 \sqrt{\frac{L (f(\xx^0) - f^\star)}{K}} \sigma
        \;.
    \end{align}
    where $\Bar{\xx}_K$ is uniformly sampled from 
    $(\xx_{i_k,k})_{k=0}^{K-1}$.
\end{theorem}
According to Theorem~\ref{thm:main-FedRed-GDNonconvex},
\algname{FedRed-GD} (using full-batch gradient) 
requires the same total gradient
computations as gradient descent to reach a certain accuracy.
However, in expectation, it can skip communication for every
$\Theta(\frac{L}{\delta_B})$ steps without additional cost.

\begin{theorem}
    Consider Algorithm~\ref{Alg:GDLocalSolver} with control 
    variate~\eqref{eq:ControlVariateLocal2} and the standard
    averaging.
    Let $f_i : \R^d \to \R$ be continuously differentiable, 
    $\mu$-convex with $\mu \ge 0$ and $L$-smooth for any $i \in [n]$. 
    Assume that $\{f_i\}$ have $\delta_A$-AHD with $\delta_A + \mu \le L$
    and that Assumption~\ref{assump:StochsticGradientfi} holds.
    Let 
    $\lambda = \delta_A$, $p = \frac{\lambda + \mu / 2}{\eta - \mu/2}$, 
    and $\eta > L$. 
    To reach $\epsilon$-accuracy, i.e. 
    $\E[ f(\Bar{\Bar{\xx}}_{K}) - f^\star] \le \epsilon$, 
    by choosing $\eta = \frac{\sigma^2}{\epsilon} + L$,
    the total number of iterations is no more than:
\begin{equation}
    K \le \Biggl\lceil \Bigl( \frac{2L}{\mu} + \frac{2\sigma^2}{\mu \epsilon}\Bigr)
    \ln\Bigl(1 + \frac{\mu||\xx^0 - \xx^\star||^2}{\epsilon} \Bigr) \Biggr\rceil 
    \;.
    \nonumber
\end{equation}
where  
$
\Bar{\Bar\xx}_{K} 
:=
\sum_{k=1}^K \frac{1}{q^k} \Bar{\xx}_k / \sum_{k=1}^K \frac{1}{q^k}$,
$\Bar{\xx}_k := \Avg \xx_{i,k}$,
and $q := 1 - \frac{\mu}{2\eta - \mu}$.
\end{theorem}

Since $\Bigl( \frac{2L}{\mu} + \frac{2\sigma^2}{\mu \epsilon}\Bigr)
    \ln\Bigl(1 + \frac{\mu||\xx^0 - \xx^\star||^2}{\epsilon} \Bigr) 
    \le 
    \frac{2L ||\xx^0 - \xx^\star||^2}{\epsilon} 
    +
    \frac{2\sigma^2 ||\xx^0 - \xx^\star||^2}{\epsilon^2}$,
the estimate of $K$ is continuous in $\mu$. 
Let us suppose $\sigma = 0$ and $\mu = 0$ for simplicity.
\algname{FedRed-GD} achieves the same computational complexity $K=\cO(\frac{L}{\epsilon})$ as \algname{GD}. However, the 
communication complexity is $pK = \cO(\frac{\delta_A}{\epsilon})$
which is $\frac{L}{\delta_A}$ times faster than \algname{GD}.
The same conclusion remains effective where $\mu > 0$.

\textbf{Discussion.} Corollary~\ref{thm:MainDANE+InExactSolutionConvexLocalSteps} shows 
\algname{DANE+-FGD} achieves a faster local convergence rate than 
\algname{DANE+-GD}. We suspect the fast gradient method can also improve complexity for  \algname{FedRed}. We leave this potential
as a future work.

\section{Experiments}
In this section, we illustrate the main theoretical properties of our studied methods in numerical experiments on both simulated and real datasets. 

\textbf{Synthetic data.}
We consider the minimization problem of the form:
$f(\xx)=\Avg f_i(\xx)$ where 
$f_i(\xx) := \frac{1}{m} \sum_{j=1}^{m} 
\frac{1}{2}(\xx - \bb_{i,j})^T A_{i,j} (\xx - \bb_{i,j}) 
+ \beta \sum_{k=1}^d \frac{[\xx]_k^2}{1 + [\xx]_k^2 }$, $\bb_{i,j} \in \R^d$, $A_{i,j} \in \R^{d \times d}$, and $[\cdot]_k$ is an indexing operation of a vector. We set $\beta = 0$ for convex problems and
$\beta = 400$ for the non-convex case.
We further use $n=5$, $m=10$, and $d=1000$. Note that 
$\delta_A = \sqrt{\Avg ||\Bar{A}_i - \Bar{A}||^2}$ and 
$\delta_B = \max_i\{||\Bar{A}_i - \Bar{A}||\}$ where 
$\Bar{A}_i:= \frac{1}{m} \sum_{j=1}^{m} A_{i,j}$,\ and
$\Bar{A}:= \Avg \Bar{A}_i$. 
We generate $\{A_{i,j}\}$ such
that $\max_{i,j}\{||A_{i,j}||\} = 100$ and 
$\delta_A \approx \delta_B\approx 5$. 
We generate a strongly-convex instance by further controlling
the minimum eigenvalue of $A_{i,j}$ to be $1$, and a general convex instance by setting some of the eigenvalues to small values close
to zero (while ensuring that each $A_{i,j}$ is positive 
semi-definite). For the non-convex instance, we leave each $A_{i,j}$
as an indefinite matrix. By these
constructions, we have that 
$\frac{L}{\delta_A}\approx\frac{L}{\delta_B} \gtrsim 20$.

We compare \algname{DANE+-GD} and \algname{FedRed-GD} against the vanilla 
gradient descent method. We compute approximate solutions for \algname{DANE+-GD} by running local gradient descent until certain stopping criteria are reached. We use the constant probability ($\approx 0.05$) schedule for \algname{FedRed-GD}. 
Lastly, we set the same step size for all three methods. 
In Figure~\ref{fig:quadratics}, we can observe that 
both \algname{DANE+-GD} and \algname{FedRed-GD} require approximately $20$ 
times fewer communication rounds than \algname{GD} to reach the same accuracy. More importantly, the total cost of gradient 
computation for \algname{FedRed-GD} is at the same scale as \algname{GD},
demonstrating the usefulness of local steps and validating the theory.

\textbf{Binary classification on LIBSVM datasets.}
We experiment with the binary classification task on four
real-world LIBSVM datasets~\cite{libsvm}. We use the 
standard regularized logistic loss: 
$f(\xx)=\Avg f_i(\xx)$ with 
$
f_i(\xx) 
:= 
\frac{n}{M}\sum_{j=1}^{m_i} 
\log(1 + \exp(-y_{i,j} \aa_{i,j}^T \xx))
+
\frac{1}{2 M}
||\xx||^2
$
where $(\aa_{i,j},y_{i,j})\in\R^{d+1}$ 
are feature and labels and $M := \sum_{i=1}^n{m_i}$ is the total number of data points in the training dataset.
We use $n=5$ and split the dataset according to the 
Dirichlet distribution. We benchmark against popular 
distributed algorithms including \algname{Scaffold}~\cite{scaffold}, \algname{Scaffnew}~\cite{proxskip}, \algname{Fedprox}~\cite{fedprox}, and \algname{GD}. We use control variate~\eqref{eq:ControlVariateLocal2} for Scaffold.
We perform grid search to find the best hyper-parameters
for each algorithm including the number of local steps and the stepsizes. From Figure~\ref{fig:libsvm-main}, we observe that 
\algname{DANE+-GD} and \algname{FedRed-GD} consistently achieve 
fast convergence due to the implicit similarities of the objective functions among the workers.

\textbf{Deep learning tasks.} We defer the study of our proposed 
algorithms for image classification tasks using ResNet-18~\cite{resnet} to Appendix~\ref{sec:DLexperiments}. More experiments with different neural network structures on various datasets need to be further investigated. 

\textbf{Practical choices of hyper-parameters.}
For \algname{FedRed}, the main theories suggest that $p \sim \frac{\delta_A}{L}$ for convex problems and $p \sim \frac{\delta_B}{L}$ for non-convex instances. The stepsize $\eta$ should be of order $L$. Hence, we can always fix $\lambda = p\eta$ which has the same order as the similarity constant.

\section{Conclusion}
We propose a new federated optimization framework that 
simultaneously  achieves both communication reduction and 
efficient local computations. Interesting future directions may include: theoretical analysis for client sampling and inexact
control variates, extensions to the decentralized settings,
accelerated version of \algname{FedRed}, 
a more comprehensive study of \algname{FedRed} and \algname{DANE+}
for deep learning.

\section*{Acknowledgments}
We acknowledge partial funding from a Google Scholar Research Award. 

\newpage

\bibliographystyle{plainnat}
\bibliography{reference}

\appendix
\numberwithin{equation}{section}
\numberwithin{figure}{section}
\numberwithin{table}{section}

\newpage
\onecolumn
{\Huge\textbf{Appendix}}
\vspace{0.5cm}
\small

\section{Technical Preliminaries}
\label{sec:TechnicalPreliminaries}

\subsection{Basic Definitions}
We use the following definitions throughout the paper.

\begin{definition}[Convexity]
A differentiable function $f:\R^d\to\mathbb{R}$ is $\mu$-convex with $\mu \ge 0$ if\;$\forall\; \xx,\yy\in\R^d$,
\begin{equation}
    f(\yy)
    \ge 
    f(\xx) + \lin{\nabla f(\xx),\yy-\xx} + \frac{\mu}{2}||\xx-\yy||^2 \;.
    \label{df:stconvex}
\end{equation}
\end{definition} 

\begin{definition}[$L$-smooth]
Let function $f:\R^d\to\mathbb{R}$ be differentiable. $f$ is smooth if there exists $L\ge0$ such that $\forall\;\xx,\yy\in\R^d$,
\begin{equation}
    ||\nabla f(\xx)-\nabla f(\yy)||\le L||\xx-\yy||\;.
    \label{df:smooth}
\end{equation}
\end{definition}

\subsection{Useful Lemmas}
\label{sec:UsefulLemmas}
We frequently use the following helpful lemmas for the proofs. 
\begin{lemma}
    Let $\xx,\yy\in \R^d$. For any $\gamma > 0$, we have:
    \begin{equation}
        -\frac{1}{2\gamma} ||\xx||^2 
        - \frac{\gamma}{2}||\yy||^2
        \le
        \lin{\xx, \yy} \le \frac{1}{2\gamma} ||\xx||^2 
        + \frac{\gamma}{2}||\yy||^2 \;,
        \label{eq:BasicInequality1}
    \end{equation}
    \begin{equation}
        ||\xx + \yy||^2 \le (1 + \gamma) ||\xx||^2 
        + \Bigl( 1 + \frac{1}{\gamma} \Bigr) ||\yy||^2\;.
        \label{eq:BasicInequality2}
    \end{equation}
\end{lemma}

\begin{lemma}
    Let $\xx \in \R^d$ and $a > 0$. For any $\yy \in \R^d$, 
    we have:
    \begin{equation}
        \lin{\xx, \yy} + \frac{a}{2}||\yy||^2 \ge -\frac{||\xx||^2}{2a} \;.
        \label{eq:QuadraticLowerBound}
    \end{equation}
\end{lemma}
\begin{proof}
    The claim follows from the first-order optimality condition. 
\end{proof}

\begin{lemma}
    \label{thm:SquareRootRecurrence}
    Let $\{x_r\}_{r=0}^{+\infty}$ be a non-negative sequence such that 
    $x_{r+1} - x_r \le a_{r+1} \sqrt{x_{r+1}}$ for any $r \ge 0$,
    where $a_i \ge 0$ for any $i \ge 1$. 
    Then for any $R \ge 1$, we have:
    \begin{equation}
        x_R 
        \le 
        \Bigl( \sqrt{x_0} + \sum_{r = 1}^R a_r \Bigr)^2
        \le
        2 x_0 + 2 \Bigl( \sum_{r=1}^R a_r \Bigr)^2 \;.
    \end{equation}
\end{lemma}
\begin{proof}
    Indeed, for any $r \ge 0$, we have:
    \begin{equation}
        x_{r+1} - x_r 
        = 
        (\sqrt{x_{r+1}} - \sqrt{x_{r}}) 
        (\sqrt{x_{r+1}} + \sqrt{x_{r}}) 
        \ge 
        (\sqrt{x_{r+1}} - \sqrt{x_{r}}) \sqrt{x_{r+1}}
        \;.
    \end{equation}
    It follows that:
    \begin{equation}
        \sqrt{x_{r+1}} \le \sqrt{x_r} + a_{r+1} \;.
    \end{equation}
    Summing up from $r=0$ to $R-1$, we get:
    \begin{equation}
        \sqrt{x_R} \le \sqrt{x_0} + \sum_{r=1}^R a_{r} \;.
    \end{equation}
    Taking the square on both sides, we have:
    \begin{equation}
        x_R 
        \le 
        \Bigl( \sqrt{x_0} + \sum_{r = 1}^R a_r \Bigr)^2
        \le
        2 x_0 + 2 \Bigl( \sum_{r=1}^R a_r \Bigr)^2 \;. 
        \qedhere
    \end{equation}
\end{proof}

\begin{lemma}
    \label{thm:StrongConvexityRecurrence}
    Let $q > 0$ and
    let $(F_k)_{k=1}^{+\infty}$, $(D_k)_{k=1}^{+\infty}$ be two non-negative
    sequences such that: $F_{k+1} + D_{k+1} \le q D_k$ for any $k \ge 0$.
    Then for any $K \ge 1$, it holds that:
    \begin{equation}
        \frac{1}{S_k} 
        \sum_{k=1}^K \frac{F_{k}}{q^k}
        +
        \frac{1 - q}{1 - q^K} D_K
        \le 
        \frac{1 - q}{\frac{1}{q^K} - 1} D_0 \;,
    \end{equation}
    where $S_k := \sum_{k=1}^K \frac{1}{q^k}$.
    \begin{proof}
        Indeed, for any $k \ge 0$, we have:
        \begin{equation}
            \frac{F_{k+1}}{q^{k+1}} + \frac{D_{k+1}}{q^{k+1}} \le \frac{D_k}{q^k} \;.
        \end{equation}
        Summing up from $k = 0$ to $K - 1$, we get:
        \begin{equation}
            \sum_{k=1}^K \frac{F_{k}}{q^k} + \frac{D_K}{q^K} \le D_0 \;.
        \end{equation}
        Dividing both sides by $\sum_{k=1}^{K} \frac{1}{q^k}$ and using the 
        fact that $\sum_{k=1}^{K} \frac{1}{q^k} = \frac{\frac{1}{q^K} - 1}{1 - q}$,
        we obtain:
        \begin{equation}
            \frac{1 - q}{\frac{1}{q^K} - 1} 
            \sum_{k=1}^K \frac{F_{k}}{q^k}
            +
            \frac{1 - q}{1 - q^K} D_K
            \le 
            \frac{1 - q}{\frac{1}{q^K} - 1} D_0 \;.
        \end{equation}
    \end{proof}
\end{lemma}

\begin{lemma}
    Let $\{x_i\}_{i=1}^n$ be a set of vectors in $\R^d$ and let 
    $\vv \in \R^d$ be an arbitrary vector. It holds that:
    \begin{equation}
        \Avg ||\xx_{i} - \vv||^2 \
        =
        ||\xx - \vv||^2 + \Avg ||\xx_{i} - \xx||^2 ,\quad
        \text{with} \quad
        \xx := \Avg \xx_i \;.
        \label{eq:AverageOfSquaredDifference}
    \end{equation}
\end{lemma}
\begin{proof}
    Indeed,
    \begin{align}
        \Avg ||\xx_{i} - \vv||^2
        &=
        \Avg ||\xx_{i} - \xx + \xx - \vv||^2
        \\
        &= 
        \Avg \Bigl[ ||\xx_{i} - \xx||^2 + 
        2\lin{\xx_{i} - \xx, \xx - \vv} + ||\xx - \vv||^2 \Bigr]
        \\
        &=
        ||\xx - \vv||^2 + \Avg ||\xx_{i} - \xx||^2 \;.
        \qedhere
    \end{align}
\end{proof}

\begin{lemma}[\citet{nesterov-book}, Lemma 1.2.3]
    Smoothness~\eqref{df:smooth} implies that there exists a quadratic upper bound on f:
\begin{equation}
    f(\yy) \le f(\xx)+\lin{\nabla f(\xx),\yy-\xx}|+\frac{L}{2}||\yy-\xx||^2, \; \forall\xx,\yy\in\R^d \;. \label{eq:SmoothUpperBound}
\end{equation}
\end{lemma}

\begin{lemma}[\citet{nesterov-book}, Theorem 2.1.12]
    Suppose a function $f$ is $L$-smooth and $\mu$-convex, then it holds that:
\begin{equation}
    \lin{\nabla f(\xx) - \nabla f(\yy), \xx - \yy}
    \ge 
    \frac{\mu L}{\mu + L}||\xx -\yy||^2 
    + 
    \frac{1}{\mu + L} ||\nabla f(\xx) - \nabla f(\yy)||^2,
    \qquad
    \forall \xx, \yy \in \R^d.
\label{eq:ConvexSmoothInnerProductLowerBound}
\end{equation}
\end{lemma}

\begin{lemma}[Hessian Dissimilarity]
\label{thm:SecondOrderToFirstOrderAHD}
Let $h_i : \R^d \to \R$ be twice continuously differentiable for any $i\in[n]$. Suppose for any $\xx \in \R^d$ and any $p \ge 1$, 
it holds that:
\begin{equation}
    \frac{1}{n} \sum_{i=1}^n
    ||\nabla^2 h_i(\xx)||^p \le \delta^p\;. 
    \label{eq:HessianSimilarityOld}
\end{equation}
Then for any $\xx, \yy \in \R^d$, we have:
\begin{equation}
    \frac{1}{n}\sum_{i=1}^n 
    ||\nabla h_i(\xx) - \nabla h_i(\yy)||^p 
    \le 
    \delta^p ||\xx - \yy||^p\;.
    \label{eq:LemmaOfHessianDissimilarity}
\end{equation}
\end{lemma}

\begin{proof}
    Since $h_i$ is twice continuously differentiable, by Taylor’s Theorem, we have:
    \begin{equation}
        \nabla h_i(\yy) 
        = \nabla h_i(\xx) 
            + 
            \int_0^1 \nabla^2 h_i(\xx + t (\yy - \xx)) [\yy - \xx] dt \;. 
    \end{equation}
    It follows that:
    \begin{align}
        ||\nabla h_i(\yy) - \nabla h_i(\xx)||^p
        &= 
        \Bigl|\Bigl| \int_0^1 \nabla^2 h_i(\xx + t (\yy - \xx)) [\yy - \xx] dt \Bigl|\Bigl|^p
        \\
        &\le
        \biggl[
        \int_0^1 ||\nabla^2 h_i(\xx + t (\yy - \xx))|| ||\yy - \xx|| dt 
        \biggr]^p
        \\
        &=
        ||\yy - \xx||^p 
        \biggl[
        \int_0^1 ||\nabla^2 h_i(\xx + t (\yy - \xx))|| dt 
        \biggr]^p 
        \\
        &\le 
        ||\yy - \xx||^p\int_0^1 
        ||\nabla^2 h_i(\xx + t (\yy - \xx))||^p dt \;.
    \end{align}
    where in the last inequality, we use the Jensen's inequality
    and the fact that $\tau \mapsto \tau^p$ is a convex function for $p \geq 1$.
    
    By our assumption that $\frac{1}{n}\sum_{i=1}^n || \nabla^2 h_i (\zz) ||^p \le \delta^p$ for any $\zz \in \R^d$, we get:
    \begin{align}
        \Avg ||\nabla h_i(\yy) - \nabla h_i(\xx)||^p
        &\le 
        ||\yy - \xx||^p \int_0^1 
        \Avg ||\nabla^2 h_i(\xx + t (\yy - \xx))||^p dt  
        \\
        &\le
        \delta^p ||\yy - \xx||^p\;. 
        \qedhere
    \end{align}
\end{proof}

\begin{remark}
\label{rm:HessianDissimilarity}
According to Lemma~\ref{thm:SecondOrderToFirstOrderAHD}, 
if each $f_i$ is twice continuously differentiable and satisfies:
$\Avg ||\nabla^2 f_i(\xx) - \nabla f^2 (\xx)||^2 \le \delta_A ^2$,
for any $\xx \in \R^d$, then $\{f_i\}$ have $\delta_A$-AHD. 
However the reverse does not hold. 
Let $n = 3$ and 
$h_i := f_i - f := \frac{1}{2} \xx^T \mA_i \xx$, where $\xx \in \R^2$
and
\begin{equation}
    \mA_1 = \begin{pmatrix} 3 & 0 \\ 0 & 2 \end{pmatrix} \quad
    \mA_2 = \begin{pmatrix} -1 & 0 \\ 0 & 1 \end{pmatrix} \quad 
    \mA_3 = \begin{pmatrix} -2 & 0 \\ 0 & -3 \end{pmatrix} \quad \;.
\end{equation}
It holds that:
\begin{equation}
    \Avg ||\nabla h_i (\xx) - \nabla h_i (\yy)||^2 
    = \frac{9+1+4}{3}||\xx - \yy||^2 = \frac{14}{3} ||\xx - \yy||^2\;,
    \forall \xx, \yy \in \R^2 \;.
\end{equation}
On the other hand,
\begin{equation}
    \Avg ||\nabla^2 h_i (\xx)||^2 
    = \frac{9+1+9}{3} = \frac{19}{3}, \quad
    \forall \xx \in \R^2 \;.
    \qedhere
\end{equation}
\end{remark}

\section{Proofs of main results}
\label{sec:ProofsOfMainResults}

\subsection{Convex results and proofs for Algorithm~\ref{Alg:FrameworkDeterministic} with control variate~\eqref{eq:ControlVariate2}}

\begin{lemma}
    Consider Algorithm~\ref{Alg:FrameworkDeterministic}. Let $f_i : \R^d \to \R$ be $\mu$-convex with $\mu \ge 0$ for any $i \in [n]$. Then for any
    $r \ge 0$, and any $\xx \in \R^d$, we have: 
    \begin{align}
        f_i(\xx) 
        - \lin{\xx, \hh_{i,r}}
        + \frac{\lambda}{2} ||\xx - \xx^r||^2 
        \ge 
        f_i(\xx_{i,r+1}) 
        &- \lin{\xx_{i,r+1}, \hh_{i,r}}
        + \frac{\lambda}{2} ||\xx_{i,r+1} - \xx^r||^2
        \\ 
        &+
        \lin{\nabla F_{i,r}(\xx_{i,r+1}), \xx - \xx_{i,r+1}}
        + 
        \frac{\mu + \lambda}{2} 
        ||\xx_{i,r+1} - \xx||^2 \; .
    \end{align}
    \label{thm:OneStepRecurrenceOfConvexFramework}
\end{lemma}
\begin{proof}
    Let
    $
    F_{i,r} (\xx)
    := 
    f_i(\xx) - \lin{\xx, \hh_{i,r}} 
    + \frac{\lambda}{2} ||\xx - \xx^r||^2
    $. Since $F_{i,r}$ is $(\mu + \lambda)$-convex, we have:
    \begin{align}
        F_{i,r}(\xx)
        &\stackrel{\eqref{df:stconvex}}{\ge} 
        F_{i,r}(\xx_{i,r+1}) 
        + 
        \lin{\nabla F_{i,r}(\xx_{i,r+1}), \xx - \xx_{i,r+1}}
        +
        \frac{\mu+\lambda}{2} ||\xx_{i,r+1} - \xx||^2 \;.
    \end{align} 
    Plugging in the definition of $F_{i,r}$, we get the claim.
\end{proof}

\begin{lemma}
    Consider Algorithm~\ref{Alg:FrameworkDeterministic} with control variate~\eqref{eq:ControlVariate2} and 
    the standard averaging.
    Let $f_i : \R^d \to \R$ be continuously differentiable and $\mu$-convex with $\mu \ge 0$ for any $i \in [n]$. Assume that $\{f_i\}$ have $\delta_A$-AHD. Let $\lambda \ge \frac{\delta_A}{1-c}$ for some $c \in [0,1)$.
    Then for any $\xx \in \R^d$ and any $r \ge 0$, we have:
    \begin{equation}
    \begin{split}
        f(\xx^{r+1}) - f(\xx)
        +
        \frac{\mu + \lambda}{2} \Avg ||\xx_{i,r+1} - \xx||^2
        &+
        \frac{c\lambda}{2}\Avg ||\xx_{i,r+1} - \xx^r||^2
        \\
        &\le 
        \frac{\lambda}{2} ||\xx^r - \xx||^2 
        +
        \Avg \lin{\nabla F_{i,r}(\xx_{i,r+1}), \xx_{i,r+1} - \xx}
        \;.
    \end{split}
    \end{equation}
\label{thm:BasicRecurrenceControlVariate2}
\end{lemma}

\begin{proof}
    According to Lemma~\ref{thm:OneStepRecurrenceOfConvexFramework}, 
    we have: 
    \begin{align}
        f_i(\xx) 
        + \lin{\xx, \nabla f(\xx^r) - \nabla f_i(\xx^r)}
        &+ \frac{\lambda}{2} ||\xx - \xx^r||^2 
        \ge 
        f_i(\xx_{i,r+1}) 
        + \lin{\xx_{i,r+1}, \nabla f(\xx^r) - \nabla f_i(\xx^r)}
        \\
        &+ \frac{\lambda}{2} ||\xx_{i,r+1} - \xx^r||^2 
        + 
        \lin{\nabla F_{i,r}(\xx_{i,r+1}), \xx - \xx_{i,r+1}}
        + 
        \frac{\mu + \lambda}{2} 
        ||\xx_{i,r+1} - \xx||^2 \; .
    \end{align}
    By $\mu$-convexity of $f_i$, we further get:
    \begin{align}
        f_i(\xx) 
        + \lin{\xx, \nabla f(\xx^r) - \nabla f_i(\xx^r)}
        &+ \frac{\lambda}{2} ||\xx - \xx^r||^2 
        \stackrel{\eqref{df:stconvex}}{\ge} 
        f_i(\xx^{r+1}) 
        + \lin{\nabla f_i(\xx^{r+1}), \xx_{i,r+1} - \xx^{r+1}}
        \\
        &+
        \frac{\mu}{2} ||\xx_{i,r+1} - \xx^{r+1}||^2
        + \lin{\xx_{i,r+1}, \nabla f(\xx^r) - \nabla f_i(\xx^r)}
        + \frac{\lambda}{2} ||\xx_{i,r+1} - \xx^r||^2 
        \\
        &+
        \lin{\nabla F_{i,r}(\xx_{i,r+1}), \xx - \xx_{i,r+1}}
        + 
        \frac{\mu + \lambda}{2} 
        ||\xx_{i,r+1} - \xx||^2 \; .
    \end{align}
    Taking the average on both sides over $i=1$ to $n$, we get:
    \begin{equation}
    \begin{split}
        f(\xx) &+ \frac{\lambda}{2} ||\xx - \xx^r||^2 
        \ge 
        f(\xx^{r+1}) 
        + 
        \Avg \Big[ 
        \lin{\nabla f_i(\xx^{r+1}), \xx_{i,r+1} - \xx^{r+1}}
        +
        \lin{\xx_{i,r+1}, \nabla f(\xx^r) - \nabla f_i(\xx^r)}
        \Bigr]
        \\
        &+ 
        \frac{\mu}{2} \Avg ||\xx_{i,r+1} - \xx^{r+1}||^2
        + 
        \frac{\lambda}{2}
        \Avg ||\xx_{i,r+1} - \xx^r||^2 
        +
        \frac{\mu + \lambda}{2} 
        \Avg ||\xx_{i,r+1} - \xx||^2 
        \\
        &\qquad+
        \Avg \lin{\nabla F_{i,r}(\xx_{i,r+1}), \xx - \xx_{i,r+1}} \;.
        \label{eq:BasicRecurrenceConvex1}
    \end{split}
    \end{equation}
    Let $h_i = f_i - f$.
    Note that 
    $
    \Avg \lin{\nabla f(\xx^{r+1}), \xx_{i,r+1} - \xx^{r+1}}
    =
    \Avg \lin{\nabla f(\xx^{r}) - \nabla f_i(\xx^r),\xx^{r+1}}
    = 0
    $. It follows that:
    \allowdisplaybreaks{
    \begin{align}
        &\;\quad \Avg \Bigl[ 
        \lin{\nabla f_i(\xx^{r+1}), \xx_{i,r+1} - \xx^{r+1}}
        +
        \lin{\xx_{i,r+1}, \nabla f(\xx^r) - \nabla f_i(\xx^r)}
        \Bigr]
        \\
        &=
        \Avg \lin{\nabla f_i(\xx^{r+1}) - \nabla f(\xx^{r+1})
        - \nabla f_i(\xx^{r}) + \nabla f(\xx^{r}), \xx_{i,r+1} - \xx^{r+1}}
        \\
        &=
        \Avg \lin{\nabla h_i (\xx^{r+1}) - \nabla h_i (\xx^r), 
        \xx_{i,r+1} - \xx^{r+1}}
        \;.
    \end{align}}
    
    We next split $\frac{\lambda}{2}\Avg ||\xx_{i,r+1} - \xx^r||^2$ 
    into 
    $\frac{c\lambda}{2}\Avg ||\xx_{i,r+1} - \xx^r||^2
     + \frac{(1-c) \lambda}{2}\Avg ||\xx_{i,r+1} - \xx^r||^2$ 
    with $0 \le c \le 1$. We 
    combine the second part with 
    $\Avg \lin{\nabla h_i (\xx^{r+1}) - \nabla h_i (\xx^r), 
        \xx_{i,r+1} - \xx^{r+1}}$ to get:
    \begin{align*}
        &\qquad \frac{(1-c)\lambda}{2} \Avg ||\xx_{i,r+1} - \xx^r||^2
        +
        \Avg \lin{\nabla h_i (\xx^{r+1}) - \nabla h_i (\xx^r), 
        \xx_{i,r+1} - \xx^{r+1}} 
        \\
        &\stackrel{\eqref{eq:AverageOfSquaredDifference}}=
        \frac{(1-c)\lambda}{2}||\xx^{r+1} - \xx^r||^2 
        +
        \Avg \biggl[
        \frac{(1-c) \lambda}{2}||\xx_{i,r+1} - \xx^{r+1}||^2 
        +
        \lin{\nabla h_i (\xx^{r+1}) - \nabla h_i (\xx^r), 
        \xx_{i,r+1} - \xx^{r+1}} 
        \biggr]
        \\
        &\stackrel{\eqref{eq:QuadraticLowerBound}}\ge 
        \frac{(1-c)\lambda}{2}||\xx^{r+1} - \xx^r||^2 
        - \frac{1}{2(1-c)\lambda} \Avg ||\nabla h_i (\xx^{r+1}) - \nabla h_i (\xx^r)||^2
        \\
        &\stackrel{\eqref{eq:HessianSimilarity}}\ge 
        \frac{(1-c)\lambda}{2}||\xx^{r+1} - \xx^r||^2 
        -\frac{\delta_A^2}{2(1-c)\lambda}
        ||\xx^{r+1} - \xx^r||^2 \ge 0.
    \end{align*}
    where in the last inequality, we use the assumption that 
    $\lambda \ge \frac{\delta_A}{1-c}$.
    
    Plugging this inequality into~\eqref{eq:BasicRecurrenceConvex1},
    and dropping the non-negative 
    $\frac{\mu}{2}\Avg ||\xx_{i,r+1} - \xx^{r+1}||^2$,
    we get the claim.
\end{proof}

\begin{corollary}
    \label{thm:ConvexFrameworkControlVariate2ExactSolution}
    Consider Algorithm~\ref{Alg:FrameworkDeterministic} with control variate~\eqref{eq:ControlVariate2} and the standard
    averaging. Let $f_i : \R^d \to \R$ be continuously differentiable and $\mu$-convex with $\mu \ge 0$ for any $i \in [n]$. Assume that $\{f_i\}$ have $\delta$-AHD. 
    Let $\lambda \ge \delta_A$ and
    suppose that each local solver provides the exact solution. 
    Then for any $r \ge 0$, we have:
    \begin{equation}
         f(\xx^{r+1}) - f(\xx^r) 
        \le - \frac{\mu + \lambda}{2} ||\xx^{r+1} - \xx^{r}||^2 \;,
    \end{equation}
    and after $R$ communication rounds, it holds that:
    \begin{equation}
        f(\xx^R) - f(\xx^\star) + \frac{\mu}{2} ||\xx^R - \xx^\star||^2 
        \le
        \frac{\mu}{2[(1+\frac{\mu}{\lambda})^R - 1]} ||\xx^0 - \xx^\star||^2
        \le
        \frac{\lambda}{2R}||\xx^0 - \xx^\star||^2.
    \end{equation}
\end{corollary}

\begin{proof}
    By the assumption that the subproblem is solved exactly, we have 
    $||\nabla F_{i,r}(\xx_{i,r+1})|| = 0$ for any $i \in [n]$ and $r \ge 0$. According to Lemma~\ref{thm:BasicRecurrenceControlVariate2} with $c = 1$,
    we have:
    \begin{align}
        f(\xx^{r+1}) - f(\xx)
        +
        \frac{\mu + \lambda}{2} \Avg ||\xx_{i,r+1} - \xx||^2
        \le 
        \frac{\lambda}{2} ||\xx^r - \xx||^2  \;.
    \end{align}
    with $\lambda \ge \delta_A$.
    It follows that:
    \begin{equation}
        f(\xx^{r+1}) - f(\xx) 
        \stackrel{\eqref{eq:AverageOfSquaredDifference}}\le 
        \frac{\lambda}{2} ||\xx^r - \xx||^2
        -
        \frac{\lambda + \mu}{2} ||\xx^{r+1} - \xx||^2 \;.
    \end{equation}
    Let $\xx = \xx_r$. The function value gap monotonically decreases as:
    \begin{equation}
        f(\xx^{r+1}) - f(\xx^r) 
        \le - \frac{\mu + \lambda}{2} ||\xx^{r+1} - \xx^{r}||^2 \;.
        \label{eq:SufficientDecreaseControlVariate2}
    \end{equation}
     Let $\xx = \xx^\star$. We have:
     \begin{equation}
         \frac{2}{\mu + \lambda} 
         \bigl( f(\xx^{r+1}) - f(\xx^\star) \bigr)
         \le 
         \underbrace{\Bigl( 1 - \frac{\mu}{\mu + \lambda} \Bigr)}_{:=q} ||\xx^r - \xx^\star||^2
         - 
         ||\xx^{r+1} - \xx^\star||^2 \;.
     \end{equation}
     Dividing both sides by $q^{r+1}$, we get:
     \begin{equation}
         \frac{2}{\mu + \lambda} \frac{1}{q^{r+1}}
         \bigl( f(\xx^{r+1}) - f(\xx^\star) \bigr)
         \le 
         \frac{||\xx^r - \xx^\star||^2}{q^r}
         - 
         \frac{||\xx^{r+1} - \xx^\star||^2}{q^{r+1}} \;.
     \end{equation}
     Summing up from $r=0$ to $R-1$, we have:
     \begin{align}
         \biggl[ \frac{2}{\mu + \lambda} 
         \sum_{r=1}^R \frac{1}{q^r} \biggr]
         \bigl( f(\xx^R) - f(\xx^\star) \bigr)
         \stackrel{\eqref{eq:SufficientDecreaseControlVariate2}}{\leq}
         \frac{2}{\mu + \lambda} 
         \sum_{r=1}^R \frac{1}{q^{r}} \bigl( f(\xx^{r}) - f(\xx^\star) \bigr)
         \le 
         ||\xx^0 - \xx^\star||^2 - \frac{||\xx^R - \xx^\star||^2}{q^R} \;.
     \end{align}
     Using the fact that 
     $
     \sum^{R}_{r=1}\frac{1}{q^r} = \frac{\frac{1}{q^R} - 1}{1-q}
     $ 
     and rearranging, we have:
     \begin{align}
        f(\xx^R) - f^\star
        &\le 
        \frac{\mu + \lambda}{2} 
        \frac{\frac{\mu}{\mu+\lambda}}{(1+\frac{\mu}{\lambda})^R -1}
        ||\xx^0 - \xx^\star||^2 
        -
        \frac{\mu+\lambda}{2}\frac{\frac{\mu}{\mu+\lambda}}{1-q^R}
        ||\xx^R - \xx^\star||^2 
        \\
        &\le 
        \frac{\mu}{2[(1 + \frac{\mu}{\lambda})^R - 1]}
        ||\xx^0 - \xx^\star||^2 
        -
        \frac{\mu}{2} ||\xx^R - \xx^\star||^2 \;.
        \qedhere
     \end{align}
\end{proof}

\begin{lemma}
    \label{thm:LemmaConvexFrameworkInExactSolution}
    Consider Algorithm~\ref{Alg:FrameworkDeterministic} with control variate~\eqref{eq:ControlVariate2} and the standard
    averaging. 
    Let $f_i : \R^d \to \R$ be continuously differentiable and $\mu$-convex with $\mu \ge 0$ for any $i \in [n]$.
    Assume that $\{f_i\}$ have $\delta_A$-AHD. 
    After $R$ communication rounds, we have: 
    \begin{equation}
        \frac{2}{\mu+\lambda} \sum_{r=1}^R \frac{f(\xx^r) - f^\star}{q^r}
        + \frac{A_R}{q^R}
        + \frac{\lambda}{2(\mu+\lambda)} \sum_{r=1}^R \frac{B_{r-1}}{q^r}
        \le 
        2||\xx^0 - \xx^\star||^2 + \frac{4}{(\mu + \lambda)^2} 
        \Biggl( \sum_{r=1}^R  
        \sqrt{\frac{C_{r}}{q^{r}}} \Biggr)^2
        \;.
    \end{equation}
    where 
    $A_r := \Avg ||\xx_{i,r} - \xx^\star||^2$, 
    $B_r := \Avg ||\xx_{i,r+1} - \xx^r||^2$, 
    $C_r := \Avg ||\nabla F_{i,r}(\xx_{i,r})||^2$,
    and 
    $q := \frac{\lambda}{\lambda + \mu}$.
\end{lemma}

\begin{proof}
     According to Lemma~\ref{thm:BasicRecurrenceControlVariate2}
    with $\xx = \xx^\star$, for any $r \ge 0$, we have:
    \allowdisplaybreaks{
    \begin{align}
        f(\xx^{r+1}) - f(\xx^\star)
        &+
        \frac{\mu + \lambda}{2} \Avg ||\xx_{i,r+1} - \xx^\star||^2
        +
        \frac{\lambda}{4}\Avg ||\xx_{i,r+1} - \xx^r||^2
        \\
        &\le 
        \frac{\lambda}{2} ||\xx^r - \xx^\star||^2 
        +
        \Avg ||\nabla F_{i,r}(\xx_{i,r+1})|| ||\xx_{i,r+1} - \xx^\star||
        \\ 
        &\stackrel{\eqref{eq:AverageOfSquaredDifference}}\le 
        \frac{\lambda}{2} \Avg ||\xx_{i,r} - \xx^\star||^2 
        +
        \Avg ||\nabla F_{i,r}(\xx_{i,r+1})|| ||\xx_{i,r+1} - \xx^\star||
        \\
        &\le
        \frac{\lambda}{2} \Avg ||\xx_{i,r} - \xx^\star||^2 
        +
        \sqrt{\Avg ||\nabla F_{i,r}(\xx_{i,r+1})||^2}
        \sqrt{\Avg ||\xx_{i,r+1} - \xx^\star||^2} \;,
    \end{align}}
    where in the last inequality, we use the standard Cauchy-Schwarz
    inequality. We now use the simplified notations  and
    divide both sides by $\frac{\mu + \lambda}{2}$ to get:
    \begin{equation}
        \frac{2}{\mu + \lambda} 
        \bigl( f(\xx^{r+1}) - f^\star\bigr)
        + A_{r+1} + \frac{\lambda}{2(\mu + \lambda)} B_r
        \le 
        \underbrace{(1 - \frac{\mu}{\mu + \lambda})}_{:=q} A_r
        +
        \frac{2}{\mu+\lambda}\sqrt{C_{r+1}}\sqrt{A_{r+1}} \;.
    \end{equation}
    Dividing both sides by $q^{r+1}$ and summing up from $r=0$ to $R-1$, 
    we obtain:
    \begin{equation}
        \frac{2}{\mu+\lambda} \sum_{r=1}^R \frac{f(\xx^r) - f^\star}{q^r}
        + \frac{A_R}{q^R}
        + \frac{\lambda}{2(\mu+\lambda)} \sum_{r=1}^R \frac{B_{r-1}}{q^r}
        \le 
        \underbrace{A_0 + \frac{2}{\mu+\lambda} \sum_{r=1}^R \frac{\sqrt{C_r}\sqrt{A_r}}{q^r}}_{:= Q_R}
        \;,
        \label{eq:DiccoMainRecurrenceConvexInexact}
    \end{equation}
    from which we can deduce that $\frac{A_R}{q^R} \le Q_R$ for any $R \ge 1$.
    
    We next upper bound $Q_R$. Let $Q_0 := A_0$.
    By definition, for any $R \ge 0$, it holds that:
    \begin{align}
        Q_{R+1} - Q_{R} 
        &=
        \frac{2}{\mu + \lambda}
        \frac{\sqrt{C_{R+1}} \sqrt{A_{R+1}}}{q^{R+1}}  
        \\
        &\le
        \frac{2}{\mu + \lambda} 
        \frac{\sqrt{C_{R+1}} \sqrt{Q_{R+1}q^{R+1}}}{q^{R+1}}
        \\
        &=
        \frac{2}{\mu + \lambda} 
        \frac{\sqrt{C_{R+1}} \sqrt{Q_{R+1}}}{\sqrt{q^{R+1}}} \;.
    \end{align}
    We now apply Lemma~\ref{thm:SquareRootRecurrence} 
    with $a_{r+1} = \frac{2}{\mu + \lambda} \sqrt{\frac{C_{r+1}}{q^{r+1}}}$.
    For any $R \ge 1$, we get:
    \begin{align}
        Q_{R} 
        \le 
        2Q_0 + \frac{4}{(\mu + \lambda)^2} 
        \Biggl( \sum_{r=1}^R  
        \sqrt{\frac{C_{r}}{q^{r}}} \Biggr)^2
        \;.
    \end{align}
    Plugging this upper bound into~\eqref{eq:DiccoMainRecurrenceConvexInexact},
    we get the claim.
\end{proof}

\begin{theorem}
    \label{thm:ConvexFrameworkInExactSolutionGeneral}
    Consider Algorithm~\ref{Alg:FrameworkDeterministic} with control variate~\eqref{eq:ControlVariate2} and the standard
    averaging. 
    Let $f_i : \R^d \to \R$ be continuously differentiable and $\mu$-convex with $\mu \ge 0$ for any $i \in [n]$.
    Assume that $\{f_i\}$ have $\delta_A$-AHD. 
    In general, suppose that the solutions returned by local solvers satisfy $\Avg ||\nabla F_{i,r}(\xx_{i,r+1})||^2 \le e_{r+1}^2$ for any
    $r \ge 0$ and $e_{r+1} \ge 0$.
    Let $\lambda \ge 2\delta_A$. 
    After $R$ communication rounds, we have:
    \begin{equation}
         f(\Bar{\xx}^R) - f(\xx^\star) + \frac{\mu}{2} 
         \Avg ||\xx_{i,R} - \xx^\star||^2 
        \le
        \frac{\mu}{(1+\frac{\mu}{\lambda})^R - 1} ||\xx^0 - \xx^\star||^2 
        + \frac{2}{\mu + \lambda} \sum_{r=1}^R e_r^2
        \le
        \frac{\lambda}{R} ||\xx^0 - \xx^\star||^2 
        + \frac{2}{\mu + \lambda} \sum_{r=1}^R e_r^2
        \;.
    \end{equation}
    where $\Bar{\xx}^R := \argmin_{\xx \in \{\xx^r\}_{r=1}^R} f(\xx)$.
\end{theorem}

\begin{proof}
    Applying Lemma~\ref{thm:LemmaConvexFrameworkInExactSolution}, 
    dropping the non-negative 
    $\frac{\lambda}{2(\mu+\lambda)} \sum_{r=1}^R \frac{B_{r-1}}{q^r}$,
    and plugging $||\nabla F_{i,r}(\xx_{i,r+1})||^2 \le e_{r+1}^2$
    into the bound, we obtain:
    \begin{align}
        \frac{2}{\mu+\lambda} \sum_{r=1}^R \frac{f(\xx^r) - f^\star}{q^r}
        + \frac{A_R}{q^R}
        &\le 
        2||\xx^0 - \xx^\star||^2 + \frac{4}{(\mu + \lambda)^2} 
       \Biggl( \sum_{r=1}^R  
        \frac{e_r}{(\sqrt{q})^r} \Biggr)^2 
        \\
        &\le 
        2||\xx^0 - \xx^\star||^2  + \frac{4}{(\mu + \lambda)^2} \sum_{r=1}^R e_r^2
        \sum_{r=1}^R \frac{1}{q^r} \;.
    \end{align}
    Define $\Bar{\xx}^R := \argmin_{\xx \in \{\xx^r\}_{r=1}^R} f(\xx)$.
    Dividing both sides by 
    $\sum_{r=1}^R \frac{1}{q^r} = \frac{\frac{1}{q^R} - 1}{1-q}$, we get:
    \begin{equation}
        \frac{2}{\mu + \lambda} \bigl( f(\Bar{\xx}^R) - f^\star \bigr)
        +
        (1-q) A_R 
        \le 
        \frac{2 (1-q)}{\frac{1}{q^R}-1} ||\xx^0 - \xx^\star||^2  + 
        \frac{4}{(\mu+\lambda)^2} \sum_{r=1}^R e_r^2  \;.
    \end{equation}
    Plugging in the definition of $q$, we get the claim.
\end{proof}

\begin{theorem}
    \label{thm:ConvexFrameworkInExactSolutionSpecial}
    Consider Algorithm~\ref{Alg:FrameworkDeterministic} with control variate~\eqref{eq:ControlVariate2} and the standard
    averaging. 
    Let $f_i : \R^d \to \R$ be continuously differentiable and $\mu$-convex with $\mu \ge 0$ for any $i \in [n]$.
    Assume that $\{f_i\}$ have $\delta_A$-AHD. 
    Suppose that the solutions returned by local solvers satisfy $\sum_{i=1}^n ||\nabla F_{i,r}(\xx_{i,r+1})||^2 \le e_r^2 \sum_{i=1}^n ||\xx_{i,r+1} - \xx^r||^2$ for any
    $r \ge 0$ with $e_r \ge 0$.
    Let $\lambda \ge 2\delta_A$ and let 
    $\sum_{r=0}^{+\infty} e_r^2 \le \frac{\lambda (\mu + \lambda)}{8}$.
    After $R$ communication rounds, we have:
    \begin{equation}
         f(\Bar{\xx}^R) - f(\xx^\star) + \frac{\mu}{2} 
         \Avg ||\xx_{i,R} - \xx^\star||^2 
        \le
        \frac{\mu}{[(1+\frac{\mu}{\lambda})^R - 1]} ||\xx^0 - \xx^\star||^2 
        \le
        \frac{\lambda}{R} ||\xx^0 - \xx^\star||^2 
        \;.
    \end{equation}
    where $\Bar{\xx}^R := \argmin_{\xx \in \{\xx^r\}_{r=1}^R} f(\xx)$.

\end{theorem}

\begin{proof}
    Applying Lemma~\ref{thm:LemmaConvexFrameworkInExactSolution}, 
    and plugging 
    $\sum_{i=1}^n ||\nabla F_{i,r}(\xx_{i,r+1})||^2 
    \le e_r^2 \sum_{i=1}^n ||\xx_{i,r+1} - \xx^r||^2$
    into the bound, we obtain:
    \begin{align}
        \frac{2}{\mu+\lambda} \sum_{r=1}^R \frac{f(\xx^r) - f^\star}{q^r}
        + \frac{A_R}{q^R}
        + \frac{\lambda}{2(\mu+\lambda)} \sum_{r=1}^R \frac{B_{r-1}}{q^r}
        &\le 
        2||\xx^0 - \xx^\star||^2 + \frac{4}{(\mu + \lambda)^2} 
        \Biggl( \sum_{r=1}^R e_{r-1}
        \sqrt{\frac{B_{r-1}}{q^{r}}} \Biggr)^2 
        \\
        &\le 
        2||\xx^0 - \xx^\star||^2 + \frac{4}{(\mu + \lambda)^2} 
        \sum_{r=1}^R e_{r-1}^2 
        \sum_{r=1}^R \frac{B_{r-1}}{q^{r}} \;.
    \end{align}
    Let $\frac{4}{(\mu + \lambda)^2} 
        \sum_{r=1}^R e_{r-1}^2 \le \frac{\lambda}{2(\mu+\lambda)}$.
    We get:
    \begin{equation}
        \frac{2}{\mu+\lambda} \sum_{r=1}^R \frac{f(\xx^r) - f^\star}{q^r}
        + \frac{A_R}{q^R} 
        \le 
        2 ||\xx^0 - \xx^\star||^2 \;.
    \end{equation}
    Dividing both sides by 
    $\sum_{r=1}^R \frac{1}{q^r} = \frac{\frac{1}{q^R} - 1}{1-q}$, we get:
    \begin{equation}
        \frac{2}{\mu + \lambda} \bigl( f(\Bar{\xx}^R) - f^\star \bigr)
        +
        (1-q) A_R 
        \le 
        \frac{2 (1-q)}{\frac{1}{q^R}-1} ||\xx^0 - \xx^\star||^2 \;.
    \end{equation}
    Plugging in the definition of $q$, we get the claim.
\end{proof}

\begin{corollary}
Consider Algorithm~\ref{Alg:FrameworkDeterministic} with control variate~\eqref{eq:ControlVariate2} and the standard
    averaging. 
    Let $f_i : \R^d \to \R$ be continuously differentiable, $\mu$-convex with $\mu \ge 0$ and $L$-smooth for any $i \in [n]$.
    Assume that $\{f_i\}$ have $\delta_A$-AHD. 
    Suppose that each local solver returns a solution such that
    $\sum_{i=1}^n ||\nabla F_{i,r}(\xx_{i,r+1})||^2 
    \le 
    \frac{\lambda (\mu+\lambda)}{8 (r+1)(r+2)}
    \sum_{i=1}^n ||\xx_{i,r+1} - \xx^r||^2$ for any $r \ge 0$.
    Let $\lambda = 2\delta_A$.
    Then after $R$ communication rounds, we have:
    \begin{equation}
         f(\Bar{\xx}^R) - f(\xx^\star) + \frac{\mu}{2} 
         \Avg ||\xx_{i,R} - \xx^\star||^2 
        \le
        \frac{\mu}{[(1+\frac{\mu}{2\delta_A})^R - 1]} ||\xx^0 - \xx^\star||^2 
        \le
        \frac{2\delta_A}{R} ||\xx^0 - \xx^\star||^2 
        \;.
    \end{equation}
    where $\Bar{\xx}^R := \argmin_{\xx \in \{\xx^r\}_{r=1}^R} f(\xx)$. 
    
    Suppose each device uses the standard gradient descent initialized
    at $\xx^r$, Then the total number of local steps required at each round $r$
    is no more than:
     \begin{equation}
        K_r
        =
        \Theta\Biggl( \frac{L}{\mu + \lambda} \ln 
        \biggl( \frac{L}{\sqrt{\lambda (\mu + \lambda)}} (r+2) \biggr) 
        \Biggr) \qquad (\algname{DICCO-GD}) \;.
    \end{equation}

    Suppose each device uses the fast gradient descent initialized
    at $\xx^r$, Then the total number of local steps required at each round $r$
    is no more than:
     \begin{equation}
        K_r
        =
        \Theta\Biggl( \sqrt{\frac{L}{\mu + \lambda}} \ln 
        \biggl( \frac{L}{\sqrt{\lambda (\mu + \lambda)}} (r+2) \biggr) 
        \Biggr) \qquad (\algname{DICCO-FGD}) \;.
    \end{equation}
    
\end{corollary}

\begin{proof}
    According to 
    Theorem~\eqref{thm:ConvexFrameworkInExactSolutionSpecial},
    to achieve the corresponding convergence rate, 
    the solutions returned by local solvers should satisfy 
    $\sum_{i=1}^n ||\nabla F_{i,r}(\xx_{i,r+1})||^2 \le e_r^2 \sum_{i=1}^n ||\xx_{i,r+1} - \xx^r||^2$ for any $r \ge 0$,
    with
    $\sum_{r=0}^{+\infty} e_r^2 \le \frac{\lambda (\mu + \lambda)}{8}$.

    Let $\xx_{i,r}^\star := \argmin_\xx\{F_{i,r}(\xx)\}$.
    Note that $F_{i,r}$ is $(\mu + \lambda)$-strongly convex. 
    This gives:
    \begin{align}
        ||\xx_{i,r+1} - \xx^r|| 
        &\ge ||\xx^r - \xx_{i,r}^\star|| - ||\xx_{i,r+1} - \xx_{i,r}^\star||
        \\
        &\ge ||\xx^r - \xx_{i,r}^\star|| - \frac{1}{\mu + \lambda} 
        ||\nabla F_{i,r}(\xx_{i,r+1})|| \;.
    \end{align}
    Therefore, suppose we want to have
    $\sum_{i=1}^n ||\nabla F_{i,r}(\xx_{i,r+1})||^2 \le e_r^2 \sum_{i=1}^n ||\xx_{i,r+1} - \xx^r||^2$, 
    it is sufficient to have:
    \begin{equation}
        ||\nabla F_{i,r}(\xx_{i,r+1})||
        \le 
        e_r \Bigl[
        ||\xx^r - \xx_{i,r}^\star|| - \frac{1}{\mu + \lambda} 
        ||\nabla F_{i,r}(\xx_{i,r+1})|| \Bigr], \qquad
        \forall i \in [n]
        \;,
    \end{equation}
    or equivalently:
    \begin{equation}
        ||\nabla F_{i,r}(\xx_{i,r+1})||
        \le 
        \frac{e_r}{1 + \frac{e_r}{\lambda + \mu}}
        ||\xx^r - \xx_{i,r}^\star||,
        \qquad
        \forall i \in [n]
        \;.
    \end{equation}
    
    By our choice of $e_r=\frac{\lambda (\mu + \lambda)} {8(r+1)(r+2)} $, 
    we have:
    \begin{equation}
        e_r \ge \frac{\sqrt{\lambda (\mu + \lambda})}{4(r+2)},
        \quad \text{and} \quad
        1 + \frac{e_r}{\lambda + \mu} 
        \le 
        1 + \frac{\mu + \lambda}{(r+1)(\mu + \lambda)}
        \le 
        2 \;.
    \end{equation}
    Thus, the accuracy condition is sufficient to be:
    \begin{equation}
        ||\nabla F_{i,r}(\xx_{i,r+1})||^2
        \le 
        \frac{\lambda (\mu + \lambda)}{64 (r+2)^2}
        ||\xx^r - \xx_{i,r}^\star||^2 \;.
        \label{eq:ProximalSolutionAccuracy}
    \end{equation}
    \textbf{DICCO-GD: Local Gradient Descent.}
    Recall that for any $\nu_1$-convex and $\nu_2$-smooth function $g$, initial point $\yy_0$ and any number of iteration steps $k$,
    the standard gradient method has the following convergence guarantee: 
    $||\nabla g(\yy_k)||^2 \le 
    \cO\Bigl( \nu_2^2\exp\bigl( -(\nu_1/\nu_2) k \bigr) 
    ||\yy_0 - \yy^\star||^2 \Bigr)$ 
    where $\yy^\star$ is the minimizer of $g$
    (See Theorem 3.12 \cite{bubeckbook} and use the fact that 
    $g(\yy_k) - g^\star \ge \frac{1}{2\nu_2}||\nabla g(\yy_k)||^2$).
    Now we use this result to compute the required local steps $K_r$ to 
    satisfy~\eqref{eq:ProximalSolutionAccuracy}.
    Recall that $F_{i,r}$ is $(\mu + \lambda)$-convex and $(L + \lambda)$-smooth.
    For any $r \ge 0$, we have that:
    \begin{equation}
        \Theta\Bigl( (L+\lambda)^2 \exp (-\frac{\mu + \lambda}{L + \lambda}K_r) \Bigr) 
        \le 
        \Theta\Biggl( \frac{\lambda(\mu + \lambda)}{(r+2)^2} \Biggr) \;.
    \end{equation}
    This gives:
    \begin{equation}
        K_r
        \ge 
        \Theta\Biggl( \frac{L}{\mu + \lambda} \ln 
        \biggl( \frac{L}{\sqrt{\lambda (\mu + \lambda)}} (r+2) \biggr) 
        \Biggr) \;.
    \end{equation}
    where we use the fact that $\lambda \lesssim L$. 

    \textbf{DICCO-FGD: Local Fast Gradient Descent.}
    Recall that for any $\nu_1$-convex and $\nu_2$-smooth function $g$, initial point $\yy_0$ and any number of iteration steps $k$,
    the fast gradient method has the following convergence guarantee: 
    $||\nabla g(\yy_k)||^2 \le 
    \cO\Bigl( \nu_2^2\exp\bigl( -\sqrt{\nu_1/\nu_2} k \bigr) 
    ||\yy_0 - \yy^\star||^2 \Bigr)$ 
    where $\yy^\star$ is the minimizer of $g$
    (See Theorem 3.18~\cite{bubeckbook} and use the fact that 
    $g(\yy_k) - g^\star \ge \frac{1}{2\nu_2}||\nabla g(\yy_k)||^2$) . 
    Now we use this result to compute the required local steps $K_r$ to 
    satisfy~\eqref{eq:ProximalSolutionAccuracy}.
    Recall that $F_{i,r}$ is $(\mu + \lambda)$-convex and $(L + \lambda)$-smooth.
    For any $r \ge 0$, we have that:
    \begin{equation}
        \Theta\Bigl( (L+\lambda)^2 \exp (-\sqrt{\frac{\mu + \lambda}{L + \lambda}} K_r) \Bigr) 
        \le 
        \Theta\Biggl( \frac{\lambda(\mu + \lambda)}{(r+2)^2} \Biggr) \;.
    \end{equation}
    This gives:
    \begin{equation}
        K_r
        \ge 
        \Theta\Biggl( \sqrt{\frac{L}{\mu + \lambda}} \ln 
        \biggl( \frac{L}{\sqrt{\lambda (\mu + \lambda)}} (r+2) \biggr) 
        \Biggr) \;.
    \end{equation}
    where we use the fact that $\lambda \lesssim L$. 

\end{proof}

\subsection{Non-convex results and proofs for Algorithm~\ref{Alg:FrameworkDeterministic} with control variate~\eqref{eq:ControlVariate2}}

\begin{lemma}
\label{thm:LowerBoundFir}
    Consider Algorithm~\ref{Alg:FrameworkDeterministic}
    with control variate~\eqref{eq:ControlVariate2}. 
    Let $f_i : \R^d \to \R$ be continuously differentiable for any $i \in [n]$.
    Assume that $\{f_i\}$ have $\delta_B$-BHD. 
    For any $r \ge 0$, $i \in [n]$, and any $\xx \in \R^d$, it holds that:
    \begin{equation}
        || \nabla F_{i,r}(\xx) - \nabla f(\xx)|| 
        \le
        (\delta_B + \lambda) ||\xx^r - \xx|| \;.
    \end{equation}
\end{lemma}
\begin{proof}
    Let $h_i := f - f_i$.
    Using the definition of $\nabla F_{i,r} (\xx)$, we get:
    \begin{align}
        ||\nabla F_{i,r} (\xx) - \nabla f(\xx)|| 
        &=
        ||\nabla h_i (\xx^r) - \nabla h_i(\xx) + \lambda (\xx - \xx^r)||
        \\
        &\le 
        ||\nabla h_i(\xx) - \nabla h_i (\xx^r)||
        + \lambda||\xx - \xx^r||
        \\
        &\stackrel{\eqref{eq:MaxHessianSimilarity}}{\le} 
        \delta_B ||\xx - \xx^r|| + \lambda||\xx - \xx^r|| \;.
        \qedhere
    \end{align}
\end{proof}

\begin{lemma}
    \label{thm:TwoPointRelationNonConvexFramework}
    Let $f_i : \R^d \to \R$ be continuously differentiable for any $i \in [n]$.
    Assume that $\{f_i\}$ have $\delta_B$-BHD. 
    For any $i \in [n]$ and any $\yy \in \R^d$, $\xx^r \in \R^d$, we have:
    \begin{equation}
        f_i(\xx^r) + \lin{\nabla h_i (\xx^r), \xx^r - \yy}
        - f_i(\yy) - \frac{\lambda}{2} ||\yy - \xx^r||^2
        \le 
        f(\xx^r) - f(\yy) - \frac{\lambda - \delta_B}{2} ||\yy - \xx^r||^2 \;.
    \end{equation}
    For Algorithm~\ref{Alg:FrameworkDeterministic} with control variate~\ref{eq:ControlVariate2}, 
    the left-hand side is equal to $F_{i,r}(\xx^r) - F_{i,r}(\yy)$.
\end{lemma}
\begin{proof}
    Let $h_i := f - f_i$.
    Using the definition of $\nabla F_{i,r} (\xx)$, we get:
    \begin{align}
        F_{i,r}(\xx^r) - F_{i,r}(\yy) 
        &=
        f_i(\xx^r) + \lin{\nabla h_i (\xx^r), \xx^r - \yy}
        - f_i(\yy) - \frac{\lambda}{2} ||\yy - \xx^r||^2
        \\
        &= f(\xx^r) - f(\yy) + h_i(\yy) - h_i(\xx^r)
        - \lin{\nabla h_i(\xx^r), \yy - \xx^r} 
        - \frac{\lambda}{2} ||\yy - \xx^r||^2
        \\
        &\stackrel{\eqref{eq:MaxHessianSimilarity}}{\le}
        f(\xx^r) - f(\yy) - \frac{\lambda - \delta_B}{2} ||\yy - \xx^r||^2 \;.
    \end{align}
\end{proof}

\begin{lemma}
    \label{thm:LemmaNonConvexFramework}
    Consider Algorithm~\ref{Alg:FrameworkDeterministic} with control variate~\eqref{eq:ControlVariate2}.
    Let the global model be
    updated by choosing an arbitrary local model with an index set
    $(i_r)_{r=0}^{+\infty}$. 
    Let $f_i : \R^d \to \R$ be continuously differentiable for any $i \in [n]$.
    Assume that $\{f_i\}$ have $\delta_B$-BHD. 
    Suppose that the solutions returned by local solvers satisfy $F_{i,r}(\xx_{i,r+1}) \le F_{i,r}(\xx^r)$ for any $i \in [n]$.
    Then for any $r \ge 0$, it holds that:
    \begin{equation}
        \label{eq:SufficientDecreaseNonconvex}
        f(\xx_{i,r+1}) 
        + \frac{\lambda - \delta_B}{2} ||\xx_{i,r+1} - \xx^r||^2 
        \le 
        f(\xx^r) \;.
    \end{equation}
    and thus:
    \begin{equation}
        \label{eq:SufficientDecreaseNonconvex2}
        f(\xx^{r+1}) 
        \le 
        f(\xx^r) \;.
    \end{equation}
\end{lemma}
\begin{proof}
    According to Lemma~\ref{thm:TwoPointRelationNonConvexFramework} with
    $\yy = \xx_{i,r+1}$, we get:
    \begin{equation}
         F_{i,r}(\xx^r) - F_{i,r}(\xx_{i,r+1})  \le 
        f(\xx^r) - f(\xx_{i,r+1}) - \frac{\lambda - \delta_B}{2} ||\xx_{i,r+1} - \xx^r||^2 \;.
    \end{equation}
    By the assumption that 
    the function value decreases locally, we have that:
    \begin{equation}
        f(\xx_{i,r+1}) 
        + \frac{\lambda - \delta_B}{2} ||\xx_{i,r+1} - \xx^r||^2 
        \le 
        f(\xx^r) \;.
        \qedhere
    \end{equation}
    Since the previous display holds for any $i \in [n]$ and 
    $\xx^{r+1} = \xx_{i_{r},r+1}$, we have:
    $f(\xx^{r+1}) \le f(\xx^r)$.
\end{proof}

\begin{theorem}
    \label{thm:NonConvexFramework}
    Consider Algorithm~\ref{Alg:FrameworkDeterministic} with control variate~\eqref{eq:ControlVariate2}. Let the global model be
    updated by choosing an arbitrary local model with an index set
    $(i_r)_{r=0}^{+\infty}$. 
    Let $f_i : \R^d \to \R$ be continuously differentiable for any $i \in [n]$.
    Assume that $\{f_i\}$ have $\delta_B$-BHD. 
    Suppose that the solutions returned by local solvers satisfy $F_{i,r}(\xx_{i,r+1}) \le F_{i,r}(\xx^r)$ and
    $||\nabla F_{i,r}(\xx_{i,r+1})|| \le e_{r+1}$
    for any $r \ge 0$ and any $i \in [n]$.
    Let $\lambda = a\delta_B$ with $a > 1$.
    Then after $R$ communication rounds, we have:
    \begin{equation}
        ||\nabla f(\Bar{\xx}^R)||^2
        \le
        \frac{4(a+1)^2}{(a-1)} \frac{\delta_B (f(\xx^0) - f^\star)}{R}
        + 2\frac{1}{R}\sum_{r=1}^{R} e_r^2 \;.
    \end{equation}
    where $\Bar{\xx}^R = \arg\min_{\xx\in\{\xx^r\}_{r=0}^R} 
    \{||\nabla f(\xx)||\}$.
\end{theorem}

\begin{proof}
    According to Lemma~\ref{thm:LemmaNonConvexFramework}, 
    for any $r \ge 0$ and any $i \in [n]$,
    it holds that
    \begin{equation}
        f(\xx_{i,r+1}) 
        + \frac{\lambda - \delta_B}{2} ||\xx_{i,r+1} - \xx^r||^2 
        \le 
        f(\xx^r) \;.
    \end{equation}
    It remains to lower bound $||\xx_{i,r+1} - \xx^r||^2$. 
    Recall that the solution returned by the local solver satisfies
    $||\nabla F_{i,r}(\xx_{i,r+1})|| \le e_{r+1}$.
    Using Lemma~\ref{thm:LowerBoundFir} with $\xx = \xx_{i,r+1}$,
    we get:
    \begin{align}
        e_{r+1}
        \ge 
        ||\nabla F_{i,r}(\xx_{i,r+1})||
        \ge
        ||\nabla f(\xx_{i,r+1})|| - \delta_B ||\xx^r - \xx_{i,r+1}||
        - \lambda ||\xx_{i,r+1} - \xx^r|| \;.
    \end{align}
    It follows that:
    \begin{equation}
        ||\xx_{i,r+1} - \xx^r||^2
        \ge
        \biggl( \frac{||\nabla f(\xx_{i,r+1})||}{\delta_B + \lambda} 
        - \frac{e_{r+1}}{\delta_B + \lambda} \biggr)^2
        \stackrel{\eqref{eq:BasicInequality1}}{\ge}
        \frac{||\nabla f(\xx_{i,r+1})||^2}{2(\lambda + \delta_B)^2}
        - \frac{e_{r+1}^2}{(\lambda + \delta_B)^2} 
        \;.
    \end{equation}
    Plugging this bound into the previous display, we get, for any $i \in [n]$:
    \begin{equation}
        \frac{\lambda - \delta_B}{4(\lambda + \delta_B)^2}
        ||\nabla f(\xx_{i,r+1})||^2 
        \le f(\xx^r) - f(\xx_{i,r+1}) 
        +\frac{\lambda - \delta_B}{2(\lambda + \delta_B)^2} e_{r+1}^2
        \;.
    \end{equation}
    Substituting $\lambda = a \delta_B$, using the fact that $\xx^{r+1} = \xx_{i_r, r+1}$ and the fact that the previous display holds
    for any $i \in [n]$,
    we obtain:
    \begin{equation}
        \frac{a-1}{4(a + 1)^2\delta_B} 
        ||\nabla f(\xx^{r+1})||^2
        \le 
        f(\xx^r) - f(\xx^{r+1})
        +
        \frac{a-1}{2(a+1)^2} e_{r+1}^2\;.
    \end{equation}
    Summing up from $r=0$ to $r=R-1$,
    and dividing both sides by $R$, we get the claim.
\end{proof}

\begin{corollary}
    Consider Algorithm~\ref{Alg:FrameworkDeterministic} with control variate~\eqref{eq:ControlVariate2}. Let the global model be
    updated by choosing an arbitrary local model with an index set
    $(i_r)_{r=0}^{+\infty}$. 
    Let $f_i : \R^d \to \R$ be continuously differentiable 
    and $L$-smooth for any $i \in [n]$.
    Assume that $\{f_i\}$ have $\delta_B$-BHD.
    Suppose for any $r \ge 0$,
    each local solver runs the standard gradient descent 
    starting from $\xx^r$ until 
    $||\nabla F_{i,r} (\xx_{i,r+1})||^2 
    \le \frac{18 \delta_B (f(\xx^0) - f^\star)}{R}$
    is satisfied.
    Let $\lambda_r = 2\delta_B$. 
    After $R$ communication rounds, we have:
    \begin{equation}
        ||\nabla f(\Bar{\xx}^R)||^2
        \le 
        \frac{72 \delta_B (f(\xx^0) - f^\star)}{R} \;.
    \end{equation}
    where $\Bar{\xx}^R = \arg\min_{\xx\in\{\xx^r\}_{r=0}^R} 
    \{||\nabla f(\xx)||\}$.
    For any $r \ge 0$,
    the required number of local steps $K_r$ at 
    round $r$ is no more than
    $\Theta\Bigl( \frac{L}{\delta_B}R \Bigr)$.
\end{corollary}
\begin{proof}
    According to Theorem~\ref{thm:NonConvexFramework}, for any
    $R \ge 1$, we have:
    \begin{equation}
        ||\nabla f(\Bar{\xx}^R)||^2
        \le
        \frac{36 \delta_B (f(\xx^0) - f^\star)}{R}
        + 2\frac{1}{R}\sum_{r=1}^{R} e_r^2 \;.
    \end{equation}
    For any $1 \le r \le R$, let
    $e_r ^2 \le \frac{18 \delta_B (f(\xx^0) - f^\star)}{R}$.
    We further get:
    \begin{equation}
        ||\nabla f(\Bar{\xx}^R)||^2
        \le
        \frac{72 \delta_B (f(\xx^0) - f^\star)}{R} \;.
    \end{equation}
    We next estimate the number of local steps required to have
    $||\nabla F_{i,r} (\xx_{i,r+1})||^2 
    \le \frac{18 \delta_B (f(\xx^0) - f^\star)}{R}$.
    By running standard gradient descent on $F_{i,r}$ at each round $r$
    for $K_{i,r}$ steps and returning the point with the minimum gradient norm, 
    we get:
    \begin{equation}
        ||\nabla F_{i,r}(\xx_{i,r+1})||^2 
        \le 
        \frac{(\lambda + L) (F_{i,r}(\xx^r) - F_{i,r}^\star)}{K_{i,r}} \;.
    \end{equation}
    where $F_{i,r}^\star = \inf_x\{ F_{i,r}(\xx) \}$.
    According to Lemma~\ref{thm:TwoPointRelationNonConvexFramework},
    for any $\yy \in \R^d$, we have:
    \begin{align}
        F_{i,r}(\xx^r) -  F_{i,r}(\yy)
        &\le
        f(\xx^r) - f(\yy) - \frac{\lambda - \delta_B}{2} ||\yy - \xx^r||^2 
        \\
        &\le 
        f(\xx^r) - f(\yy) \le f(\xx^r) - f^\star \;.
    \end{align}
    It follows that:
    \begin{equation}
        F_{i,r}(\xx^r) - F_{i,r}^\star \le f(\xx^r) - f^\star \;.
    \end{equation}
    According to Lemma~\ref{thm:LemmaNonConvexFramework}, the function value
    monotonically decreases. Hence, we get:
    \begin{equation}
        F_{i,r}(\xx^r) - F_{i,r}^\star 
        \le 
        f(\xx^0) - f^\star, \qquad \text{and} \qquad
        ||\nabla F_{i,r} (\xx_{i,r+1})||^2 
        \le 
        \frac{(\lambda + L) (f(\xx^0) - f^\star)}{K_{i,r}} \;.
    \end{equation}
    Let 
    $\frac{(\lambda + L) (f(\xx^0) - f^\star}{K_{i,r}}
    \le \frac{18 \delta_B (f(\xx^0) - f^\star)}{R}$. 
    We obtain, for any $i \in [n]$:

    \begin{equation}
        K_{i,r} \ge \frac{\lambda + L}{18 \delta_B} R
        = \Theta\Bigl( \frac{L}{\delta_B}R \Bigr) \;.
    \end{equation}
    This concludes the proof.
\end{proof}

\subsection{Algorithm~\ref{Alg:FrameworkDeterministic} with a special control variate under strong convexity}
\label{sec:ProofsForControlVariate1}

In this section, we consider the following choice of control variate:
\begin{equation}
\hh_{i, r+1}
:=
m (\xx^{r+1} - \xx_{i, r+1})
+
\hh_{i, r}
\;.
\label{eq:ControlVariate1}
\end{equation}
with $m \ge 0$ and $\frac{1}{n} \sum_{i=1}^n \hh_{i, 0} = 0$.

\begin{lemma}
\label{thm:BasicRecurrenceControlVariate1}
    Consider Algorithm~\ref{Alg:FrameworkDeterministic} with the standard averaging and control variate~\eqref{eq:ControlVariate1} with
    $m = \lambda$. 
    Let $f_i : \R^d \to \R$ be $L$-smooth and $\mu$-convex with $\mu > 0$. 
    Suppose that $n \ge 2$ and that each local solver provides
    the exact solution, then for any $\xx \in \R^d$, we have:
    \begin{equation}
    \begin{split}
        &||\xx^{r+1} - \xx^\star||^2 
        +
        \frac{2 \mu L}{\lambda (\mu + L)} \Avg ||\xx_{i,r+1} - \xx^\star||^2
        + \frac{1}{\lambda^2} 
        \Avg ||\hh_{i,r+1} - \nabla f_i (\xx^\star)||^2 
        \le 
        ||\xx^r - \xx^\star||^2
        \\
        &\quad+ \frac{1}{\lambda^2} \Avg ||\hh_{i,r} - \nabla f_i(\xx^\star)||^2
        - \frac{1}{\lambda^2} \Avg ||\nabla f_i(\xx_{i,r+1}) - \hh_{i,r}||^2 
        - \frac{2}{\lambda (\mu + L)} 
        \Avg ||\nabla f_i (\xx_{i,r+1}) - \nabla f_i (\xx^\star)||^2 \;.
    \end{split}
    \end{equation}
\end{lemma}
\begin{proof}
    Recall that 
    $
    \hh_{i, r+1}
    :=
    \lambda (\xx^{r+1} - \xx_{i, r+1})
    +
    \hh_{i, r}
    $ and thus we have:
    \begin{align}
        ||\xx^{r+1} - \xx^\star||^2 + \frac{1}{\lambda^2} 
        \Avg ||\hh_{i,r+1} - \nabla f_i (\xx^\star)||^2
        &\stackrel{\eqref{eq:ControlVariate1}}{=}
        ||\xx^{r+1} - \xx^\star||^2 + 
        \Avg ||\xx^{r+1} - \xx_{i,r+1} + \frac{1}{\lambda}
        (\hh_{i,r} - \nabla f_i(\xx^\star))||^2
        \\
        &\stackrel{\eqref{eq:AverageOfSquaredDifference}}{=}
        \Avg ||\xx_{i,r+1} - \xx^\star 
        - \frac{1}{\lambda} (\hh_{i,r} - \nabla f_i(\xx^\star))||^2\;.
    \end{align}
    where we use the fact that
    $
    \Avg [\xx_{i.r+1} - \frac{1}{\lambda} 
    (\hh_{i,r} - \nabla f_i(\xx^\star))] = \xx^{r+1}
    $.
    
    Note that 
    $
    \nabla F_{i,r} (\xx_{i,r+1}) 
    = 
    \mathbf{0}
    =
    \nabla f_i(\xx_{i,r+1}) - \hh_{i,r} + \lambda (\xx_{i,r+1} - \xx^r)$,
    it follows that:
    \begin{align}
        ||\xx^{r+1} - \xx^\star||^2 
        &+ \frac{1}{\lambda^2} 
        \Avg ||\hh_{i,r+1} - \nabla f_i (\xx^\star)||^2
        \\
        &= 
        \Avg \Bigl|\Bigl| 
        \xx^r - \xx^\star - \frac{1}{\lambda} 
        \bigl( \nabla f_i (\xx_{i,r+1}) - \nabla f_i(\xx^\star) \bigr)
        \Bigr|\Bigr|^2
        \label{eq:BasicRecurrenceOneControlVariate1}
    \end{align}

    We now upper bound the last display. Unrolling it gives:
    \begin{align}
        &\quad\; \Avg \Bigl|\Bigl| 
        \xx^r - \xx^\star - \frac{1}{\lambda} 
        \bigl( \nabla f_i (\xx_{i,r+1}) - \nabla f_i(\xx^\star) \bigr)
        \Bigr|\Bigr|^2 
        \\ 
        &= ||\xx^r - \xx^\star||^2 
        -\frac{2}{\lambda} \Avg 
        \lin{\xx^r - \xx^\star, \nabla f_i(\xx_{i,r+1}) - \nabla f_i(\xx^\star)}
        + \frac{1}{\lambda^2} 
        \Avg ||\nabla f_i (\xx_{i,r+1}) - \nabla f_i(\xx^\star)||^2 
        \label{eq:UnrollOneControlVariate1}
    \end{align}
    Note that the inner product can be lower bounded by:
    \allowdisplaybreaks{
    \begin{align}
        &\quad\; \lin{\xx^r - \xx^\star, \nabla f_i(\xx_{i,r+1}) - \nabla f_i(\xx^\star)}
        \\
        &= 
        \lin{\xx^r - \xx_{i,r+1}, \nabla f_i(\xx_{i,r+1}) - \nabla f_i(\xx^\star)}
        +
        \lin{\xx_{i,r+1} - \xx^\star, \nabla f_i(\xx_{i,r+1}) - \nabla f_i(\xx^\star)}
        \\
        &\stackrel{\eqref{eq:ConvexSmoothInnerProductLowerBound}}{\ge} 
        \frac{1}{\lambda} 
        \lin{ \nabla f_i(\xx_{i,r+1}) - \hh_{i,r},
        \nabla f_i(\xx_{i,r+1}) - \nabla f_i(\xx^\star) }
        +\frac{\mu L}{\mu + L} ||\xx_{i,r+1} - \xx^\star||^2 
        + \frac{1}{\mu + L} ||\nabla f_i (\xx_{i,r+1}) - \nabla f_i (\xx^\star)||^2
        \\
        &=
        \frac{1}{\lambda} ||\nabla f_i (\xx_{i,r+1}) - \nabla f_i(\xx^\star)||^2
        +
        \frac{1}{\lambda} 
        \lin{\nabla f_i(\xx^\star) - \hh_{i,r},
        \nabla f_i(\xx_{i,r+1}) - \nabla f_i(\xx^\star) }
        + \frac{\mu L}{\mu + L} ||\xx_{i,r+1} - \xx^\star||^2 
        \\ 
        &\quad\; 
        + \frac{1}{\mu + L} ||\nabla f_i (\xx_{i,r+1}) - \nabla f_i (\xx^\star)||^2
        \\
        &=
        \frac{1}{\lambda} ||\nabla f_i (\xx_{i,r+1}) - \nabla f_i(\xx^\star)||^2
        -\frac{1}{\lambda} ||\hh_{i,r} - \nabla f_i(\xx^\star)||^2
        +\frac{1}{\lambda} 
        \lin{\nabla f_i(\xx^\star) - \hh_{i,r}, \nabla f_i(\xx_{i,r+1}) - \hh_{i,r}}
        \\
        &\quad\; 
        + \frac{\mu L}{\mu + L} ||\xx_{i,r+1} - \xx^\star||^2 
        + \frac{1}{\mu + L} ||\nabla f_i (\xx_{i,r+1}) - \nabla f_i (\xx^\star)||^2
        \;.
        \label{eq:LowerBoundInnerProductControlVariate1}
    \end{align}}
    Plugging~\eqref{eq:LowerBoundInnerProductControlVariate1} 
    into~\eqref{eq:UnrollOneControlVariate1}, we obtain: 
    \begin{align}
        &\quad\; \Avg \Bigl|\Bigl| 
        \xx^r - \xx^\star - \frac{1}{\lambda} 
        \bigl( \nabla f_i (\xx_{i,r+1}) - \nabla f_i(\xx^\star) \bigr)
        \Bigr|\Bigr|^2 
        \\
        &\le 
        ||\xx^r - \xx^\star||^2
        -\frac{2 \mu L}{\lambda (\mu + L)} \Avg ||\xx_{i,r+1} - \xx^\star||^2
        -\biggl( \frac{1}{\lambda^2} + \frac{2}{\lambda (\mu + L)} \biggr)
        \Avg ||\nabla f_i (\xx_{i,r+1}) - \nabla f_i (\xx^\star)||^2
        \\
        & 
        + \frac{2}{\lambda^2} \Avg ||\hh_{i,r} - \nabla f_i(\xx^\star)||^2
        + \frac{2}{\lambda^2} \Avg 
        \lin{\hh_{i,r} - \nabla f_i(\xx^\star), \nabla f_i(\xx_{i,r+1}) - \hh_{i,r}} \;.
        \label{eq:UnrollTwoControlVariate1}
    \end{align}
    Further note that:
    \begin{align}
        &\quad\; - ||\nabla f_i (\xx_{i,r+1}) - \nabla f_i(\xx^\star)||^2
        + 2 
        \lin{\hh_{i,r} - \nabla f_i(\xx^\star), \nabla f_i(\xx_{i,r+1}) - \hh_{i,r}}
        \\ 
        &= -\Bigl( ||\hh_{i,r} - \nabla f_i(\xx^\star)||^2 
            + \lin{\nabla f_i(\xx_{i,r+1}) - \hh_{i,r}, 
            \nabla f_i(\xx_{i,r+1}) +\hh_{i,r} - 2\nabla f_i(\xx^\star)} \Bigr)
        \\
        &\quad\; 
        + 2 
        \lin{\hh_{i,r} - \nabla f_i(\xx^\star), \nabla f_i(\xx_{i,r+1}) - \hh_{i,r}}
        \\
        &=
        - ||\hh_{i,r} - \nabla f_i(\xx^\star)||^2 
        - ||\nabla f_i(\xx_{i,r+1}) - \hh_{i,r}||^2 \;.
    \end{align}
    Hence~\eqref{eq:UnrollTwoControlVariate1} can be simplified to:
    \begin{align}
        &\quad\; \Avg \Bigl|\Bigl| 
        \xx^r - \xx^\star - \frac{1}{\lambda} 
        \bigl( \nabla f_i (\xx_{i,r+1}) - \nabla f_i(\xx^\star) \bigr)
        \Bigr|\Bigr|^2 
        \\
        &\le 
        ||\xx^r - \xx^\star||^2
        -\frac{2 \mu L}{\lambda (\mu + L)} \Avg ||\xx_{i,r+1} - \xx^\star||^2
        + \frac{1}{\lambda^2} \Avg ||\hh_{i,r} - \nabla f_i(\xx^\star)||^2
        - \frac{1}{\lambda^2} \Avg ||\nabla f_i(\xx_{i,r+1}) - \hh_{i,r}||^2 
        \\
        &\quad\;
        - \frac{2}{\lambda (\mu + L)} 
        \Avg ||\nabla f_i (\xx_{i,r+1}) - \nabla f_i (\xx^\star)||^2 \;.
    \end{align}
    Plugging this upper bound into~\eqref{eq:BasicRecurrenceOneControlVariate1}
    and rearranging give the claim.
\end{proof}

\begin{corollary}
    Consider Algorithm~\ref{Alg:FrameworkDeterministic} with the standard averaging and control variate~\eqref{eq:ControlVariate1}
    with $m = \lambda = \sqrt{\mu L}$. 
    Let $f_i : \R^d \to \R$ be $L$-smooth and $\mu$-strongly convex with $\mu > 0$.
    Suppose that $n \ge 2$ and that each local solver provides the exact solution, then after $R$ communication rounds, we have:
    \begin{equation}
        \Phi^{R} 
        \le
        \biggl(
        1 - \frac{2\sqrt{\mu}}{\sqrt{L} + \frac{\mu}{\sqrt{L}} + 2\sqrt{\mu}} 
        \biggr)^R 
        \Phi^0 \;,
    \end{equation}
    where 
    $
    \Phi^r 
    := 
    || \xx^{r} - \xx^\star||^2 
       +\frac{(\mu + L)}{\mu L (\mu + L) + 2 \mu L \sqrt{\mu L}}
       \Avg ||\hh_{i,r} - \nabla f_i (\xx^\star) ||^2 
    $.
\end{corollary}

\begin{proof}
    According to Lemma~\ref{thm:BasicRecurrenceControlVariate1}, 
    for any $r \ge 0$, we have:
    \begin{equation}
    \begin{split}
        &||\xx^{r+1} - \xx^\star||^2 
        +
        \frac{2 \mu L}{\lambda (\mu + L)} \Avg ||\xx_{i,r+1} - \xx^\star||^2
        + \frac{1}{\lambda^2} 
        \Avg ||\hh_{i,r+1} - \nabla f_i (\xx^\star)||^2 
        \le 
        ||\xx^r - \xx^\star||^2
        \\
        &\quad+ \frac{1}{\lambda^2} \Avg ||\hh_{i,r} - \nabla f_i(\xx^\star)||^2
        - \frac{1}{\lambda^2} \Avg ||\nabla f_i(\xx_{i,r+1}) - \hh_{i,r}||^2 
        - \frac{2}{\lambda (\mu + L)} 
        \Avg ||\nabla f_i (\xx_{i,r+1}) - \nabla f_i (\xx^\star)||^2 \;.
    \end{split}
    \end{equation}
    It remains to upper bound the last two terms. Note that:
    \begin{align}
        &\quad\; ||\nabla f_i (\xx_{i,r+1}) - \nabla f_i (\xx^\star)||^2
        \\
        &=
        ||\hh_{i,r} - \nabla f_i(\xx^\star)||^2
        +
        2\lin{\hh_{i,r} - \nabla f_i(\xx^\star), \nabla f_i(\xx_{i,r+1}) - \hh_{i,r}}
        +
        ||\nabla f_i(\xx_{i,r+1}) - \hh_{i,r}||^2
        \\
        &\stackrel{\eqref{eq:BasicInequality1}}{\ge}
        (1 - 2 \frac{1}{2 \beta_r})||\hh_{i,r} - \nabla f_i(\xx^\star)||^2
        + (1 -2 \frac{\beta_r}{2}) ||\nabla f_i(\xx_{i,r+1}) - \hh_{i,r}||^2 \;.
    \end{align}
    It follows that:
    \begin{align}
        &\quad\;
        - \frac{1}{\lambda^2} \Avg ||\nabla f_i(\xx_{i,r+1}) - \hh_{i,r}||^2 
        - \frac{2}{\lambda (\mu + L)} 
        \Avg ||\nabla f_i (\xx_{i,r+1}) - \nabla f_i (\xx^\star)||^2
        \\
        &\le 
        - \frac{2}{\lambda (\mu + L)}(1 - \frac{1}{\beta_r})
        \Avg ||\hh_{i,r} - \nabla f_i(\xx^\star)||^2 
        + \biggl[
        \frac{2}{\lambda (\mu + L)}(\beta_r - 1) - \frac{1}{\lambda^2} 
        \biggr]
        \Avg ||\nabla f_i(\xx_{i,r+1}) - \hh_{i,r}||^2 \;.
    \end{align}
    Let the second coefficient be zero. We get the largest possible 
    $\beta_r = \frac{\mu + L}{2 \lambda} + 1$. 
    We obtain the main recurrence:
    \begin{align}
        ||\xx^{r+1} - \xx^\star||^2 
        &+\frac{(\mu + L)}{\lambda^2 (\mu + L) + 2 \mu L \lambda}
        \Avg ||\hh_{i,r+1} - \nabla f_i (\xx^\star)||^2 
        \le 
        \\
        &\max\Bigl\{
        \frac{\lambda (\mu + L)}{\lambda (\mu + L) + 2 \mu L},
        1 - \frac{2\lambda}{2\lambda + \mu + L}
        \Bigr\}
        \biggl(
        ||\xx^{r} - \xx^\star||^2 
        +\frac{(\mu + L)}{\lambda^2 (\mu + L) + 2 \mu L \lambda}
        \Avg ||\hh_{i,r} - \nabla f_i (\xx^\star)||^2 
        \biggr)\;.
    \end{align}
    The left component of the contraction factor is increasing in $\lambda$
    while the right component is decreasing in $\lambda$. Therefore, it is clear 
    that the best $\lambda$ is such that the left and the right components 
    are equal. This gives us exactly $\lambda^\star = \sqrt{\mu L}$ and
    $
    \frac{\lambda^\star (\mu + L)}{\lambda^\star (\mu + L) + 2 \mu L} 
    = 
    1 - \frac{2\sqrt{\mu}}{\sqrt{L} + \frac{\mu}{\sqrt{L}} + 2\sqrt{\mu}}.
    $
\end{proof}

 In the context of Algorithm~\ref{Alg:GDLocalSolver}, control Variate~\eqref{eq:ControlVariate1} becomes:
\begin{equation}
    \hh_{i, k+1}
    :=
    m_k (\xx_{i, k+1} - \xx_{i, k+1})
    +
    \hh_{i, k}
    \;.
    \label{eq:ControlVariateLocal1}
\end{equation}
with $m_k \ge 0$ and $\frac{1}{n}\sum_{i=1}^n \hh_{i,0} = 0$.

\subsection{Convex result and proof for Algorithm~\ref{Alg:FedRed} with 
            control variate~\eqref{eq:ControlVariateLocal2}}
The proofs in this section are similar to the ones for Algorithm~\ref{Alg:GDLocalSolver}.
\begin{theorem}
    Consider Algorithm~\ref{Alg:FedRed} with control 
    variate~\eqref{eq:ControlVariateLocal2} and the standard
    averaging.
    Let $f_i : \R^d \to \R$ be continuously differentiable, 
    $\mu$-convex with $\mu \ge 0$ for any $i \in [n]$. 
    Assume that $\{f_i\}$ have $\delta_A$-AHD.
    By choosing $p = \frac{\lambda + \mu/2}{\eta + \mu/2}$ and
                $\eta \ge \lambda \ge \delta_A$, 
    for any $K \ge 1$, it holds that:
    \begin{equation}
        \E[ f(\Bar{\Bar{\xx}}_{K}) - f^\star] 
        + 
        \frac{\mu}{4}
        \E\biggl[ 
        ||\Tilde{\xx}_{K} - \xx^\star||^2
        + 
        \frac{1}{n} \sum_{i=1}^n
        ||\xx_{i,K} - \xx^\star||^2
        \biggr]
        \le 
        \frac{\mu}{2} \frac{||\xx_0 - \xx^\star||^2 }{(1 + \frac{\mu}{2\eta})^K - 1}
        \le
        \frac{\eta ||\xx_0 - \xx^\star||^2}{K}
        \;.
    \end{equation}
    where  
    $
    \Bar{\Bar\xx}_{K} 
    :=
    \sum_{k=1}^K \frac{1}{q^k} \Bar{\xx}_k / \sum_{k=1}^K \frac{1}{q^k}$,
    $\Bar{\xx}_k := \Avg \xx_{i,k}$,
    and $q := 1 - \frac{\mu}{2\eta + \mu}$.
\end{theorem}

\begin{proof}
    Recall that the updates of Algorithm~\ref{Alg:GDLocalSolver}
    satisfies:
    \begin{equation}
        \xx_{i,k+1}
        =\argmin\{f_i(\xx)
        -
        \lin{ \xx, \hh_{i,k} }
        +
        \frac{\eta}{2}||\xx-\xx_{i,k}||^2
        + \frac{\lambda}{2}||\xx-\Tilde{\xx}_k||^2\} \;.
    \end{equation}
    Using strong convexity, we get:
    \begin{align}
        f_i(\xx^\star)
        -
        \lin{ \xx^\star, \hh_{i,k} }
        +
        \frac{\eta}{2}||\xx^\star-\xx_{i,k}||^2
        + \frac{\lambda}{2}||\xx^\star-\Tilde{\xx}_k||^2
        \stackrel{\eqref{df:stconvex}}{\ge}
        &f_i(\xx_{i,k+1})
        -
        \lin{ \xx_{i,k+1}, \hh_{i,k} }
        +
        \frac{\eta}{2}||\xx_{i,k+1}-\xx_{i,k}||^2
        \\
        &+ \frac{\lambda}{2}||\xx_{i,k+1}-\Tilde{\xx}_k||^2
        +\frac{\eta + \lambda + \mu}{2}
        ||\xx_{i,k+1} - \xx^\star||^2 \;.
    \end{align}
    By convexity of $f_i$, we further get:
    \begin{align}
        f_i(\xx^\star)
        -
        \lin{ \xx^\star, \hh_{i,k} }
        +
        \frac{\eta}{2}||\xx^\star-\xx_{i,k}||^2
        &+ \frac{\lambda}{2}||\xx^\star-\Tilde{\xx}_k||^2
        \stackrel{\eqref{df:stconvex}}{\ge}
        f_i(\Bar{\xx}_{k+1})
        +
        \lin{\nabla f_i(\Bar{\xx}_{k+1}), \xx_{i,k+1} - \Bar{\xx}_{k+1})}
        -
        \lin{ \xx_{i,k+1}, \hh_{i,k} }
        \\
        &+
        \frac{\eta}{2}||\xx_{i,k+1}-\xx_{i,k}||^2
        + \frac{\lambda}{2}||\xx_{i,k+1}-\Tilde{\xx}_k||^2
        +\frac{\eta + \lambda + \mu}{2}
        ||\xx_{i,k+1} - \xx^\star||^2 \;.
    \end{align} 
    Taking the average on both sides over $i=1$ to $n$, we get:
    \begin{equation}
    \begin{split}
        f(\xx^\star) + \frac{\eta}{2} \Avg ||\xx_{i,k} &- \xx^\star||^2
        + \frac{\lambda}{2} ||\Bar{\xx}_k - \xx^\star||^2
        \ge
        f(\Bar{\xx}_{k+1})
        +\Avg \biggl[
        \lin{\nabla f_i(\Bar{\xx}_{k+1}), \xx_{i,k+1} - \Bar{\xx}_{k+1}}
        -
        \lin{ \xx_{i,k+1}, \hh_{i,k} }
        \biggr]
        \\
        &+ \Avg \frac{\eta}{2}||\xx_{i,k+1}-\xx_{i,k}||^2
        + \Avg \frac{\lambda}{2}||\xx_{i,k+1}-\Tilde{\xx}_k||^2
        + \Avg \frac{\eta + \lambda + \mu}{2}
            ||\xx_{i,k+1} - \xx^\star||^2 \;.
    \end{split}
    \end{equation}
    Note that: 
    $\Avg
        \lin{\nabla f(\Bar{\xx}_{k+1}), \xx_{i,k+1} 
        - \Bar{\xx}_{k+1}}
        =
        \Avg \lin{ \Bar{\xx}_{k+1}, \hh_{i,k} } $.
    Let $h_i := f_i - f$.
    It follows that:
    \begin{align}
        &\quad \Avg \biggl[
        \lin{\nabla f_i(\Bar{\xx}_{k+1}), \xx_{i,k+1} - \Bar{\xx}_{k+1}}
        -
        \lin{ \xx_{i,k+1}, \hh_{i,k} }
        \biggr]
        +
        \Avg \frac{\lambda}{2}||\xx_{i,k+1}-\Tilde{\xx}_k||^2
        \\
        &\stackrel{\eqref{eq:AverageOfSquaredDifference}}{=}
        \Avg\biggl[
        \lin{\nabla h_i(\Bar{\xx}_{k+1}) - \nabla h_i(\Tilde{\xx}_{k}), 
        \xx_{i,k+1} - \Bar{\xx}_{k+1}}
        +
        \frac{\lambda}{2}||\xx_{i,k+1}-\Bar{\xx}_{k+1}||^2
        +
        \frac{\lambda}{2}||\Bar{\xx}_{k+1}-\Tilde{\xx}_k||^2
        \biggr]
        \\
        &\stackrel{\eqref{eq:QuadraticLowerBound}}{\ge}
        \Avg \biggl[
        -\frac{1}{2\lambda} ||\nabla h_i(\Bar{\xx}_{k+1}) 
        - \nabla h_i(\Tilde{\xx}_{k})||^2
        + \frac{\lambda}{2}||\Bar{\xx}_{k+1}-\Tilde{\xx}_k||^2
        \biggr]
        \\
        &\stackrel{\eqref{eq:HessianSimilarity}}{\ge}
        \Avg \biggl[
        -\frac{\delta_A^2}{2\lambda} ||\Bar{\xx}_{k+1}-\Tilde{\xx}_k||^2
        + \frac{\lambda}{2}||\Bar{\xx}_{k+1}-\Tilde{\xx}_k||^2
        \biggr] \ge 0\;.
    \end{align}
    where in the last inequality, we use the fact that $\lambda \ge \delta_A$.

    The main recurrence is then simplified as 
    (after dropping the non-negative $\frac{\eta}{2} \Avg ||\xx_{i,k+1} - \xx_{i,k}||^2$):
    \begin{align}
        f(\xx^\star) + \frac{\eta}{2} \Avg ||\xx_{i,k} - \xx^\star||^2
        + \frac{\lambda}{2} ||\Bar{\xx}_k - \xx^\star||^2
        &\ge
        f(\Bar{\xx}_{k+1})
        + \Avg \frac{\eta}{2}||\xx_{i,k+1}-\xx_{i,k}||^2
        + \Avg \frac{\eta + \lambda + \mu}{2}
            ||\xx_{i,k+1} - \xx^\star||^2 
        \nonumber
        \\
        &\stackrel{\eqref{eq:AverageOfSquaredDifference}}{\ge} 
        f(\Bar{\xx}_{k+1})
        + \frac{\eta + \mu / 2}{2} \Avg 
            ||\xx_{i,k+1} - \xx^\star||^2 
        + \frac{\lambda + \mu / 2}{2}  
            ||\Bar{\xx}_{k+1} - \xx^\star||^2 
            \;.
    \end{align}
    Taking expectation w.r.t $\theta_k$, we have:
    \begin{equation}
        \frac{\lambda + \mu/2}{2p} 
        \E_{\theta_k}[ ||\Tilde{\xx}_{k+1} - \xx^\star||^2 ]
        =
        \frac{\lambda + \mu/2}{2} ||\Bar{\xx}_{k+1} - \xx^\star||^2
        +
        \frac{\lambda + \mu/2}{2}(\frac{1}{p} - 1) 
        ||\Tilde{\xx}_k - \xx^\star||^2 \;,
    \end{equation} 
    Adding both sides by $\frac{\lambda + \mu/2}{2}(\frac{1}{p} - 1) 
        ||\Tilde{\xx}_k - \xx^\star||^2$ and taking the expectation
        w.r.t $\theta_k$, we get:
        \begin{equation}
        \begin{split}
            f(\xx^\star) 
            + \frac{\eta}{2} \Avg ||\xx_{i,k} - \xx^\star||^2
            + \Bigl( \frac{\lambda + \mu/2}{2p} - \frac{\mu}{4} \Bigr)
            ||\Bar{\xx}_k - \xx^\star||^2
            \ge
            \E_{\theta_k}\Bigl[ 
            f(\Bar{\xx}_{k+1}) 
            &+ \frac{\eta + \mu/2}{2}  \Avg ||\xx_{i,k+1} - \xx^\star||^2 
            \\
            &+ \frac{\lambda + \mu/2}{2p} 
                ||\Tilde{\xx}_{k+1} - \xx^\star||^2 
            \Bigr] \;.
        \end{split}
        \end{equation}
        By our choice of $p = \frac{\lambda + \mu/2}{\eta + \mu/2}$ with 
        $\eta \ge \lambda$,  we get:
        \begin{equation}
            \frac{\eta + \mu/2}{2} = \frac{\lambda + \mu/2}{2p}, \quad
            \frac{\lambda + \mu/2}{2p} - \frac{\mu}{4} = \frac{\eta}{2} \;.
        \end{equation}
        Taking the full expectation of the main recurrence, 
        dividing both sides by $\frac{\eta + \mu/2}{2}$,
        applying Lemma~\ref{thm:StrongConvexityRecurrence},
        using the convexity of $f$,
        and multiplying both sides by $\frac{\eta + \mu/2}{2}$,
        we get:
        \begin{equation}
        \E[ f(\Bar{\Bar{\xx}}_{K}) - f^\star] 
        + 
        \frac{\mu}{4}
        \E\biggl[ 
        ||\Tilde{\xx}_{K} - \xx^\star||^2
        + 
        \frac{1}{n} \sum_{i=1}^n
        ||\xx_{i,K} - \xx^\star||^2
        \biggr]
        \le 
        \frac{\mu}{2} \frac{||\xx_0 - \xx^\star||^2 }{(1 + \frac{\mu}{2\eta})^K - 1}
        \;.
        \qedhere
    \end{equation}
\end{proof}

\subsection{Non-convex result and proof for Algorithm~\ref{Alg:FedRed} with 
            control variate~\eqref{eq:ControlVariateLocal2}}

\begin{theorem}
    Consider Algorithm~\ref{Alg:FedRed} with control 
    variate~\eqref{eq:ControlVariateLocal2} and randomized averaging
    with random index set $\{i_k\}_{k=0}^{+\infty}$.
    Let $f_i : \R^d \to \R$ be continuously differentiable for any $i \in [n]$. 
    Assume that $\{f_i\}$ have $\delta_B$-BHD.
    Let $\lambda = \delta_B$, $p = \frac{\lambda}{\eta}$ and 
    $\eta \ge 4 \delta_B$.
    For any $K \ge 1$, it holds that:
    \begin{equation}
        \E\Bigl[ ||\nabla f(\Bar{\xx}_{K})||^2 \Bigr]
        \le 
        \frac{150\eta (f(\xx^0) - f^\star)}{K} 
        \;,
    \end{equation}
    where $\Bar{\xx}_K$ is uniformly sampled from 
    $(\xx_{i_k,k})_{k=0}^{K-1}$.
\end{theorem}

\begin{proof}
    Let $h_i := f - f_i$. 
    Recall that for any $i \in [n]$, 
    the update of Algorithm~\ref{Alg:FedRed} satisfies: 
    $F_{i,k}(\xx_{i,k+1}) \le F_{i,k}(\xx_{i,k})$
    where
    $F_{i,k}(\xx) 
    := 
    f_i(\xx) 
    +
    \lin{ \nabla h_i(\Tilde{\xx}_k), \xx} 
    + \frac{\eta}{2} ||\xx - \xx_{i,k}||^2
    + \frac{\lambda}{2} ||\xx - \Tilde{\xx}_k||^2$.
    This gives:
    \begin{equation}
    \begin{split}
        f_i(\xx_{i,k}) 
        +
        \lin{\nabla h_i(\Tilde{\xx}_k), \xx_{i,k}}
        +
        \frac{\lambda}{2} ||\xx_{i,k} - \Tilde{\xx}_k||^2
        \ge 
        f_{i} (\xx_{i,k+1})
        +   
        &\lin{\nabla h_i(\Tilde{\xx}_k), \xx_{i,k+1}}
        +
        \frac{\lambda}{2} ||\xx_{i,k+1} - \Tilde{\xx}_k||^2
        +
        \frac{\eta}{2} ||\xx_{i,k+1} - \xx_{i,k}||^2 \;.
    \end{split}
    \end{equation}
    According to Lemma~\ref{thm:TwoPointRelationNonConvexFramework}
    with $\xx^r = \xx_{i,k}$ and $\yy = \xx_{i,k+1}$, we get:
    \begin{equation}
        f_i(\xx_{i,k}) + \lin{\nabla h_i (\xx_{i,k}), \xx_{i,k} - \xx_{i,k+1}}
        - f_i(\xx_{i,k+1}) \le f(\xx_{i,k}) - f(\xx_{i,k+1})
        + \frac{\delta_B}{2} ||\xx_{i,k+1} - \xx_{i,k}||^2 \;.
    \end{equation}
    Substituting this inequality into the previous display, we get:
    \begin{equation}
    \begin{split}
        f(\xx_{i,k})
        +
        \frac{\lambda}{2} ||\xx_{i,k} - \Tilde{\xx}_k||^2
        \ge 
        f(\xx_{i,k+1})
        -
        &\lin{ \nabla h_i(\xx_{i,k}) - \nabla h_i(\Tilde{\xx}_k), \xx_{i,k+1} - \xx_{i,k}}
        \\
        &+
        \frac{\lambda}{2} ||\xx_{i,k+1} - \Tilde{\xx}_k||^2
        +
        \frac{\eta - \delta_B}{2} ||\xx_{i,k+1} - \xx_{i,k}||^2 \;.
    \end{split}
    \end{equation}
    For any $\alpha > 0$, we have that:
    \begin{align}
        -
        \lin{\nabla h_i(\xx_{i,k}) - \nabla h_i(\Tilde{\xx}_k),
        \xx_{i,k+1} - \xx_{i,k}}
        &=
        \lin{\nabla h_i(\Tilde{\xx}_k) - \nabla h_i(\xx_{i,k}),
        \xx_{i,k+1} - \xx_{i,k}}
        \\
        &\stackrel{\eqref{eq:BasicInequality1}}{\ge} 
        -\frac{||\nabla h_i(\Tilde{\xx}_k) - \nabla h_i(\xx_{i,k})||^2}{2 \alpha}
        -\frac{\alpha ||\xx_{i,k+1} - \xx_{i,k}||^2}{2}
        \\
        &\ge
        -\frac{\delta_B^2||\Tilde{\xx}_k - \xx_{i,k}||^2}{2 \alpha}
        -\frac{\alpha ||\xx_{i,k+1} - \xx_{i,k}||^2}{2} \;.
    \end{align}
    It follows that:
    \begin{equation}
        f(\xx_{i,k})
        +
        \Bigl( \frac{\lambda}{2} + \frac{\delta_B^2}{2\alpha}\Bigr)
        ||\xx_{i,k} - \Tilde{\xx}_k||^2
        \ge 
        f(\xx_{i,k+1})
        +
        \frac{\lambda}{2} ||\xx_{i,k+1} - \Tilde{\xx}_k||^2
        +
        \frac{\eta - \delta_B - \alpha}{2} ||\xx_{i,k+1} - \xx_{i,k}||^2 \;.
        \label{eq:FedRedNonConvex1}
    \end{equation}
    Suppose that $\eta - \delta_B - \alpha > 0$.
    We lower bound the last term. Recall that $\xx_{i,k+1}$ is a stationary 
    point of $F_{i,k}$. We have:
    \begin{equation}
        \nabla f_i(\xx_{i,k+1}) + \nabla h_i (\Tilde{\xx}_k) 
        + \eta (\xx_{i,k+1} - \xx_{i,k}) + \lambda (\xx_{i,k+1} - \Tilde{\xx}_k)
        = 0 \;,
    \end{equation}
    which implies:
    \begin{equation}
        \xx_{i,k+1} - \xx_{i,k}
        =
        -\frac{1}
            {\eta + \lambda} \bigl(
            \nabla f_i(\xx_{i,k+1}) 
            + \nabla h_i (\Tilde{\xx}_k) + \lambda (\xx_{i,k} - \Tilde{\xx}_k) \bigr) \;.
    \end{equation}
    It follows that:
    \begin{align}
        (\eta + \lambda) ||\xx_{i,k+1} - \xx_{i,k}||
        &=
        ||\nabla f_i(\xx_{i,k+1}) 
            + \nabla h_i (\Tilde{\xx}_k) 
            + \lambda (\xx_{i,k} - \Tilde{\xx}_k)||
        \\
        &\ge ||\nabla f(\xx_{i,k+1})|| 
                - ||\nabla h_i (\Tilde{\xx}_k) - \nabla h_i (\xx_{i,k+1})||
                - \lambda ||\xx_{i,k} - \Tilde{\xx}_k||
        \\
        &\stackrel{\eqref{eq:MaxHessianSimilarity}}{\ge} 
        ||\nabla f(\xx_{i,k+1})|| 
        - \delta_B ||\Tilde{\xx}_k - \xx_{i,k+1}||
        - \lambda ||\xx_{i,k} - \Tilde{\xx}_k|| \;.
    \end{align}
    This gives:
    \begin{equation}
        (\eta + \lambda)^2 ||\xx_{i,k+1} - \xx_{i,k}||^2
        \stackrel{\eqref{eq:BasicInequality1}}{\ge}
        \frac{||\nabla f(\xx_{i,k+1})||^2}{4}
        -
        \frac{\delta_B^2 ||\Tilde{\xx}_k - \xx_{i,k+1}||^2}{2}
        -
        \lambda^2 ||\xx_{i,k} - \Tilde{\xx}_k||^2 \;.
    \end{equation}
    Plugging this inequality into~\eqref{eq:FedRedNonConvex1}, 
    we get, for any $i \in [n]$:
    \begin{equation}
        f(\xx_{i,k}) + A ||\xx_{i,k} - \Tilde{\xx}_k||^2 
        \ge
        f(\xx_{i,k+1}) + B ||\xx_{i,k+1} - \Tilde{\xx}_k||^2
        + C ||\nabla f(\xx_{i,k+1})||^2 \;.
    \end{equation}
    where 
    $A:= \frac{\lambda}{2} 
        + \frac{\delta_B^2}{2\alpha} 
        + \frac{\lambda^2(\eta - \delta_B - \alpha)}{2(\lambda + \eta)^2}$, 
    $B:= \frac{\lambda}{2} 
            - \frac{\delta_B^2 (\eta - \delta_B - \alpha)}
                    {4(\eta + \lambda)^2}$ 
    and 
    $C:= \frac{\eta - \delta_B - \alpha}{8 (\eta + \lambda)^2}$.
    
    Since the previous display holds for any $i \in [n]$, 
    we can take the expectation w.r.t $\theta_k$ and $i_k$ to get:
    \begin{equation}
        \E_{i_k,\theta_k}[
        f(\xx_{i_k,k}) + A ||\xx_{i_k,k} - \Tilde{\xx}_k||^2 ]
        \ge
        \E_{i_k,\theta_k}[
        f(\xx_{i_k,k+1}) + B ||\xx_{i_k,k+1} - \Tilde{\xx}_k||^2
        + C ||\nabla f(\xx_{i_k,k+1})||^2 
        ] \;.
    \end{equation}
    Recall that with probability $p$, $\Tilde{\xx}_{k+1} = \xx_{i_k,k+1}$.
    It follows that:
    \begin{equation}
        \E_{i_k,\theta_k}[|| \xx_{i_k,k+1} - \Tilde{\xx}_{k+1} ||^2]
        =
        (1 - p)
        \E_{i_k,\theta_k}[|| \xx_{i_k,k+1} - \Tilde{\xx}_k ||^2] \;.
    \end{equation}
    Substituting this identity into the previous display, we get:
    \begin{equation}
        \E_{i_k,\theta_k}[
        f(\xx_{i_k,k}) + A ||\xx_{i_k,k} - \Tilde{\xx}_k||^2 ]
        \ge
        \E_{i_k,\theta_k}\bigg[
        f(\xx_{i_k,k+1}) + \frac{B}{1-p} ||\xx_{i_k,k+1} - \Tilde{\xx}_{k+1}||^2
        + C ||\nabla f(\xx_{i_k,k+1})||^2 
        \biggr] \;.
    \end{equation}
    Note that $i_{k+1}$ and $i_k$ follow the same distribution and are 
    independent of $(\theta_k)_{k=0}^{+\infty}$. It follows that:
    $
        \E_{i_k, \theta_k}
        [f(\xx_{i_k,k+1})] = \E_{i_{k+1}, i_k, \theta_k}[f(\xx_{i_{k+1},k+1})]$
    ,
    $
    \E_{i_k, \theta_k}
    [||\xx_{i_k,k+1} - \Tilde{\xx}_{k+1}||^2]
    =
    \E_{i_{k+1}, i_k, \theta_k}[||\xx_{i_{k+1},k+1} - \Tilde{\xx}_{k+1}||^2] 
    $,
    and
    $
    \E_{i_k, \theta_k}
    [||\nabla f(\xx_{i_k,k+1})||^2]
    =
    \E_{i_{k+1}, i_k, \theta_k}[||\nabla f(\xx_{i_{k+1},k+1})||^2] 
    $,

    Taking expectation w.r.t $i_{k+1}$ on both sides of the previous display, substituting these three identities, and then taking the full expectation,
    we obtain our main recurrence:
    \begin{equation}
        \E[
        f(\xx_{i_k,k}) + A ||\xx_{i_k,k} - \Tilde{\xx}_k||^2 ]
        \ge
        \E\bigg[
        f(\xx_{i_{k+1},k+1}) + \frac{B}{1-p} ||\xx_{i_{k+1},k+1} - \Tilde{\xx}_{k+1}||^2
        + C ||\nabla f(\xx_{i_{k+1},k+1})||^2 
        \biggr] \;.
    \end{equation}
    Suppose that $A \le \frac{B}{1 - p}$ and that $\eta - \delta_B - \alpha > 0$. 
    Summing up from $k=0$ to $K-1$ and 
    dividing both sides by $K$, we get:
    \begin{equation}
        \frac{1}{K}\sum_{k=1}^{K}
        \E\Bigl[ ||\nabla f(\xx_{i_k,k})||^2 \Bigr]
        \le 
        \frac{f(\xx^0) - f^\star}{C K} 
        \;,
    \end{equation}
    where we use the fact that $\E[||\xx_{i_0,0} - \Tilde{\xx}_0||^2] = 0$.

    We next choose the parameters to satisfy $A \le \frac{B}{1 - p}$ and 
    $\eta - \delta_B - \alpha > 0$.
    Plugging in the definitions of $A$ and $B$, we get:
    \begin{equation}
        \frac{\lambda}{2} 
        + \frac{\delta_B^2}{2\alpha} 
        + \frac{\lambda^2(\eta - \delta_B - \alpha)}{2(\lambda + \eta)^2}
        + \frac{\delta_B^2 (\eta - \delta_B - \alpha)}
                    {4(\eta + \lambda)^2 (1-p)}
        \le
        \frac{\lambda}{2 (1-p)} \;.
    \end{equation}
    This is equivalent to:
    \begin{equation}
        (1-p)\frac{\delta_B^2}{2\alpha} 
        + \frac{\bigl(
        \delta_B^2 + 2(1-p)\lambda^2 \bigr) (\eta - \delta_B - \alpha)}
            {4(\eta + \lambda)^2} \le \lambda p \;.
    \end{equation}
    Note that:
    \begin{align}
        (1-p)\frac{\delta_B^2}{2\alpha} 
        + \frac{\bigl(
        \delta_B^2 + 2(1-p)\lambda^2 \bigr) (\eta - \delta_B - \alpha)}
            {4(\eta + \lambda)^2}
        &\le
        \frac{\delta_B^2}{2\alpha} 
        + \frac{\bigl(
        \delta_B^2 + 2\lambda^2 \bigr) (\eta - \delta_B - \alpha)}
            {4(\eta + \lambda)^2}
        \\
        &= 
        \frac{\delta_B^2}{2\alpha} 
        -\frac{\delta_B^2 + 2\lambda^2}{4(\eta + \lambda)^2} \alpha
        + 
        \frac{\bigl(
        \delta_B^2 + 2\lambda^2 \bigr) (\eta - \delta_B)}
            {4(\eta + \lambda)^2} \;.
    \end{align}
    Let $\lambda = \delta_B$, $p = \frac{\lambda}{\eta}$, $\eta \ge \lambda$
    and $\alpha = \frac{2}{3}\eta$. We have:
    \begin{equation}
        \frac{\delta_B^2}{2\alpha} 
        -\frac{\delta_B^2 + 2\lambda^2}{4(\eta + \lambda)^2} \alpha
        + 
        \frac{\bigl(
        \delta_B^2 + 2\lambda^2 \bigr) (\eta - \delta_B)}
            {4(\eta + \lambda)^2}
        \le 
        \lambda^2\biggl[  
        \frac{1}{2\alpha} + \frac{3(\eta - \alpha)}{4 (\eta + \lambda)^2}
        \biggr]
        \le 
        \lambda^2 \Bigl[\frac{3}{4\eta} + \frac{1}{4\eta} \Bigr] 
        = \lambda p \;.
    \end{equation}
    At the same time, $\eta$ should satisfy
    $\eta > 3 \delta_B$ such that
    $\eta - \delta_B - \alpha >0$. 
    Let $\eta \ge 4\delta_B$, we have:
    \begin{equation}
        C:= \frac{\eta - \delta_B - \alpha}{8 (\eta + \lambda)^2}
          = \frac{\frac{\eta}{3} - \delta_B}{8 (\eta + \delta_B)^2} 
          \ge 
          \frac{\frac{\eta}{3} - \frac{\eta}{4}}
          {8 (\eta + \frac{1}{4}\eta)^2}
          =
          \frac{\eta}{150} \;.
    \end{equation}

\end{proof}

\subsection{Convex result and proof for Algorithm~\ref{Alg:GDLocalSolver} with 
            control variate~\eqref{eq:ControlVariateLocal2}}

\begin{lemma}
    \label{thm:MainRecurrenceConvexFedRedGD}
    Consider Algorithm~\ref{Alg:GDLocalSolver} with control 
    variate~\eqref{eq:ControlVariateLocal2} and the standard
    averaging.
    Let $f_i : \R^d \to \R$ be continuously differentiable, 
    $\mu$-convex with $\mu \ge 0$ and $L$-smooth for any $i \in [n]$. 
    Assume that $\{f_i\}$ have $\delta_A$-AHD with $\delta_A \le L$.
    Assume that $\E[g_i(\xx)] = \nabla f_i(\xx)$ for any $\xx \in \R^d$ 
    and that $\E[||g_i(\xx) - \nabla f_i(\xx)||^2] \le \sigma^2$.
    Let $\lambda \ge \delta_A$ and $\eta \ge L$.
    Then for any $k \ge 0$, it holds that:
    \begin{equation}
    \begin{split}
        \E\Biggl[
        f(\xx^\star) 
        + \biggl[ \frac{\lambda + \mu/2}{2p} 
        - \frac{\mu}{4} \biggr] ||\Tilde{\xx}_k - \xx^\star||^2
        + \frac{\eta - \mu}{2} 
        \Avg &||\xx_{i,k} - \xx^\star||^2 \Biggr]
        \ge 
        \E\Biggl[
        f(\Bar{\xx}_{k+1}) 
        + \frac{\lambda + \mu/2}{2p} ||\Tilde{\xx}_{k+1} - \xx^\star||^2 
        \\
        &+ \frac{\eta - \mu/2}{2} 
        \Avg ||\xx_{i,k+1} - \xx^\star||^2
        \Biggr] 
        - \frac{\sigma^2}{2(\eta - L)} \;.
    \end{split}
    \end{equation}
\end{lemma}

\begin{proof}
    Let 
    $
    G_{i,k}(\xx) 
    := 
    f_i(\xx_{i,k})
    +
    \lin{g_i(\xx_{i,k}), \xx - \xx_{i,k}}
    +
    \frac{\eta}{2} ||\xx - \xx_{i,k}||^2
    $.
    Recall that the update of Algorithm~\ref{Alg:GDLocalSolver} satisfies:
    \begin{equation}
    \xx_{i,k+1} 
    = 
    \argmin_{\xx \in \R^d} 
    \{ G_{i,k} (\xx) 
    -
    \lin{ \hh_{i,k}, \xx} 
    + \frac{\lambda}{2} ||\xx - \Tilde{\xx}_k||^2 \} \;.
    \end{equation}
    Using strong convexity, we get:
    \begin{equation}
    \begin{split}
        G_{i,k} (\xx^\star)
        -
        \lin{\hh_{i,k}, \xx^\star} 
        + 
        \frac{\lambda}{2} ||\xx^\star - \Tilde{\xx}_k||^2 
        \stackrel{\eqref{df:stconvex}}{\ge} 
        G_{i,k} (\xx_{i,k+1})
        &-
        \lin{\hh_{i,k}, \xx_{i,k+1}} 
        \\
        &+ 
        \frac{\lambda}{2} ||\xx_{i,k+1} - \Tilde{\xx}_k||^2 
        +
        \frac{\lambda + \eta}{2} ||\xx_{i,k+1} - \xx^\star||^2 \;.
        \label{eq:GDControlVariate2ConvexFirstEquation}
    \end{split}
    \end{equation}

    Denote the randomness coming from $g_i(\xx_{i,k})$ by $\xi_i$ and
    denote all the randomness $\{\xi_i\}_{i=1}^n$ by $\xi$. 

    Recall that $f_i$ is $\mu$-strongly convex and $L$-smooth.
    Suppose that $\eta \ge L$, we get:
    \begin{equation}
        \E_\xi [ G_{i,k} (\xx^\star) ]
        \stackrel{\eqref{df:stconvex}}{\le}
        f_i(\xx^\star)
        +
        \frac{\eta - \mu}{2}
        ||\xx_{i,k} - \xx^\star||^2 \;,
    \end{equation}
    and
    \begin{align}
        \E_{\xi} [ G_{i,k} (\xx_{i,k+1}) ]
        &\stackrel{\eqref{eq:SmoothUpperBound}}{\ge}
        \E_\xi 
        \Bigl[
        f_i(\xx_{i,k+1}) + 
        \lin{g_i(\xx_{i,k}) - \nabla f_i(\xx_{i,k}), 
        \xx_{i.k+1} - \xx_{i,k} } 
        + \frac{\eta - L}{2} ||\xx_{i,k+1} - \xx_{i,k}||^2
        \Bigr]
        \\
        &\stackrel{\eqref{eq:QuadraticLowerBound}}{\ge}
        \E_\xi 
        \Bigl[
        f_i(\xx_{i,k+1}) 
        -
        \frac{||g_i(\xx_{i,k}) - \nabla f_i(\xx_{i,k})||^2}{2(\eta - L)}
        \Bigr]
        \\
        &\ge 
        \E_\xi[ f_i(\xx_{i,k+1})] - \frac{\sigma^2}{2 (\eta - L)} \;.
    \end{align}
    where in the last inequality, we use the assumption that 
    $\E_{\xi_i} [||g_i(\xx) - \nabla f_i(\xx)||^2] \le \sigma^2$ 
    for any $\xx \in \R^d$.

    Taking the expectation w.r.t $\xi$ on both sides
    of~\eqref{eq:GDControlVariate2ConvexFirstEquation}, plugging in these bounds 
    and taking the average on both sides over $i=1$ to $n$, we get:
    \allowdisplaybreaks{
    \begin{align}
        &\quad\; f(\xx^\star) + \frac{\lambda}{2} ||\Tilde{\xx}_k - \xx^\star||^2
        + \frac{\eta - \mu}{2} \Avg ||\xx_{i,k} - \xx^\star||^2
        \\
        &\ge 
        \E_{\xi}\biggl[ \Avg f_i(\xx_{i,k+1}) 
        + \frac{\lambda}{2} \Avg ||\xx_{i,k+1} - \Tilde{\xx}_k||^2
        + \frac{\lambda + \eta}{2} 
        \Avg ||\xx_{i,k+1} - \xx^\star||^2 
        -
        \Avg \lin{\hh_{i,k}, \xx_{i, k+1}} \biggr]
        - \frac{\sigma^2}{2(\eta - L)} 
        \\
        &\stackrel{\eqref{df:stconvex}}{\ge} 
        \E_{\xi} \biggl[ f(\Bar{\xx}_{k+1})
        + \frac{\lambda}{2} \Avg ||\xx_{i,k+1} - \Tilde{\xx}_k||^2
        + \frac{\lambda + \eta}{2} 
        \Avg ||\xx_{i,k+1} - \xx^\star||^2
        + \frac{\mu}{2} \Avg ||\xx_{i,k+1} - \Bar{\xx}_{k+1}||^2 
        \\
        &\quad 
        + \Avg \lin{\nabla f_i(\Bar{\xx}_{k+1}) - \nabla f(\Bar{\xx}_{k+1}) 
        - \nabla f_i (\Tilde{\xx}_k) + \nabla f(\Tilde{\xx}_k), 
        \xx_{i,k+1} - \Bar{\xx}_{k+1}} \biggr]
        - \frac{\sigma^2}{2(\eta - L)}  \;,
    \end{align}}
    where in the last inequality, we use the strong convexity of $f_i$ and 
    the fact that:
    \begin{equation}
    \Avg \lin{\hh_{i,k}, \Bar{\xx}_{k+1}}
    =
    \Avg \lin{\nabla f(\Bar{\xx}_{k+1}), \xx_{i,k+1} - \Bar{\xx}_{k+1}}
    =
    0 \;. 
    \end{equation}
    Let $h_i := f_i - f$.
    By the assumption that $\{f_i\}$ have $\delta_A$-AHD, it follows that:
    \allowdisplaybreaks{
    \begin{align}
        &\quad \Avg \lin{\nabla h_i (\Bar{\xx}_{k+1}) 
        - \nabla h_i (\Tilde{\xx}_k), 
        \xx_{i,k+1} - \Bar{\xx}_{k+1}}
        +
        \frac{\lambda}{2} \Avg ||\xx_{i,k+1} - \Tilde{\xx}_k||^2
        \\
        &\stackrel{\eqref{eq:AverageOfSquaredDifference}}=
        \Avg \lin{\nabla h_i (\Bar{\xx}_{k+1}) 
        - \nabla h_i (\Tilde{\xx}_k), 
        \xx_{i,k+1} - \Bar{\xx}_{k+1}}
        +
        \frac{\lambda}{2} \Avg ||\xx_{i,k+1} - \Bar{\xx}_{k+1}||^2
         +
        \frac{\lambda}{2} \Avg ||\Bar{\xx}_{k+1} - \Tilde{\xx}_k||^2
        \\
        &\stackrel{\eqref{eq:QuadraticLowerBound}}\ge 
        -\frac{1}{2\lambda} \Avg
        ||\nabla h_i (\Bar{\xx}_{k+1}) 
        - \nabla h_i (\Tilde{\xx}_k)||^2
        +
        \frac{\lambda}{2} \Avg ||\Bar{\xx}_{k+1} - \Tilde{\xx}_k||^2
        \\
        &\stackrel{\eqref{eq:HessianSimilarity}} \ge
        -\frac{\delta_A^2}{2\lambda} ||\Bar{\xx}_{k+1} - \Tilde{\xx}_k||^2 +
        \frac{\lambda}{2} \Avg ||\Bar{\xx}_{k+1} - \Tilde{\xx}_k||^2 \ge 0 \;.
    \end{align}}
    where we use the assumption that $\lambda \ge \delta_A$ in the last inequality.
    We can then simplify the main recurrence as:
    \begin{equation}
    \begin{split}
        f(\xx^\star) 
        + \frac{\lambda}{2} ||\Tilde{\xx}_k - \xx^\star||^2
        + \frac{\eta - \mu}{2} \Avg ||\xx_{i,k} - \xx^\star||^2
        &\ge 
        \E_{\xi} \biggl[ 
        f(\Bar{\xx}_{k+1}) 
        + \frac{\lambda + \eta}{2} 
        \Avg ||\xx_{i,k+1} - \xx^\star||^2\biggr] 
        - \frac{\sigma^2}{2(\eta - L)} 
        \\ 
        &\stackrel{\eqref{eq:AverageOfSquaredDifference}}\ge 
        \E_{\xi} \biggl[ f(\Bar{\xx}_{k+1}) 
        + \frac{\eta - \mu/2}{2} 
        \Avg ||\xx_{i,k+1} - \xx^\star||^2
        \\
        &\qquad \qquad \qquad+
        \frac{\lambda + \mu/2}{2}
        ||\Bar{\xx}_{k+1} - \xx^\star||^2
        \biggr] 
        - \frac{\sigma^2}{2(\eta - L)} 
        \;.
    \end{split}
    \end{equation}
    where we drop the non-negative $\frac{\mu}{2} \Avg ||\xx_{i,k+1} - \Bar{\xx}_{k+1}||^2 $.
    
    Taking expectation w.r.t $\theta_k$, we have:
    \begin{equation}
        \frac{\lambda + \mu/2}{2p} 
        \E_{\theta_k,\xi}[ ||\Tilde{\xx}_{k+1} - \xx^\star||^2 ]
        =
        \frac{\lambda + \mu/2}{2} \E_{\xi}[||\Bar{\xx}_{k+1} - \xx^\star||^2]
        +
        \frac{\lambda + \mu/2}{2}(\frac{1}{p} - 1) 
        ||\Tilde{\xx}_k - \xx^\star||^2 \;,
    \end{equation} 
    Adding both sides by 
    $\frac{\lambda + \mu/2}{2}(\frac{1}{p} - 1) 
        ||\Tilde{\xx}_k - \xx^\star||^2$ and
    taking the expectation over $\theta_k$ on both sides of the main recurrence,
    we get:
    \begin{equation}
    \begin{split}
        f(\xx^\star) 
        + \biggl[ \frac{\lambda + \mu/2}{2p} 
        - \frac{\mu}{4} \biggr] ||\Tilde{\xx}_k - \xx^\star||^2
        + \frac{\eta - \mu}{2} 
        \Avg ||\xx_{i,k} - \xx^\star||^2 
        &\ge 
        \E_{\theta_k,\xi}\biggl[
        f(\Bar{\xx}_{k+1}) 
        + \frac{\lambda + \mu/2}{2p} ||\Tilde{\xx}_{k+1} - \xx^\star||^2 
        \\
        &+ \frac{\eta - \mu/2}{2} 
        \Avg ||\xx_{i,k+1} - \xx^\star||^2
        \biggr] 
        - \frac{\sigma^2}{2(\eta - L)} \;.
        \label{eq:GDControlVariate2ConvexMainRecurrence}
    \end{split}
    \end{equation}
    Taking the full expectation on both sides, we get the claim.
\end{proof}

\begin{theorem}
    Consider Algorithm~\ref{Alg:GDLocalSolver} with control 
    variate~\eqref{eq:ControlVariateLocal2} and the standard
    averaging.
    Let $f_i : \R^d \to \R$ be continuously differentiable, 
    $\mu$-convex with $\mu \ge 0$ and $L$-smooth for any $i \in [n]$. 
    Assume that $\{f_i\}$ have $\delta_A$-AHD with $\delta_A + \mu \le L$.
    Assume that $\E[g_i(\xx)] = \nabla f_i(\xx)$ for any $\xx \in \R^d$ 
    and that $\E[||g_i(\xx) - \nabla f_i(\xx)||^2] \le \sigma^2$.
    By choosing 
    $\lambda = \delta_A$, $p = \frac{\lambda + \mu / 2}{\eta - \mu/2}$, 
    and $\eta > L$,
    for any $K \ge 1$,
    it holds that:
    \begin{equation}
    \begin{split}
        \E[ f(\Bar{\Bar{\xx}}_{K}) - f^\star] 
        + 
        \frac{\mu}{4}
        \E\biggl[  
        ||\Tilde{\xx}_{K} - \xx^\star||^2
        +
        \frac{1}{n} \sum_{i=1}^n
        ||\xx_{i,K} - \xx^\star||^2
        \biggr]
        &\le 
        \frac{\mu}{2} \frac{1}{(1 + \frac{\mu}{2\eta - 2\mu})^K - 1}
        ||\xx_0 - \xx^\star||^2
        + \frac{\sigma^2}{2(\eta - L)} 
        \\
        &\le
        \frac{\eta}{K} ||\xx^0 - \xx^\star||^2
        + \frac{\sigma^2}{2(\eta - L)}  \;.
    \end{split}
    \end{equation}
where  
$
\Bar{\Bar\xx}_{K} 
:=
\sum_{k=1}^K \frac{1}{q^k} \Bar{\xx}_k / \sum_{k=1}^K \frac{1}{q^k}$,
$\Bar{\xx}_k := \Avg \xx_{i,k}$,
and $q := 1 - \frac{\mu}{2\eta - \mu}$.

To reach $\epsilon$-accuracy, i.e. 
$\E[ f(\Bar{\Bar{\xx}}_{K}) - f^\star] \le \epsilon$, 
by choosing $\eta = \frac{\sigma^2}{\epsilon} + L$, we get:
\begin{equation}
    K \le \Biggl\lceil \Bigl( \frac{2L}{\mu} + \frac{2\sigma^2}{\mu \epsilon}\Bigr)
    \ln\Bigl(1 + \frac{\mu||\xx^0 - \xx^\star||^2}{\epsilon} \Bigr) \Biggr\rceil 
    \le 
    \Biggl\lceil  
    \frac{2L ||\xx^0 - \xx^\star||^2}{\epsilon} 
    +
    \frac{2\sigma^2 ||\xx^0 - \xx^\star||^2}{\epsilon^2}
    \Biggr\rceil 
    \;.
\end{equation}

\end{theorem}

\begin{proof}

    According to Lemma~\ref{thm:MainRecurrenceConvexFedRedGD},
    we have:
    \begin{equation}
    \begin{split}
        \E\Biggl[
        f(\xx^\star) 
        + \biggl[ \frac{\lambda + \mu/2}{2p} 
        - \frac{\mu}{4} \biggr] ||\Tilde{\xx}_k - \xx^\star||^2
        + \frac{\eta - \mu}{2} 
        \Avg &||\xx_{i,k} - \xx^\star||^2 \Biggr]
        \ge 
        \E\Biggl[
        f(\Bar{\xx}_{k+1}) 
        + \frac{\lambda + \mu/2}{2p} ||\Tilde{\xx}_{k+1} - \xx^\star||^2 
        \\
        &+ \frac{\eta - \mu/2}{2} 
        \Avg ||\xx_{i,k+1} - \xx^\star||^2
        \Biggr] 
        - \frac{\sigma^2}{2(\eta - L)} \;.
    \end{split}
    \end{equation}
    By our choices of parameters: 
    $\lambda = \delta_A$, and $p = \frac{\lambda + \mu / 2}{\eta - \mu/2}$,
    we have:
    \begin{equation}
        \frac{\lambda + \mu / 2}{2 p} 
        =
        \frac{\eta - \mu/2}{2}, \quad 
        \frac{\eta - \mu}{\eta - \mu/2} 
        =
        \frac{\frac{\lambda + \mu/2}{2p} - \frac{\mu}{4}}
        {\frac{\lambda + \mu/2}{2p}}
        = 1 - \frac{\mu}{2\eta - \mu} \;.
    \end{equation}
    
    Plugging these parameters into the main recurrence and 
    dividing both sides by $\frac{\eta - \mu/2}{2}$
    we get:
    \begin{equation}
    \begin{split}
        &\quad \E\biggl[
        \frac{2}{\eta - \mu/2}[
        f(\Bar{\xx}_{k+1}) - f^\star]
        +
        \Bigl[ 
        ||\Tilde{\xx}_{k+1} - \xx^\star||^2
        +
        \Avg ||\xx_{i,k+1} - \xx^\star||^2
        \Bigr]
        \biggr]
        \\
        &\le
        \E\biggl[ 
        \underbrace{\Bigl( 
        1 - \frac{\mu}{2\eta - \mu}
        \Bigr)}_{:=q}
        \Bigl[ 
        ||\Tilde{\xx}_{k} - \xx^\star||^2
        +
        \Avg ||\xx_{i,k} - \xx^\star||^2
        \Bigr]
        \biggr]
        +
        \frac{2 \sigma^2}{(\eta - L) (2\eta - \mu)} \;.
    \end{split}
    \end{equation}
    Applying Lemma~\ref{thm:StrongConvexityRecurrence},
    using the convexity of $f$ and multiplying
    both sides by $\frac{\eta - \mu/2}{2}$, 
    we get:
    \begin{equation}
        \E[ f(\Bar{\Bar{\xx}}_{K}) - f^\star] 
        + 
        \frac{\mu}{4}
        \E\biggl[ 
        ||\Tilde{\xx}_{K} - \xx^\star||^2
        + 
        \frac{1}{n} \sum_{i=1}^n
        ||\xx_{i,K} - \xx^\star||^2
        \biggr]
        \le 
        \frac{\mu}{2} \frac{1}{(1 + \frac{\mu}{2\eta - 2\mu})^K - 1}
        ||\xx_0 - \xx^\star||^2
        + \frac{\sigma^2}{2(\eta - L)} \;.
    \end{equation}
    This concludes the proof for the first claim.

    We next prove the second claim. 
    To achieve $\E[f(\Bar{\Bar{\xx}}_{K}) - f^\star]$, it is sufficient to let
    \begin{equation}
    \begin{cases}
        \frac{\mu}{2} \frac{1}{(1 + \frac{\mu}{2\eta - 2\mu})^K - 1}
        ||\xx_0 - \xx^\star||^2 \le \frac{\epsilon}{2}
        \\
        \frac{\sigma^2}{2 (\eta - L)} = \frac{\epsilon}{2}
    \end{cases}
    \Rightarrow
    \begin{cases}
        K \ge \frac{2\eta - \mu}{\mu} 
        \ln\Bigl(1 + \frac{\mu||\xx^0 - \xx^\star||^2}{\epsilon} \Bigr)
        \\
        \eta = \frac{\sigma^2}{\epsilon} + L
    \end{cases}\;.
    \end{equation}
    Plugging $\eta = \frac{\sigma^2}{\epsilon} + L$ into the expression for $K$,
    we get:
    \begin{equation}
        K \ge \Bigl( \frac{2L}{\mu} + \frac{2\sigma^2}{\mu \epsilon}\Bigr)
        \ln\Bigl(1 + \frac{\mu||\xx^0 - \xx^\star||^2}{\epsilon} \Bigr) \;.
        \qedhere
    \end{equation}
\end{proof}

\subsection{Non-convex result and proof for Algorithm~\ref{Alg:GDLocalSolver} with control variate~\eqref{eq:ControlVariateLocal2}}
\begin{theorem}
    Consider Algorithm~\ref{Alg:GDLocalSolver} with control 
    variate~\eqref{eq:ControlVariateLocal2} and randomized averaging
    with random index set $\{i_k\}_{k=0}^{+\infty}$.
    Let $f_i : \R^d \to \R$ be continuously differentiable
    and $L$-smooth for any $i \in [n]$. 
    Assume that $\{f_i\}$ have $\delta_B$-BHD with
    $\delta_B \le L$. 
    Assume that $\E[g_i(\xx)] = \nabla f_i(\xx)$ for any $\xx \in \R^d$ 
    and that $\E[||g_i(\xx) - \nabla f_i(\xx)||^2] \le \sigma^2$.
    By choosing 
    $\eta = 3L + \sqrt{9L^2 + \frac{L \sigma^2 K}{f(\xx^0) - f^\star} } $, 
    $\lambda = \delta_B$
    and $p = \frac{\delta_B}{L}$, for any $K \ge 1$, it holds that:
    \begin{equation}
        \E\bigl[||\nabla f(\Bar{\xx}_K)||^2 \bigr]
        \le 
        \frac{96 L (f (\xx^0) - f^\star)}{K}
        + 24 \sqrt{\frac{L (f(\xx^0) - f^\star)}{K}} \sigma
        \;.
    \end{equation}
    where $\Bar{\xx}_K$ is uniformly sampled from 
    $(\xx_{i_k,k})_{k=0}^{K-1}$.
\end{theorem}

\begin{proof}
    Let $h_i := f - f_i$. 
    Let 
    $
    G_{i,k}(\xx) 
    := 
    f_i(\xx_{i,k})
    +
    \lin{g_i(\xx_{i,k}), \xx - \xx_{i,k}}
    +
    \frac{\eta}{2} ||\xx - \xx_{i,k}||^2
    $.
    Recall that for any $i \in [n]$,
    the update of Algorithm~\ref{Alg:GDLocalSolver} satisfies:
    \begin{equation}
    \xx_{i,k+1} 
    = 
    \argmin_{\xx \in \R^d} 
    \{ G_{i,k} (\xx) 
    +
    \lin{ \nabla h_i(\Tilde{\xx}_k), \xx} 
    + \frac{\lambda}{2} ||\xx - \Tilde{\xx}_k||^2 \} \;.
    \end{equation}
    Using strong convexity, we have:
    \begin{equation}
    \begin{split}
        f_i(\xx_{i,k}) 
        +
        \lin{\nabla h_i(\Tilde{\xx}_k), \xx_{i,k}}
        +
        \frac{\lambda}{2} ||\xx_{i,k} - \Tilde{\xx}_k||^2
        \stackrel{\eqref{df:stconvex}}{\ge} 
        G_{i,k} (\xx_{i,k+1})
        +
        &\lin{\nabla h_i(\Tilde{\xx}_k), \xx_{i,k+1}}
        +
        \frac{\lambda}{2} ||\xx_{i,k+1} - \Tilde{\xx}_k||^2
        \\
        &+
        \frac{\lambda + \eta}{2} ||\xx_{i,k+1} - \xx_{i,k}||^2 \;.
    \end{split}
    \end{equation}
    Denote the randomness coming from $g_i(\xx_{i,k})$ by $\xi_i$ and
    denote all the randomness $\{\xi_i\}_{i=1}^n$ by $\xi$. Denote the randomness from the random selection of the index at 
    iteration $k$ by $i_k$.
    
    Let $\eta \ge L$. We get:
    \begin{align}
        \E_{\xi} [ G_{i,k} (\xx_{i,k+1}) ]
        &\stackrel{\eqref{eq:SmoothUpperBound}}{\ge}
        \E_\xi 
        \Bigl[
        f_i(\xx_{i,k+1}) + 
        \lin{ g_i(\xx_{i,k}) - \nabla f_i(\xx_{i,k}), 
        \xx_{i.k+1} - \xx_{i,k} } 
        + \frac{\eta - L}{2} ||\xx_{i,k+1} - \xx_{i,k}||^2
        \Bigr] \;.
    \end{align}

    It follows that:
    \begin{equation}
    \begin{split}
        f_i(\xx_{i,k}) 
        +
        \lin{\nabla h_i(\Tilde{\xx}_k), \xx_{i,k}}
        +
        &\frac{\lambda}{2} ||\xx_{i,k} - \Tilde{\xx}_k||^2
        \ge
        \E_{\xi} \biggl[
        f_i(\xx_{i,k+1})
        +
        \lin{ \nabla h_i(\Tilde{\xx}_k), \xx_{i,k+1}}
        +
        \frac{\lambda}{2} ||\xx_{i,k+1} - \Tilde{\xx}_k||^2
        \\
        &+
        \lin{ g_i(\xx_{i,k}) - \nabla f_i(\xx_{i,k}), 
        \xx_{i.k+1} - \xx_{i,k} } 
        +
        \frac{\lambda + 2\eta - L}{2} ||\xx_{i,k+1} - \xx_{i,k}||^2 
        \biggr]  \;.
    \end{split}
    \end{equation}

    According to Lemma~\ref{thm:TwoPointRelationNonConvexFramework} with 
    $\xx^r = \xx_{i,k}$ and $\yy = \xx_{i,k+1}$, we get:
    \begin{equation}
        f_i(\xx_{i,k}) + \lin{\nabla h_i (\xx_{i,k}), \xx_{i,k} - \xx_{i,k+1}}
        - f_i(\xx_{i,k+1}) \le f(\xx_{i,k}) - f(\xx_{i,k+1})
        + \frac{\delta_B}{2} ||\xx_{i,k+1} - \xx_{i,k}||^2 \;.
    \end{equation}
    
    Substituting this inequality into the previous display, we get:
    \begin{equation}
    \begin{split}
        f(\xx_{i,k})
        &+
        \frac{\lambda}{2} ||\xx_{i,k} - \Tilde{\xx}_k||^2 
        \ge
        \E_{\xi}\biggl[
        f(\xx_{i,k+1}) 
        -
        \lin{ \nabla h_i(\xx_{i,k}) - \nabla h_i(\Tilde{\xx}_k), \xx_{i,k+1} - \xx_{i,k}}
        -
        \frac{\delta_B}{2} ||\xx_{i,k+1} - \xx_{i,k}||^2
        \\
        &+
        \lin{ g_i(\xx_{i,k}) - \nabla f_i(\xx_{i,k}), 
        \xx_{i.k+1} - \xx_{i,k} } 
        +
        \frac{\lambda}{2} ||\xx_{i,k+1} - \Tilde{\xx}_k||^2
        +
        \frac{\lambda + 2\eta - L}{2} ||\xx_{i,k+1} - \xx_{i,k}||^2 
        \biggr] \;.
    \end{split}
    \label{eq:GDControlVariate2NonConvexFirstEquation}
    \end{equation}
    For any $\alpha > 0$, we have that:
    \begin{align}
        -
        \lin{\nabla h_i(\xx_{i,k}) - \nabla h_i(\Tilde{\xx}_k),
        \xx_{i,k+1} - \xx_{i,k}}
        &=
        \lin{\nabla h_i(\Tilde{\xx}_k) - \nabla h_i(\xx_{i,k}),
        \xx_{i,k+1} - \xx_{i,k}}
        \\
        &\stackrel{\eqref{eq:BasicInequality1}}{\ge} 
        -\frac{||\nabla h_i(\Tilde{\xx}_k) - \nabla h_i(\xx_{i,k})||^2}{2 \alpha}
        -\frac{\alpha ||\xx_{i,k+1} - \xx_{i,k}||^2}{2}
        \\
        &\ge
        -\frac{\delta_B^2||\Tilde{\xx}_k - \xx_{i,k}||^2}{2 \alpha}
        -\frac{\alpha ||\xx_{i,k+1} - \xx_{i,k}||^2}{2} \;.
    \end{align}
    Plugging this inequality 
    into~\eqref{eq:GDControlVariate2NonConvexFirstEquation}, we obtain,
    for any $i \in [n]$:
    \begin{equation}
    \begin{split}
        f(\xx_{i,k})
        +
        \Bigl(
        \frac{\lambda}{2} + \frac{\delta_B^2}{2\alpha}
        \Bigr)
        ||\xx_{i,k} - \Tilde{\xx}_k||^2 
        \ge
        \E_{\xi}\biggl[ 
        f(\xx_{i,k+1}) 
        &+
        \frac{\lambda}{2} ||\xx_{i,k+1} - \Tilde{\xx}_k||^2
        +
        \frac{\lambda + 2\eta - L - \delta_B - \alpha}{2}
        ||\xx_{i,k+1} - \xx_{i,k}||^2 
        \\
        &+
        \lin{ g_i(\xx_{i,k}) - \nabla f_i(\xx_{i,k}), 
        \xx_{i.k+1} - \xx_{i,k} } \biggr]\;.
    \end{split}
    \label{eq:GDControlVariate2NonConvexSecondEquation}
    \end{equation}
    We now lower-bound the last two terms.
    Recall that $\xx_{i,k+1}$ satisfies:
    \begin{equation}
        g_i(\xx_{i,k}) + \eta (\xx_{i,k+1} - \xx_{i,k})
        + \nabla h_i(\Tilde{\xx}_k) + \lambda (\xx_{i,k+1} - \Tilde{\xx}_k) = 0 \;.
    \end{equation}
    which implies:
    \begin{equation}
        \xx_{i,k+1} - \xx_{i,k}
        =
        -\frac{1}{\eta + \lambda}
        \bigl( 
        g_i(\xx_{i,k})
        +
        \nabla h_i(\Tilde{\xx}_k)
        +
        \lambda (\xx_{i,k} - \Tilde{\xx}_k)
        \bigr)
        \;.
    \end{equation}
    It follows that:
    \begin{align}
        &\quad \E_{\xi}\Bigl[ \lin{ g_i(\xx_{i,k}) - \nabla f_i(\xx_{i,k}), 
        \xx_{i.k+1} - \xx_{i,k} }  \Bigr]
        \\
        &=
        \E_{\xi}\Biggl[ \lin{ g_i(\xx_{i,k}) - \nabla f_i(\xx_{i,k}), 
        -\frac{1}{\eta + \lambda}
        \bigl( 
        g_i(\xx_{i,k})
        +
        \nabla h_i(\Tilde{\xx}_k)
        +
        \lambda (\xx_{i,k} - \Tilde{\xx}_k) }  \Biggr]
        \\
        &=
        -\frac{1}{\eta + \lambda}
        \E_{\xi}\Bigl[ \lin{ g_i(\xx_{i,k}) - \nabla f_i(\xx_{i,k}), 
        g_i(\xx_{i,k}) - \nabla f_i(\xx_{i,k}) + \nabla f_i(\xx_{i,k})} 
        \Bigr]
        \\
        &=
        -\frac{1}{\eta + \lambda}
        \E_\xi\bigl[ ||g_i(\xx_{i,k}) - \nabla f_i(\xx_{i,k})||^2 \bigr] \;,
    \end{align}
    and that:
    \begin{align}
        \E_\xi[|| \xx_{i,k+1} - \xx_{i,k} ||^2]
        &=
        \frac{1}{(\eta + \lambda)^2} 
        \E_\xi \Bigl[  
        ||g_i(\xx_{i,k}) - \nabla f_i(\xx_{i,k}) 
        + 
        \nabla f_i(\xx_{i,k})
        +
        \nabla h_i(\Tilde{\xx}_k)
        +
        \lambda (\xx_{i,k} - \Tilde{\xx}_k)
        ||^2
        \Bigr]
        \\
        &=
        \frac{1}{(\eta + \lambda)^2} 
        \Biggl[
        \E_\xi \Bigl[  
        ||g_i (\xx_{i,k}) - \nabla f_i(\xx_{i,k})||^2
        \Bigr]
        +
        \E_\xi \Bigl[  
        ||\nabla f_i(\xx_{i,k})
        +
        \nabla h_i(\Tilde{\xx}_k)
        +
        \lambda (\xx_{i,k} - \Tilde{\xx}_k)||^2
        \Bigr]
        \Biggr] \;.
    \end{align}
    where we use the assumption that $g_i(\xx_{i,k})$ is an unbiased estimator 
    of $\nabla f_i (\xx_{i,k})$.
    The last two terms of~\eqref{eq:GDControlVariate2NonConvexSecondEquation}
    can thus be lower bounded by:
    \begin{align}
        &\quad \E_{\xi} \Biggl[ 
        \frac{\lambda + 2\eta - L - \delta_B - \alpha}{2}
        ||\xx_{i,k+1} - \xx_{i,k}||^2 
        +
        \lin{ g_i(\xx_{i,k}) - \nabla f_i(\xx_{i,k}), 
        \xx_{i.k+1} - \xx_{i,k} }
        \Biggr] 
        \\
        &\ge A \E_\xi \Bigl[  
        ||\nabla f_i(\xx_{i,k})
        +
        \nabla h_i(\Tilde{\xx}_k)
        +
        \lambda (\xx_{i,k} - \Tilde{\xx}_k)||^2 \Bigr]
        - 
        B \E_\xi \Bigl[  
        ||g_i (\xx_{i,k}) - \nabla f_i(\xx_{i,k})||^2
        \Bigr] 
        \\
        &\ge 
        A \E_\xi \Bigl[  
        ||\nabla f_i(\xx_{i,k})
        +
        \nabla h_i(\Tilde{\xx}_k)
        +
        \lambda (\xx_{i,k} - \Tilde{\xx}_k)||^2 \Bigr]
        -B \sigma^2 \;.
    \end{align}
    where $A := \frac{\lambda + 2\eta - L - \delta_B - \alpha}
    {2 (\eta + \lambda)^2}$ and 
    $B := \frac{\lambda + L + \delta_B + \alpha}{2(\eta + \lambda)^2}$.
    
    Further note that:
    \begin{align}
        ||\nabla f_i(\xx_{i,k})
        +
        \nabla h_i(\Tilde{\xx}_k)
        +
        \lambda (\xx_{i,k} - \Tilde{\xx}_k)||
        &\ge
        ||\nabla f_i(\xx_{i,k}) + \nabla h_i(\Tilde{\xx}_k)||
        -
        \lambda ||\Tilde{\xx}_k - \xx_{i,k}||
        \\
        &\ge
        ||\nabla f(\xx_{i,k})|| - 
        ||\nabla h_i(\Tilde{\xx}_k) - \nabla h_i(\xx_{i,k})||
        -
        \lambda ||\Tilde{\xx}_k - \xx_{i,k}||
        \\
        &\stackrel{\eqref{eq:MaxHessianSimilarity}}{\ge} 
        ||\nabla f(\xx_{i,k})|| - 
        \delta_B ||\xx_{i,k} - \Tilde{\xx}_k||
        -
        \lambda ||\Tilde{\xx}_k - \xx_{i,k}||
        \\
        &=
        ||\nabla f(\xx_{i,k})||
        -
        (\delta_B + \lambda) 
        ||\xx_{i,k} - \Tilde{\xx}_k|| \;.
    \end{align}
    This gives:
    \begin{equation}
        ||\nabla f_i(\xx_{i,k})
        +
        \nabla h_i(\Tilde{\xx}_k)
        +
        \lambda (\xx_{i,k} - \Tilde{\xx}_k)||^2
        \stackrel{\eqref{eq:BasicInequality1}}{\ge}
        \frac{||\nabla f(\xx_{i,k})||^2}{2} - 
        (\delta_B + \lambda)^2 ||\xx_{i,k} - \Tilde{\xx}_k||^2 \;.
    \end{equation}
    Substituting all the previous displays 
    into~\eqref{eq:GDControlVariate2NonConvexSecondEquation}, we get:
    \begin{equation}
        f(\xx_{i,k})
        +
        C
        ||\xx_{i,k} - \Tilde{\xx}_k||^2 
        +
        B \sigma^2
        \ge
        \E_{\xi}\biggl[ 
        f(\xx_{i,k+1}) 
        +
        \frac{\lambda}{2} ||\xx_{i,k+1} - \Tilde{\xx}_k||^2
        +
        \frac{A}{2}
        ||\nabla f(\xx_{i,k})||^2  \biggr]\;.
    \end{equation}
    where 
    $C:= 
    \frac{\lambda}{2} + \frac{\delta_B^2}{2\alpha} 
    + A (\delta_B + \lambda)^2$.

    Since the previous display holds for any $i \in [n]$,
    we can take the expectation w.r.t. $\theta_k$ and $i_k$
    and get:
    \begin{equation}
    \begin{split}
        \E_{i_k,\theta_k}\biggl[
        f(\xx_{i_k,k})
        +
        C ||\xx_{i_k,k} - \Tilde{\xx}_k||^2 
        +
        B \sigma^2
        \biggl]
        \ge
        \E_{i_k,\theta_k,\xi}\biggl[
         f(\xx_{i_k,k+1}) 
        +
        \frac{\lambda}{2} 
        ||\xx_{i_k,k+1} - \Tilde{\xx}_k||^2
        +
        \frac{A}{2} ||\nabla f(\xx_{i_k,k})||^2 \biggr] \;,
        \label{eq:GDControlVariate2NonConvexEquation3}
    \end{split}
    \end{equation}

    Recall that with probability $p$, $\Tilde{\xx}_{k+1} = \xx_{i_k,k+1}$.
    It follows that:
    \begin{equation}
        \E_{i_k,\theta_k,\xi}[||\xx_{i_k,k+1} - \Tilde{\xx}_{k+1}||^2]
        =
        (1-p) \E_{i_k,\theta_k,\xi}[ ||\xx_{i_k,k+1} - \Tilde{\xx}_{k}||^2 ]  \;,
    \end{equation}
    Substituting this identity into the previous display, we get:
    \begin{equation}
    \begin{split}
        \E_{i_k,\theta_k}\biggl[
        f(\xx_{i_k,k})
        +
        C ||\xx_{i_k,k} - \Tilde{\xx}_k||^2 
        +
        B \sigma^2
        \biggl]
        \ge
        \E_{i_k,\theta_k,\xi}\biggl[
         f(\xx_{i_k,k+1}) 
        +
        \frac{\lambda}{2(1-p)} 
        ||\xx_{i_k,k+1} - \Tilde{\xx}_{k+1}||^2
        +
        \frac{A}{2} ||\nabla f(\xx_{i_k,k})||^2 \biggr] \;.
    \end{split}
    \end{equation}
    Note that $i_k$ and $i_{k+1}$  follow the same distribution and 
    are independent of $\{\theta_k\}_{k=0}^{+\infty}$. 
    It follows that:
    \begin{equation}
        \E_{i_k,\theta_k,\xi}\bigl[
        f(\xx_{i_k,k+1}) \bigl]
        =
        \E_{i_{k+1},i_k,\theta_k,\xi}\bigl[
        f(\xx_{i_{k+1},k+1}) \bigl], \; \text{and}\;
        \E_{i_k,\theta_k,\xi}\bigl[
        ||\xx_{i_k,k+1} - \Tilde{\xx}_{k+1}||^2\bigl]
        =
        \E_{i_{k+1},i_k,\theta_k,\xi}\bigl[
        ||\xx_{i_{k+1},k+1} - \Tilde{\xx}_{k+1}||^2 \bigl] \;.
    \end{equation}
    
    Taking expectation w.r.t $i_{k+1}$ on both sides of the previous display,
    substituting these two identities, and then taking the full expectation,
    we obtain our main recurrence:
    \begin{equation}
    \begin{split}
        \E[ f(\xx_{i_k,k}) ]
        +
        C \E[ ||\xx_{i_k,k} - \Tilde{\xx}_k||^2 ]
        +
        B \sigma^2
        \ge 
        \E[ f(\xx_{i_{k+1},k+1}) ]
        +
        &\frac{\lambda}{2(1-p)}
        \E[||\xx_{i_{k+1},k+1} - \Tilde{\xx}_{k+1}||^2]
        \\
        &+
        \frac{A}{2} \E[ ||\nabla f(\xx_{i_k,k})||^2 ] \;.
        \label{eq:GDControlVariate2NonConvexMainEquation}
    \end{split}
    \end{equation}
    
    It is clear now we need to choose parameters such that 
    $C \le \frac{\lambda}{2(1-p)}$, which is equivalent to:
    \begin{equation}
      \frac{\lambda}{2} + \frac{\delta_B^2}{2\alpha} 
      + \frac{\lambda + 2\eta - L - \delta_B - \alpha}
         {2 (\eta + \lambda)^2} (\delta_B + \lambda)^2
      \le
      \frac{\lambda}{2(1-p)} \;,
    \end{equation}
    while at the same time keeping $A$ sufficiently large. 
    
    By choosing $p = \frac{\delta_B}{L}$, $\lambda = \delta_B$, $\eta \ge 3L$
    and
    $\alpha = 5L$, we obtain:
    \begin{equation}
      \frac{\lambda}{2} + \frac{\delta_B^2}{2\alpha}
      +
      \frac{\lambda + 2\eta - L - \delta_B - \alpha}{2}
      \frac{(\delta_B + \lambda)^2}{(\eta + \lambda)^2}
      \le
      \frac{\delta_B}{2} + \frac{\delta_B^2}{10 L} 
      + (\eta - 3L) \frac{4 \delta_B^2}{\eta^2}
      \le \frac{\delta_B}{2} + \frac{\delta_B^2}{2L} 
      \le \frac{\lambda}{2(1-p)} \;,
    \end{equation}
    where in the last second inequality, we use the fact that 
    $\frac{\eta - 3L}{\eta ^2} \le \frac{1}{12L}$ for any $\eta \ge 3L$.
    
    Plugging in these parameters into~\eqref{eq:GDControlVariate2NonConvexMainEquation} 
    , summing up from $k=0$ to $K-1$ and dividing both sides by $K$, 
    we obtain:
    \begin{equation}
        \frac{1}{K}\sum_{k=0}^{K-1}
        \E\Bigl[ ||\nabla f(\xx_{i_k,k})||^2 \Bigr]
        \le 
        \frac{2(f(\xx^0) - f^\star)}{A K} 
        +
        \frac{2B \sigma^2}{A}
        \;,
    \end{equation}
    where we use the fact that $\E[||\xx_{i_0,0} - \Tilde{\xx}_0||] = 0$.
    Plugging the definition of $A$ and $B$, we obtain:
    \begin{align}
        \frac{2(f(\xx^0) - f^\star)}{A K} 
        +
        \frac{2B \sigma^2}{A}   
        &=
        \frac{2 (\eta +\delta_A)^2 }{\eta - 3L} 
        \frac{(f(\xx^0) - f^\star)}{K}
        +
        \frac{2 (\delta_B + 3L)}{\eta - 3L} \sigma^2
        \\
        &\le 
        \frac{8 \eta^2}{\eta - 3L}  \frac{(f(\xx^0) - f^\star)}{K}
        +
        \frac{8 L}{\eta - 3L} \sigma^2
        \;.
    \end{align}
    Let $v_1 := \frac{f(\xx^0) - f^\star}{K}$ and $v_2 := L\sigma^2$.
    The upper bound can be written as:
    $
    8\bigl[ 
    v_1 (\eta - 3L) + \frac{9v_1 L^2 + v_2}{\eta - 3L} + 6 v_1 L
    \bigr]
    $.
    Minimizing the bound w.r.t. $\eta$ over $\eta \ge 3L$, we get
    $\eta^\star = 3L + \sqrt{9L^2 + \frac{v_2}{v_1}} $. Plugging this 
    choice into the upper bound and using the fact that
    $\eta^\star \ge 6L$, $\eta^\star \ge 3L + \sqrt{\frac{v_2}{v_1}}$,
    and $(\eta ^\star)^2 \le 18L^2 + 18L^2 + 2\frac{v_2}{v_1}$, 
    we get:
    \begin{align}
        \frac{8 \eta^2}{\eta - 3L}  \frac{(f(\xx^0) - f^\star)}{K}
        +
        \frac{8 L}{\eta - 3L} \sigma^2
        &\le 
        \frac{96 L (f (\xx^0) - f^\star)}{K}
        +
        16 \sqrt{\frac{v_2}{v_1}} \frac{f(\xx^0) - f^\star}{K}
        +
        8 L \sigma^2 \sqrt{\frac{v_1}{v_2}} 
        \\ 
        &=
        \frac{96 L (f (\xx^0) - f^\star)}{K}
        + 24 \sqrt{\frac{L (f(\xx^0) - f^\star)}{K}} \sigma \;.
    \end{align}
\end{proof}

\section{Additional experiments}
\label{sec:DLexperiments}

\textbf{Deep learning task.}
We consider the multi-class classification tasks with CIFAR10~\cite{cifar10}
and CIFAR100~\cite{cifar100} datasets using ResNet-18~\cite{resnet}.
We use $n=10$ and partition the dataset according to the Dirichlet distribution following~\cite{lin2020ensemble} with the concentration parameter $\alpha=0.5$ to simulate the heterogeneity scenario in the cross-silo setting. We compare \algname{DANE+-SGD} and \algname{FedRed-SGD} against \algname{Scaffold}~\cite{scaffold}, \algname{Scaffnew}~\cite{proxskip} 
and \algname{Fedprox}~\cite{fedprox}. We use \algname{SGD} with
a mini-batch size of $512$ for all the methods as the local subsolver.
For \algname{DANE+-SGD}, \algname{FedProx}, and \algname{Scaffold}, we fix the number of local steps and 
select the best one among $\{10, 20\}$. For \algname{FedRed-SGD} and \algname{Scaffnew},
we select the best $p$ from $\{0.1, 0.05\}$. We use $\eta_g=1$ for \algname{Scaffold}
and choose the best local constant learning rate from $\{0.01,0.02,0.05,0.1\}$.
For \algname{DANE+-SGD}, \algname{FedRed-SGD}, and \algname{FedProx}, we select the best $\lambda$
from $\{10^i\}_{i=-3,-2,...,1}$. 
We use the exact control variate~\eqref{eq:ControlVariate2} for \algname{Scaffold}, 
\algname{DANE+-SGD} and \algname{FedRed-SGD}. From Figure~\ref{fig:cifar-main},
we surprisingly observe that \algname{FedProx} consistently outperforms other methods and exhibits more stable convergence. The primary difference is that
\algname{FedProx} does not apply any drift correction. Given the reported ineffectiveness of variance reduction for training deep neural networks~\cite{vrineffectiveness}, there remains a need for deeper investigations into the effectiveness of drift correction in federated learning for the training of non-convex and non-smooth neural networks.

\begin{figure*}[tb!]
    \centering
    \includegraphics[width=1\textwidth]{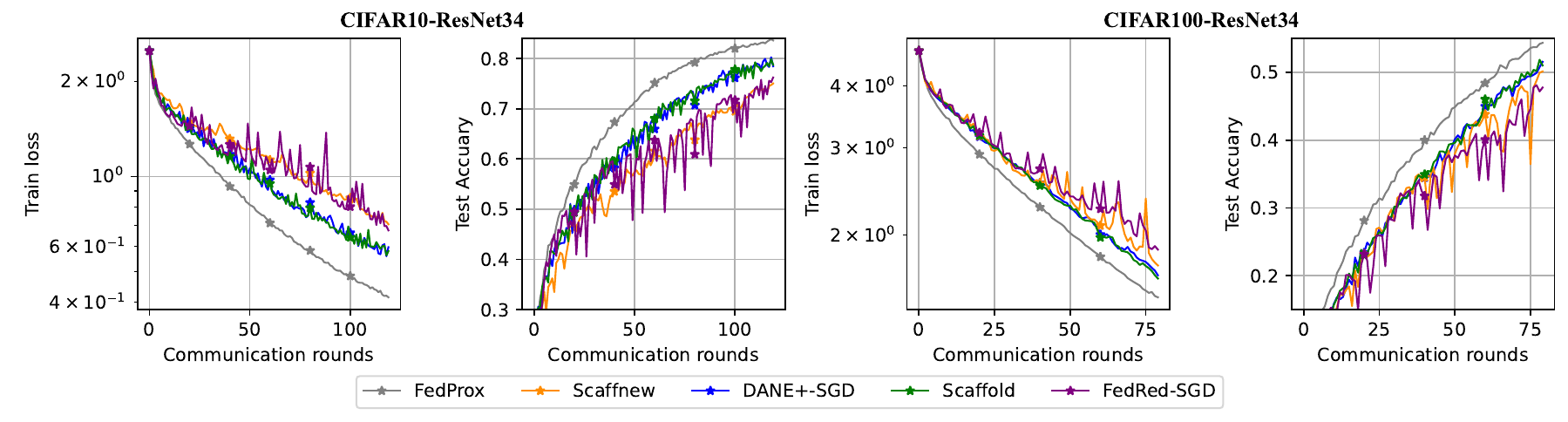}
    \caption{Comparison of \algname{DANE+-SGD} and \algname{FedRed-SGD} against three other distributed optimizers on multi-class classification tasks with CIFAR10 and CIFAR100 datasets using ResNet18 with softmax loss. 
    \algname{FedProx} (without drift correction) exhibits faster and more stable convergence compared to the other methods.}
    \label{fig:cifar-main}
\end{figure*}

\end{document}